\documentclass[conference]{IEEEtran}
\IEEEoverridecommandlockouts

\usepackage{cite}
\usepackage{amsmath,amssymb,amsfonts}
\usepackage{mathtools}
\usepackage{amsthm}
\usepackage{algorithm}
\usepackage{algpseudocode}
\usepackage{graphicx}
\usepackage{booktabs}
\usepackage{multirow}
\usepackage{subcaption}
\usepackage[maxfloats=45]{morefloats}
\usepackage{xcolor}
\usepackage{pifont}
\usepackage[hidelinks]{hyperref}
\usepackage{url}

\newcommand{\method}{\textsc{DiffInt}}
\newcommand{\green}[1]{\textbf{#1}}
\newcommand{\blue}[1]{\underline{#1}}
\newcommand{\lightred}[1]{\textit{#1}}

\newtheorem{theorem}{Theorem}
\newtheorem{proposition}{Proposition}
\newtheorem{corollary}{Corollary}

\newtheorem{remark}{Remark}
\newtheorem{definition}{Definition}

\begin{document}

\title{Differentiable Interval Bottlenecks for\\
Interpretable Anomaly Detection in Numerical Data}

\author{\IEEEauthorblockN{Lamine Diop and Marc Plantevit}
\IEEEauthorblockA{\textit{EPITA, Laboratoire de Recherche de l'EPITA (LRE)}\\Le Kremlin-Bic\^etre, France\\
\{lamine.diop, marc.plantevit\}@epita.fr}}

\maketitle

\begin{abstract}
Reconstruction-based anomaly detectors are accurate but opaque: a deep
autoencoder flags a sample without telling a practitioner \emph{which feature
ranges} made it anomalous. We propose \method{}, an autoencoder whose latent
bottleneck is structured as a set of soft, axis-aligned \emph{interval}
memberships learned end-to-end directly from raw numerical data, without any
discretization or binarization. Each latent unit corresponds to a
human-readable hyper-rectangle in feature space; an instance is encoded by how
strongly it falls inside each interval \emph{relative to the other units}, and
its reconstruction error is the anomaly score. This keeps the power of
differentiable representation learning while exposing an inspectable internal
structure. We make the inductive bias precise: a \emph{certified} reconstruction-error lower bound for points that
fall outside every active coordinate of the learned support (with a
Lipschitz-enforced decoder), and a graded, empirically verified suppression
mechanism for the usual case in which only a few features are abnormal; and we
provide a closed-form, \emph{label-free importance} that ranks each
(unit, feature) pair from quantities the model already maintains, turning
trained intervals into auditable \emph{candidate constraints} without ever
seeing an anomaly label. On 48 ADBench benchmarks against 22 baselines under a
common $[-1,1]$-normalized protocol, \method{} attains the \emph{best mean rank
overall on both metrics} ($4.10$ on ROC--AUC, $4.16$ on AUPR); among
inlier-only detectors it leads its regime clearly, and it is competitive with
the strongest contaminated-data detectors (see the stratified and complete-case
analyses). It is the only interpretable detector in the statistically-tied
leading cluster of seven methods.
\end{abstract}


\section{Introduction}
\label{sec:intro}

Anomaly detection on tabular numerical data underpins fraud screening, fault
monitoring, and clinical triage. Reconstruction-based deep detectors
(autoencoders, VAEs, and one-class networks) are among the most accurate
general-purpose methods on heterogeneous benchmarks~\cite{han2022adbench}.
Their weakness is not accuracy but \emph{accountability}: the latent code is an
entangled, rotationally ambiguous vector~\cite{locatello2019challenging,
lipton2018mythos}, so a flagged record comes with a scalar error and no
statement of \emph{which feature values} were responsible. A fraud analyst
who receives ``account $\mathrm{42819}$: error $7.4$'' cannot tell whether to
escalate, refund, or dismiss. Regulation codifies the same need -- the EU AI
Act's ``meaningful information about the logic involved,'' feature-level reason
codes in credit scoring -- so an unexplained alert has little practical use.

A natural way to obtain feature-grounded explanations is to constrain the model
to reason in \emph{intervals}: axis-aligned hyper-rectangles, the numerical
analogue of pattern-mining items and rule-learning conditions, immediately
readable as ``feature $j \in [a,b]$''. Classical interval pattern discovery,
however, relies on exhaustive enumeration over a discretized
space~\cite{DBLP:conf/ijcai/KaytoueKN11,
DBLP:conf/pkdd/BelfodilBK18}, which both loses information and scales poorly.
Differentiable pattern mining~\cite{fischer_differentiable_2021,
walter_finding_2023,DBLP:conf/pkdd/ChataingPPR24} replaces enumeration with
gradient descent but, to date, operates only on \emph{binary} data and thus
still requires the inputs to be discretized first.

We take the differentiable route but keep the data numerical. \method{}\footnote{Code:
\url{https://github.com/DiffInt/diffint}.} is an
autoencoder whose bottleneck is a layer of soft interval memberships: each
latent unit is a learnable hyper-rectangle (a center and a softplus half-width,
with sigmoid boundaries for differentiability). A decoder reconstructs the input
from these memberships; the model is trained on normal data with a
reconstruction loss, and the reconstruction error is the anomaly score. Because
the bottleneck \emph{is} the intervals, the internal state is inspectable, and a
simple statistic turns each into an auditable candidate constraint.

We limit our claims: \method{} is \emph{not} a pattern-set miner -- we claim no
minimality, optimal pattern count, or description-length objective. Interval
structure here is a \emph{mechanism} for an interpretable anomaly detector,
evaluated in that role. Our contributions are:

\begin{itemize}
\item \textbf{An interval-structured differentiable bottleneck} for
reconstruction-based anomaly detection that operates on raw numerical features
with no discretization, using a competitive log-sum-exp aggregator that gives
a scalar per-unit code (Sec.~\ref{sec:method}).
\item \textbf{A method-specific clipping-margin result}: a graded suppression
mechanism (Cor.~\ref{cor:sparse}) whereby a unit's weight decays geometrically
in the number of its active coordinates a point violates -- the operative
property for the realistic anomaly with a few abnormal features -- of which the
full \emph{certified} reconstruction-error lower bound (Thm.~\ref{prop:clip},
needing violation on \emph{every} active coordinate and a Lipschitz-enforced
decoder) is the extreme case, verified numerically in
Sec.~\ref{sec:certify}. This is what justifies the reconstruction error as an
anomaly score (Sec.~\ref{sec:theory}), together with the stability properties
(EMA support, Lipschitz regularity) that the model relies on.
\item \textbf{An analysis of the input-domain choice}: we show that min--max
scaling to $[-1,1]$ is an integral part of the method. It aligns random
initialization with the data support, makes the sigmoid temperature $\tau$
dimensionless, places the empty-code attractor of Theorem~\ref{prop:clip}
at the inlier centroid, and bounds the Lipschitz slack uniformly across datasets
(Sec.~\ref{sec:norm}); one configuration then transfers to all 48 datasets
without retuning.
\item \textbf{A large-scale evaluation}: 48 ADBench datasets, 22
baselines, with the semi-supervised training protocol, the handling of missing
baseline runs, and the exact statistical reading of the critical-difference
diagrams all stated up front (Sec.~\ref{sec:exp}). Because \method{} and three
baselines train on inliers only while the remaining baselines run on
contaminated data, we add a \emph{stratified} and a \emph{complete-case}
re-analysis in App.~\ref{sup:strat} so the comparison is read within, not
across, training regimes. We claim only what the statistics support.
\end{itemize}

\section{Related Work}
\label{sec:related}

\paragraph*{Reconstruction and one-class deep detectors}
Autoencoders~\cite{sakurada2014ae}, VAEs~\cite{kingma2013vae}, and
DAGMM~\cite{zong2018dagmm} score anomalies by reconstruction error;
DeepSVDD~\cite{ruff2018deepsvdd} learns a one-class hypersphere. Recent
variants add an OCSVM-guided boundary or sample retrieval, but expose a
\emph{global} kernel boundary or retrieval cues rather than per-feature ranges.
These models are strong but opaque, and are known to reconstruct anomalies that
lie near the learned
manifold~\cite{zong2018deep,pidhorskyi2020adversarial,bergman2020deep}. Our
contribution is orthogonal to raw accuracy: we keep the reconstruction paradigm
but replace the entangled bottleneck with an interval-structured one, with a
deterministic argument (Sec.~\ref{sec:theory}) for \emph{why} out-of-support
points are not reconstructed.

\paragraph*{Classical, self-supervised and generative detectors}
LOF~\cite{breunig2000lof}, OCSVM~\cite{scholkopf1999ocsvm},
PCA~\cite{shyu2003pca}, KDE~\cite{latecki2007kde}, GOAD~\cite{bergman2020goad},
ICL~\cite{shenkar2022icl}, SLAD~\cite{xu2023slad}, MCM~\cite{yin2024mcm},
GAN-based~\cite{schlegl2017anogan}, diffusion/flow
detectors~\cite{ho2020ddpm,livernoche2023dte} and the graph detector
LUNAR~\cite{goodge2022lunar} span the ADBench suite~\cite{han2022adbench}; the
recent flow-matching contraction model TCCM~\cite{tccm2025} is a strong recent
inductive baseline. They differ widely in interpretability; none exposes
feature-range explanations natively.

\paragraph*{Pattern mining and differentiable rule learning}
Frequent and quality pattern mining~\cite{DBLP:conf/vldb/AgrawalS94,
DBLP:conf/kdd/ZakiPOL97,DBLP:journals/datamine/VreekenLS11,
DBLP:journals/datamine/Bie11} and interval pattern
structures~\cite{DBLP:conf/ijcai/KaytoueKN11,DBLP:conf/pkdd/BelfodilBK18}
enumerate patterns over discretized data. BinaPs~\cite{fischer_differentiable_2021}
and DiffNaps/DiffVersify~\cite{walter_finding_2023,
DBLP:conf/pkdd/ChataingPPR24,chataing2024diffversify} introduced
\emph{differentiable} pattern discovery, but only for binary data.
Neuro-symbolic rule miners such as Aerial+~\cite{karabulut25a} still discretize
numerical attributes; supervised hyper-rectangle ensembles such as
GBM--HRBM~\cite{konstantinov2023interpretable} need labels and are not
end-to-end differentiable; differentiable rule
networks~\cite{wang2021scalable} target
prediction, not unsupervised detection. To our knowledge, \method{} is the first
to use soft numerical interval memberships as a structured bottleneck in a
reconstruction-based anomaly detector trained on raw numerical features without
discretization, and to empirically validate the resulting interval-based
explanations. We claim novelty only for this specific combination, not for
axis-aligned soft regions in general.

\paragraph*{Why a structured bottleneck}
Representation-learning analyses show generic encoders produce \emph{entangled},
non-identifiable factors absent a strong inductive
bias~\cite{locatello2019challenging,bengio2013representation,lipton2018mythos};
an interval bottleneck supplies one -- coordinate-wise bounds that keep
geometric meaning and yield stable regions across runs, and whose hard
boundaries counteract the near-manifold anomaly reconstruction
(Sec.~\ref{sec:theory}) that an unconstrained code permits. Related structured
bottlenecks each fall short on some axis: prototype networks~\cite{chen2019looks}
expose exemplar \emph{points}, not ranges; sparse/disentangled autoencoders
($\beta$-VAE~\cite{higgins2017beta}) align directions but bound no coordinate
(identifiability results~\cite{locatello2019challenging} show this is
fundamental); supervised hyper-rectangle
ensembles~\cite{konstantinov2023interpretable} are not end-to-end
differentiable. The interval bottleneck is unsupervised, differentiable, and
bounded coordinate-wise at once.

\section{\method{}}
\label{sec:method}

\subsection{Problem and notation}
Let $\mathcal{D}=\{x^{(i)}\}_{i=1}^{n}\subset\mathbb{R}^{d}$ be a numerical
dataset, with features min--max scaled to $[-1,1]$ (Sec.~\ref{sec:norm}). An
\emph{interval pattern} $\mathcal{X}_k$ assigns to every feature $j$ a closed
interval $[\underline{c}_{kj},\overline{c}_{kj}]$; instance $x$ is covered by
$\mathcal{X}_k$ iff $\underline{c}_{kj}\le x_j\le\overline{c}_{kj}$ for all $j$,
i.e.\ $\mathcal{X}_k$ is a hyper-rectangle. We train only on normal
(inlier) data and, at test time, score an instance by its reconstruction
error. We use $K$ interval units as a fixed bottleneck width; $K$ is a capacity
hyperparameter, \emph{not} a claimed minimal pattern set
(Sec.~\ref{sec:limits}).

\subsection{Input normalization to $[-1,1]$}
\label{sec:norm}

We prescribe min--max scaling to $[-1,1]^d$ as part of the method, not an
interchangeable preprocessing detail. Every component is calibrated to a
symmetric bounded domain, which is what lets a single configuration generalize
across all 48 datasets.

\noindent\textbf{(i) Initialization is data-aligned.}
Centers $m_{kj}$ and width logits $\delta_{kj}$ are drawn from a narrow
zero-mean Gaussian; after softplus, half-widths concentrate around
$\log 2\!\approx\!0.69$, so a fresh unit covers $\approx[-0.7,0.7]$ per
feature, a non-degenerate fraction of $[-1,1]$. Under un-normalized or $[0,1]$ data
the same init lies outside or off-center, and training must first
\emph{translate} every interval; the symmetric choice removes this transient
and makes the random init of Algorithm~\ref{alg:train} valid across datasets.

\noindent\textbf{(ii) The temperature $\tau$ becomes dimensionless.}
The boundary sigmoid $\sigma((x_j-\underline c_{kj})/\tau)$ has a transition
zone of width $\Theta(\tau)$ \emph{in the units of $x_j$}; on $[-1,1]^d$ every
feature spans $2$, so a single $\tau{=}0.1$ is a $5\%$ zone everywhere. Any
non-uniform scaling forces a per-feature $\tau_j$, and the joint membership
$I_{kj}$ then loses the \emph{single} clipping-margin parameter
$\beta=\sigma(-\eta/\tau)$ of Thm.~\ref{prop:clip}, and the closed-form bound
is lost.

\noindent\textbf{(iii) The empty-code attractor is the inlier centroid.}
Out-of-support points are pulled toward $g_\theta(f_0)$ (Thm.~\ref{prop:clip}).
Trained on inliers with RMSE, $g_\theta(f_0)$ sits near the inlier centroid,
which under $[-1,1]$ scaling is a neighborhood of the origin; anomalies lie far
from it, so the attractor separates them by the full \emph{inlier amplitude},
the largest margin available. $[0,1]$ shifts the centroid off the attractor,
and StandardScaler makes magnitudes dataset-dependent and the bound
incomparable.

\noindent\textbf{(iv) The Lipschitz slack is dataset-independent.}
The slack $\tfrac{L}{d}\kappa(\beta)$ of Thm.~\ref{prop:clip} transfers across
datasets only if the domain has a fixed diameter; $[-1,1]^d$ supplies $2\sqrt
d$, whereas StandardScaler's $\sup\|x\|$ scales with the tail. The bounded
domain likewise keeps the encoder constant of Prop.~\ref{prop:lip} the same
quantity across datasets.

\noindent\textbf{(v) Symmetry balances the two boundaries.}
The two boundary gradients have magnitude $\tfrac{1}{4\tau}I_{kj}$ each; on
$[0,1]$ the inlier mass loads them asymmetrically and intervals tend to
collapse to half-rectangles ``$x_j\le b$''. Symmetric scaling equalizes them in
expectation, yielding the two-sided rules ``$x_j\in[a,b]$'' an analyst reads off
(Sec.~\ref{sec:lfi}).

\noindent\textbf{Empirical confirmation.}
What matters is \emph{whether} inputs are mapped into a bounded symmetric domain:
removing normalization costs $\approx3$ AUROC points (Table~\ref{tab:ablate}a,
$0.805$ vs.\ $[-1,1]$ $0.833$). Among bounded/standardized choices $[-1,1]$ is
also the \emph{most accurate} ($0.833$ vs.\ StandardScaler $0.818$, $[0,1]$
$0.824$; $\le\!1.5$-point spread), and it is the only one under which reasons
(i)--(v) hold \emph{simultaneously}, so one configuration transfers across all
$48$ datasets without retuning. At inference, test features are train-scaled and
\emph{clipped} to $[-1,1]$; clamping can only \emph{lower} memberships, so it is
unlikely to mask anomalies (none observed).

\subsection{Differentiable interval bottleneck}
To make boundaries learnable we use a center/half-width parameterization. For
unit $k$ and feature $j$, an unconstrained center $m_{kj}$ and width logit
$\delta_{kj}$ give a strictly positive half-width via softplus,
\(
\Delta_{kj}=\operatorname{softplus}(\delta_{kj})=\ln(1+e^{\delta_{kj}})>0,
\)
so the interval is
$[\,\underline{c}_{kj},\overline{c}_{kj}\,]=[\,m_{kj}-\Delta_{kj},\,
m_{kj}+\Delta_{kj}\,]$. The softplus keeps each half-width strictly positive and
decoupled across coordinates, with bounded derivative
$\sigma(\delta_{kj})\!\in\!(0,1)$, so widths and gradients remain bounded
throughout training. Hard membership is non-differentiable; we replace each
boundary with a sigmoid of temperature $\tau>0$:
\[
s^{\mathrm{low}}_{kj}(x_j)=\sigma\!\Big(\tfrac{x_j-\underline{c}_{kj}}{\tau}\Big),
\quad
s^{\mathrm{up}}_{kj}(x_j)=\sigma\!\Big(\tfrac{\overline{c}_{kj}-x_j}{\tau}\Big),
\]
and define the soft per-feature membership
\(
I_{kj}(x_j)=s^{\mathrm{low}}_{kj}(x_j)\,s^{\mathrm{up}}_{kj}(x_j)\in[0,1],
\)
which is $\approx 1$ well inside the interval and decays smoothly outside.
Centers and width logits are initialized from a narrow zero-mean Gaussian, and
all interval parameters are learned jointly with the decoder.

\subsection{Membership aggregation}
The $d$ per-feature memberships of a unit must be combined into a pattern
activation. We use a competitive log-sum-exp aggregator,
\[
f_{k}(x)=
\frac{\exp\!\big(\sum_{j}\log I_{kj}(x_j)\big)}
{\sum_{r=1}^{K}\exp\!\big(\sum_{j}\log I_{rj}(x_j)\big)}\in[0,1].
\]
This is a softmax over log-memberships: units compete to explain a point, and
as $\tau\!\to\!0$ a single covering unit wins. It yields a scalar per unit, so
the encoder output is unambiguously $f(x)\in\mathbb{R}^{K}$.
The code is a \emph{relative} competition, not an absolute membership:
$\sum_k f_k(x)=1$ always, a point outside every box does not get a small code
but collapses to the fixed empty code $f_0$ (Thm.~\ref{prop:clip}(i)), and
per-instance ``how far inside box $k$'' should be read from the soft box
membership $m_k(x)=\prod_j I_{kj}(x_j)$, not $f_k(x)$. Numerically we clamp
$I_{kj}\ge10^{-8}$ before the log and normalize with a max-subtracting softmax,
so the aggregator and its gradients stay finite arbitrarily far out of support.

\subsection{Decoder, training, and EMA support}
\label{sec:decoder}
The decoder is a two-layer ReLU MLP $g_\theta:\mathbb{R}^{K}\!\to\!\mathbb{R}^{d}$
with hidden width $h{=}128$, a single LayerNorm conditioning the hidden
activations (default), and a linear read-out; it maps the bottleneck back to
input space, $\hat{x}=g_\theta(f(x))$. Training minimizes the per-sample RMSE
$\mathcal{L}_{\mathrm{RMSE}}(x)=\sqrt{\tfrac{1}{d}\|x-\hat{x}\|_2^2}$ on inliers
only (label $0$). We expose two decoders: the \emph{default} (LayerNorm,
finite but unconstrained Lipschitz constant) and the \emph{Certified-Lipschitz
\method{}}, which spectrally normalizes every linear map and drops every
non-$1$-Lipschitz layer (including LayerNorm), so $L=\prod_l\|W_l\|_2$ is an
explicit Lipschitz upper bound; this regularity is what turns
Thm.~\ref{prop:clip} into a deterministic certificate (Sec.~\ref{sec:certify}).
To obtain a stable, dataset-level activation statistic for interval scoring and
interpretability, \method{} keeps an exponential moving average of the soft
support,
\[
\mathrm{support}^{(t+1)}_{kj}\leftarrow
\rho\,\mathrm{support}^{(t)}_{kj}
+(1-\rho)\tfrac{1}{|\mathcal{B}|}\sum_{i\in\mathcal{B}}I_{kj}(x^{(i)}_j),
\quad\rho\in[0,1).
\]
The EMA is purely a diagnostic statistic; it is not part of the forward pass.
At test time the anomaly score of $x$ is the mean absolute reconstruction error
$\mathrm{MAE}(x)=\tfrac1d\|x-\hat{x}\|_1$ (averaged over the $d$ scaled
coordinates, so $\mathrm{MAE}(x)\!\in\![0,2]$ on $[-1,1]^d$). RMSE is used as
the training objective and MAE as the test-time score, a standard pairing
that keeps the optimization robust while the score itself stays easy to read.

\begin{algorithm}[t]
\caption{\method{} training (one epoch)}
\label{alg:train}
\begin{algorithmic}[1]
\Require inliers $\mathcal{D}$, $K$, $\tau$, EMA decay $\rho$
\ForAll{mini-batches $\mathcal{B}\subset\mathcal{D}$}
       \For{$k=1,\dots,K$ \textbf{and} $j=1,\dots,d$}
              \State $\Delta_{kj}\gets\operatorname{softplus}(\delta_{kj})$
              \State compute $s^{\mathrm{low}}_{kj}, s^{\mathrm{up}}_{kj}$ and
                                    $I_{kj}(x_j)\gets s^{\mathrm{low}}_{kj}s^{\mathrm{up}}_{kj}$
       \EndFor
       \State $f_k(x)\gets\textsc{Aggregate}(\{I_{kj}\}_{j=1}^d)$
       \State $\hat{x}\gets g_\theta(f(x))$
       \State backprop $\mathcal{L}_{\mathrm{RMSE}}$; optimizer step
       \State update $\mathrm{support}_{kj}$ by EMA
\EndFor
\Ensure interval params $\{m_{kj},\delta_{kj}\}$, decoder $\theta$
\end{algorithmic}
\end{algorithm}

\begin{figure*}[t]
\centering
\includegraphics[width=0.98\textwidth]{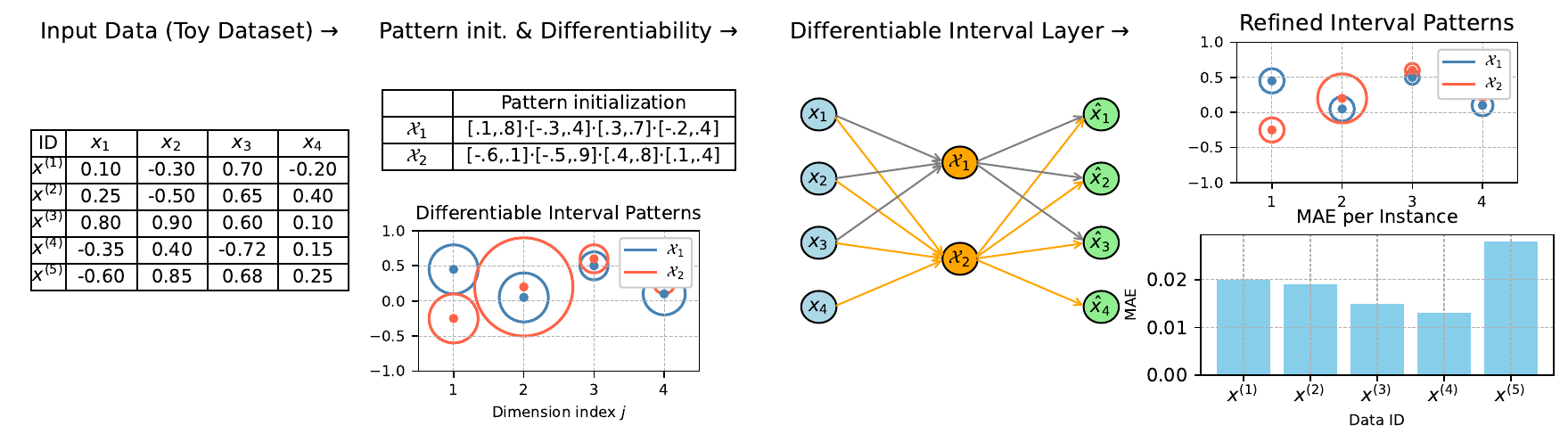}
\caption{\textbf{\method{} end-to-end pipeline on a 4-feature toy
dataset} (left $\to$ right).
\textbf{(a) Input:} $n{=}5$ samples with $d{=}4$ features min--max scaled
to $[-1,1]$ (Sec.~\ref{sec:norm}); $x^{(1)}{-}x^{(4)}$ are inliers,
$x^{(5)}$ is the anomaly.
\textbf{(b) Initialization:} $K{=}2$ interval units
$\mathcal{X}_1,\mathcal{X}_2$ as centers (dots) with softplus half-widths
(rings); sigmoid-soft boundaries of temperature $\tau$ make all gradients
flow through.
\textbf{(c) Differentiable interval layer:} the bottleneck is the bipartite
map $\{x_j\}\!\to\!\{\mathcal{X}_k\}\!\to\!\{\hat x_j\}$ with per-feature
memberships $I_{kj}(x_j)\!=\!\sigma((x_j{-}\underline c_{kj})/\tau)
\sigma((\overline c_{kj}{-}x_j)/\tau)$, aggregated per unit
and decoded by a small MLP $g_\theta$; RMSE between $x,\hat x$ is
backpropagated \emph{jointly} into $(m_{kj},\delta_{kj})$ and $\theta$.
\textbf{(d) Refined intervals and anomaly score:} after training,
intervals shrink onto the inlier cloud (top), and per-instance MAE
(bottom) is small for inliers but visibly large for $x^{(5)}$, the
empirical signature of the clipping-margin bound
(Thm.~\ref{prop:clip}), which pulls out-of-support points toward the
empty-code reconstruction $g_\theta(f_0)$. Each refined unit reads as an
auditable candidate constraint ``$\mathcal{X}_k\!:\,x_j\!\in\![\underline c_{kj},\overline
c_{kj}]$'' (Sec.~\ref{sec:lfi}).}
\label{fig:framework}
\end{figure*}

\section{What the Interval Structure Provides}
\label{sec:theory}

We state results that are specific to this model and that the detector actually
uses. The core is a clipping-margin result on the reconstruction error of
out-of-support points (\emph{certified} when the decoder is Lipschitz-enforced,
an empirical mechanism otherwise): it makes no independence-across-coordinates
assumption and is consistent with the well-known fact that autoencoders can
otherwise reconstruct high-dimensional anomalies. All proofs are in
Appendix~\ref{app:proofs}.

\begin{definition}[Active coordinate]
\label{def:active}
Coordinate $j$ is \emph{active} for unit $k$ if its learned interval
constrains inliers there, i.e.\ the mean inlier membership
$\bar I_{kj}=\mathbb{E}_{\mathcal{D}}[I_{kj}(x_j)]$ is bounded away from $1$
(e.g.\ when the half-width satisfies $\Delta_{kj}<\tfrac12\,\mathrm{range}_j$).
Inactive coordinates span (almost) the whole feature range and contribute
$\log I_{kj}\approx 0$ to the aggregation.
\end{definition}

\begin{theorem}[Certified clipping margin]
\label{prop:clip}
Fix a test point $x$ that in each active coordinate $j$ (Def.~\ref{def:active})
lies beyond the units' \emph{global envelope}, i.e.\
$x_j>\max_k\overline c_{kj}+\eta$ or $x_j<\min_k\underline c_{kj}-\eta$ for some
$\eta>0$; then $x$ is outside every box \emph{on the same side} per coordinate, so
$I_{kj}(x_j)\le\beta=\sigma(-\eta/\tau)$ for all $k$ and all active $j$. (A point
in an interior gap \emph{between} disjoint sub-intervals violates boxes on
opposite sides; such points fall under the graded suppression of
Cor.~\ref{cor:sparse}, not the full collapse below.)
\textbf{(i) Code collapse (encoder only).} The bottleneck code collapses to an
$x$-independent ``empty'' code $f_0$, $\|f(x)-f_0\|_2\le\kappa(\beta)$ with
$\kappa(\beta)=O(\beta)\to0$; this holds for any decoder.
\textbf{(ii) Certified bound.} If, in addition, the decoder is the
\emph{Certified-Lipschitz} \method{} of Sec.~\ref{sec:decoder} (spectrally
normalized linear maps and no non-$1$-Lipschitz layers), so that
$L=\prod_l\|W_l\|_2$ is an explicit Lipschitz upper bound, then
\[
\mathrm{MAE}(x)\;\ge\;\tfrac1d\big\|x-g_\theta(f_0)\big\|_1
\;-\;\tfrac{L}{\sqrt{d}}\,\kappa(\beta),
\]
a \emph{deterministic certificate}: any point far outside the learned support
incurs an error bounded below by its distance to the single point
$g_\theta(f_0)$ the model collapses to, minus a slack controllable by $\tau$.
\end{theorem}

\begin{corollary}[Graded suppression under partial violation]
\label{cor:sparse}
The collapse in Thm.~\ref{prop:clip} is the all-coordinates extreme of a
graded effect and does \emph{not} require a point to be out-of-support
everywhere. If $x$ violates (in a consistent direction) $s_k\ge1$ active
coordinates of unit $k$, then $\log I_{kj}\le\log\beta$ on each, so the
unnormalized weight obeys $\exp(\ell_k)\le\beta^{s_k}$: a unit is suppressed
geometrically in the number of its violated active coordinates, and the
bottleneck mass shifts toward the units a point violates \emph{least}. The full
collapse to a single $x$-independent $f_0$ is recovered only when $s_k=|A_k|$
for every unit (all active coordinates violated). We therefore do \emph{not}
claim a fixed reconstruction under partial violation, only that the matching
units are exponentially down-weighted, the counterpart of the
common anomaly profile in which a few features are abnormal.
\end{corollary}

Part (i) holds for \emph{any} decoder and is the mechanism the detector relies on;
the certificate (ii) is an optional, on-demand guarantee. The realistic
anomaly, a few abnormal features, is the partial-violation regime of
Cor.~\ref{cor:sparse}: the units a point matches are down-weighted geometrically
in the coordinates it violates, so the decoder cannot route the point's abnormal
values through them and the reconstruction is pulled toward the inlier image.
The all-coordinate collapse of Thm.~\ref{prop:clip} -- a single fixed
reconstruction with a certified error bound -- is the extreme of this effect,
used as a verifiable sanity check (Sec.~\ref{sec:certify}). In both regimes the
error grows with the point's distance from the inlier image, which is why MAE
works as a score here, specific to the interval bottleneck, not a generic
autoencoder property.
Fig.~\ref{fig:toy} makes this concrete on a 2D toy set: inliers inside the
learned box are reconstructed faithfully, while out-of-support points are
clipped onto the inlier image and thus separate cleanly in the error histogram.

\begin{remark}[Theory vs.\ the reported experiments]
\label{rem:spectral}
We are explicit about which part of Thm.~\ref{prop:clip} the results use.
Part~(i) (code collapse) is an encoder property and holds for \emph{every} run.
Part~(ii) (the certified error lower bound) additionally needs the spectrally
normalized \emph{Certified-Lipschitz} decoder; \textbf{all headline accuracy
numbers use the default unconstrained ReLU decoder}, for which part~(ii) is a
\emph{structural mechanism} and the collapse is an \emph{empirical} phenomenon
(Figs.~\ref{fig:framework}d,~\ref{fig:toy}), not a guarantee. The certified
variant makes it a verified inequality ($0$ violations on every dataset,
Sec.~\ref{sec:certify}) at a mean-AUROC cost of $0.908\!\to\!0.851$ on the six
datasets tested. The partial-violation regime a real anomaly usually occupies
is the graded suppression of Cor.~\ref{cor:sparse}, which carries no
reconstruction-error bound -- only the exponential down-weighting of the units
the point matches.
\end{remark}

\noindent\textbf{Checklist for Thm.~\ref{prop:clip}.} The hypotheses are
operational: $j$ is \emph{active} for $k$ iff its mean inlier membership
$\bar I_{kj}<0.9$ (it constrains inliers there), and $x$ is \emph{out-of-support}
for unit $k$ when its box membership $m_k(x)=\prod_j I_{kj}(x_j)$ falls below
$\beta=\sigma(-\eta/\tau)$ ($\eta{=}0.2$, $\beta{\approx}0.12$). Under this check
$\ge99\%$ of test anomalies are out-of-support of \emph{every} unit. The
\emph{strict} hypothesis of Thm.~\ref{prop:clip}(ii) -- violation of \emph{every}
active coordinate of some unit -- is stronger: it is met by a dataset-dependent
fraction (routine when units keep few active coordinates, rare when they keep
many; App.~\ref{sec:robust}, Table~\ref{tab:thresh-sens}), and the remaining
anomalies fall under the graded suppression of Cor.~\ref{cor:sparse}, which
down-weights the matched units but carries no reconstruction-error bound.
Empirically, the MAE inequality of Thm.~\ref{prop:clip}(ii) is nonetheless
satisfied for \emph{every} out-of-support point under the Certified-Lipschitz
decoder ($0$ violations on all six datasets; Sec.~\ref{sec:certify},
Table~\ref{tab:certify}, Fig.~\ref{fig:bridge}).

\begin{figure}[t]
\centering
\includegraphics[width=.74\linewidth]{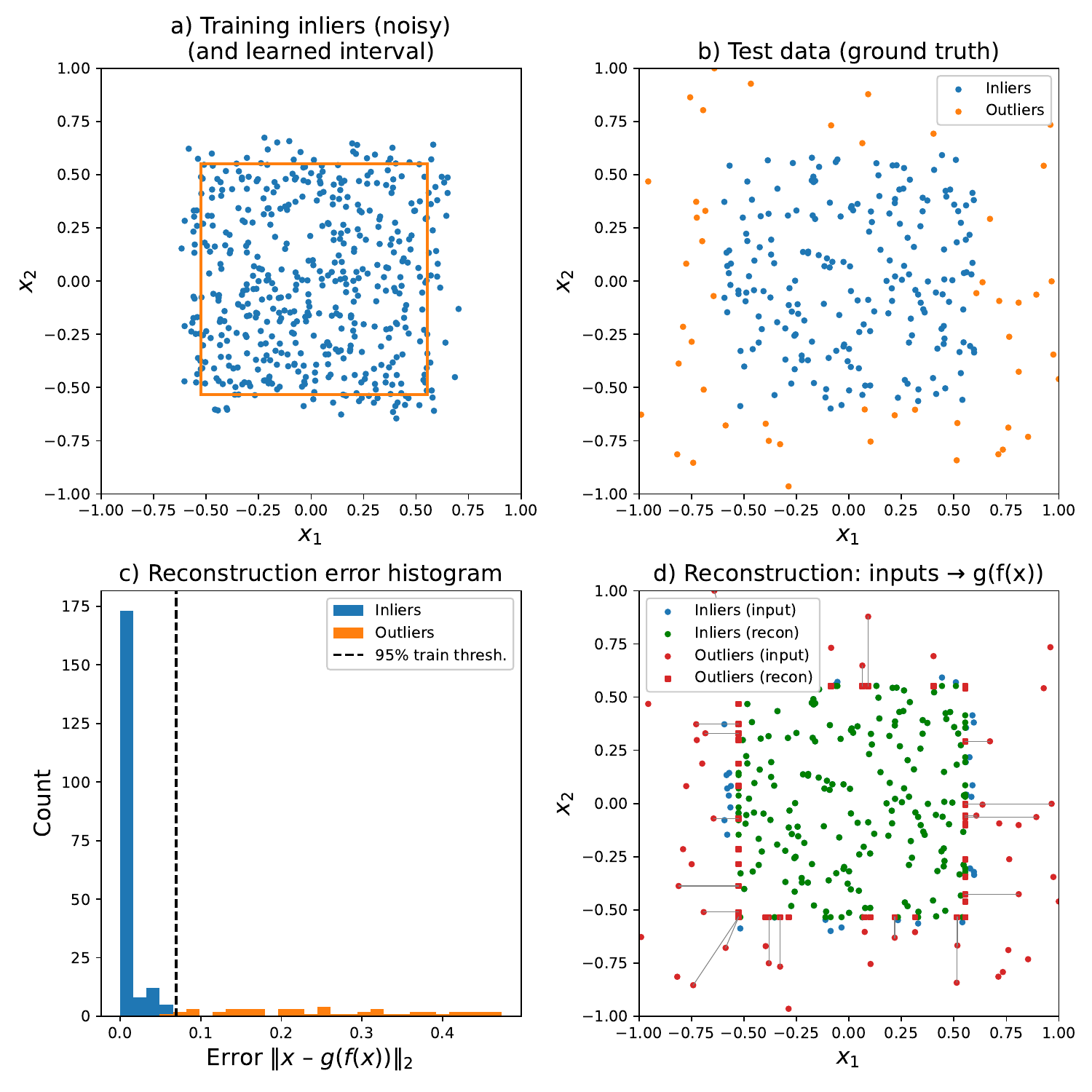}
\caption{Toy illustration of Theorem~\ref{prop:clip}. (a) noisy training
inliers and the learned interval; (b) test data; (c) reconstruction-error
histogram (inliers vs.\ outliers, dashed: 95\% train threshold);
(d) inputs $\to$ reconstructions: out-of-support points are clipped, producing
large error.}
\label{fig:toy}
\end{figure}

\noindent\textbf{Two further properties the method uses} are stated and proved
in App.~\ref{app:proofs}: the EMA support is asymptotically unbiased with finite
effective memory (Prop.~\ref{prop:ema}), and the \method{} encoder is globally
Lipschitz on $[-1,1]^d$ (Prop.~\ref{prop:lip}); with a spectrally normalized
decoder the whole model is Lipschitz, the regularity used by
Theorem~\ref{prop:clip}.

\begin{proposition}[Capacity: soft box-unions cover any support]
\label{prop:capacity}
Let the inlier support $D\subset[-1,1]^d$ be compact, write the soft box
membership of unit $k$ as $m_k(x)=\prod_j I_{kj}(x_j)\in[0,1]$ and the soft-union
$F=1-\prod_{k=1}^K(1-m_k)$. For every $\varepsilon>0$ there exist a finite $K$,
intervals $\{[\underline c_{kj},\overline c_{kj}]\}$ and a temperature $\tau>0$
with $\|F-\mathbf 1_D\|_{L^1([-1,1]^d)}\le\varepsilon$; moreover, for any
$\delta>0$, $F\to\mathbf 1_D$ \emph{uniformly} on
$\{x:\operatorname{dist}(x,\partial D)\ge\delta\}$ as $\tau\to0$.
\end{proposition}

\noindent Axis-aligned intervals are thus \emph{not} a representational
bottleneck on the \emph{shape} of the inlier region; the residual error
concentrates only in an arbitrarily thin shell around $\partial D$, where no
continuous score can resolve a hard indicator. This is a capacity statement:
whether training \emph{realizes} such a cover is empirical, answered by the
learned boxes localizing inliers (Fig.~\ref{fig:scatter}) and by accuracy
saturating after a few units (Fig.~\ref{fig:ablation}). The operative detection
result remains Thm.~\ref{prop:clip}.

\section{Turning Intervals into Auditable Candidate Constraints}
\label{sec:lfi}

The bottleneck is inspectable, but a practitioner needs a scalar per
(unit, feature) saying how diagnostic that interval is. The score also has to
work without anomaly labels, since the detector is unsupervised. The simplest
way to get one is to read it off the trained model itself.

\paragraph*{Label-free importance (LFI)}
LFI is a \emph{global}, model-level diagnostic: it scores a (unit, feature)
pair, not a single flagged instance. The per-instance question -- \emph{which
features made \textbf{this} point anomalous} -- is answered directly by
Cor.~\ref{cor:sparse}: the active coordinates the point violates are exactly
the ones that suppress its matching units and drive up its error.
For unit $k$ and feature $j$, write $s_{kj}\!\in\![0,1]$ for the EMA soft
support (Sec.~\ref{sec:decoder}), $\Delta_{kj}$ for the learned half-width,
and $\mathrm{range}_j{=}2$ for the $[-1,1]$ feature range. Let
$\mathrm{stab}_{kj}\!\in\![0,1]$ be the inverse spread of the interval bounds
across seeds. Latent units carry no canonical order, so $\mathrm{stab}_{kj}$ is
computed \emph{after} a per-dataset alignment of units across seeds (matched by
proximity of their $(m_{kj},\Delta_{kj})$ vectors); on a single-seed run it is
set to $1$ and the factor is inert. We score (unit, feature) pairs by
\[
\mathrm{LFI}_{kj}\;=\;
\underbrace{s_{kj}}_{\text{support}}\,\cdot\,
\underbrace{(1-\Delta_{kj}/\mathrm{range}_j)}_{\text{tightness}}\,\cdot\,
\underbrace{\mathrm{stab}_{kj}}_{\text{stability}}\;\in\;[0,1].
\]
The three factors are simply how often the interval fires on inliers
($s_{kj}$), how much narrower it is than the full feature range
($\Delta_{kj}$ small), and how repeatable it is across re-runs. None of them
need an anomaly label. To pick a unit to display, we sum the top-$N$ LFI
values across its features; within that unit, we keep the top-$N$ features by
LFI. This is the selection rule behind Fig.~\ref{fig:scatter}. Stability is
validated matching-free at the feature level (App.~\ref{sec:robust}: mean
cross-seed Spearman $\rho\!\ge\!0.74$).

\paragraph*{Faithfulness of the LFI explanation}
Post-hoc attributions such as SHAP~\cite{lundberg2017shap} or per-feature
residuals from TCCM~\cite{tccm2025} are computed after the fact and can
disagree with the score they explain. \method{} does not have this gap: a
point's violated constraints are exactly the suppressed units
(Cor.~\ref{cor:sparse}), and that suppression is what moves the reconstruction.
The explanation and the score share one mechanism, so no second model has to
align them. App.~\ref{sec:xai-quant} evaluates the \emph{global} LFI feature
ranking against the same ranking obtained by aggregating KernelSHAP, integrated
gradients, gradient$\times$input and ECOD over a pool of flagged anomalies,
inside the axis-aligned regime \method{} targets. LFI is a global
\emph{interpretability} readout, not a per-instance explainer; within that
regime the zero-query LFI ranking agrees with the aggregated KernelSHAP /
integrated-gradients rankings at Spearman $\rho\approx0.7$, is as stable across
re-trainings as KernelSHAP, and passes a model-randomisation sanity check --
i.e.\ it carries the same global information the post-hoc methods buy with
hundreds of model queries, at no cost.

\paragraph*{Is LFI a contribution?}
We treat LFI as a minor, auxiliary contribution. The formula is a product of
three quantities the model already keeps, not a new theoretical result. Its
role is to make the interval bottleneck usable when no anomaly labels are at
hand: a practitioner can rank candidate constraints from the trained model
alone. We do \emph{not} claim LFI is superior to supervised attributors such as
SHAP -- it is a global, label-free constraint-ranking tool, not a per-instance
explainer (App.~\ref{sec:xai-quant}). LFI is the score to reach for when labels
are not available, which is the usual setting in unsupervised anomaly detection.

\section{Experiments}
\label{sec:exp}

\paragraph*{Protocol (stated up front)}
We evaluate on the 48 native-numerical ADBench datasets~\cite{han2022adbench}
(12 small, 15 medium, 11 large, 10 high-dimensional; full list in App.~\ref{sec:data}) against 22 baselines via the TCCM/ADBench
codebase~\cite{tccm2025,han2022adbench}. The remaining ADBench tasks are
pre-computed vision/text embedding vectors whose coordinates are not
individually meaningful, so the per-feature readability that motivates
\method{} does not apply; we match the tabular scope of our closest baseline
TCCM~\cite{tccm2025}. \textbf{\method{} is trained
semi-supervised on inliers only}, with a $40\%$ test split, 10 seeds, and a
single configuration fixed across datasets (lr $5\!\times\!10^{-5}$, $K{=}200$,
$\tau{=}0.1$, $\rho{=}0.999$, 1000 epochs, 2-layer decoder; batch $64/512/1024$
by $n$). All methods share the MinMax $[-1,1]$ normalization and the inductive
baselines (TCCM, AutoEncoder, VAE) the same inlier-only training; this single
shared protocol is the one under which \method{}'s clipping bound and Lipschitz
slack transfer across datasets (Sec.~\ref{sec:norm}). Inlier-only training matches the deployment-relevant regime: a curated normal
reference set, as in monitoring and fraud. As this is a clean-training setup, we
restrict headline claims to inductive methods and report unsupervised
contaminated-data detectors as ``competitive,'' not ``beaten.'' Complete per-dataset tables, dataset characteristics, per-scale CD diagrams,
and ablations are in the appendices.\footnote{Code:
\url{https://github.com/DiffInt/diffint}.}

\paragraph*{Handling of missing baseline runs}
Some baselines did not complete on the largest datasets within budget (KPCA and
LMDD on $8$--$9$, MO\_GAAL on $6$, KDE/OCSVM/LUNAR on $3$--$4$, and
AutoEncoder/CBLOF/DIF/others on $1$ each) and appear as ``--''. We keep all $48$
datasets and assign each missing run the \emph{worst rank} (the number of
methods) on that dataset, the standard conservative convention; this mildly
favors methods that ran everywhere, \method{} included.

\subsection{Anomaly detection accuracy}
Fig.~\ref{fig:cd} shows the ROC--AUC critical-difference diagram over all 48
datasets, with runs that did not complete floored to the worst rank (Friedman
$\chi^2{=}454.36$, $p{<}10^{-80}$; Nemenyi at 95\%, CD$=5.01$). Under this
common $[-1,1]$ protocol, \method{} attains the \emph{best mean rank} (4.10),
ahead of LUNAR (6.00), AutoEncoder (7.42), TCCM (8.22), GMM (8.56), IForest
(8.60), and VAE (8.90). The Nemenyi bar links \method{} with this leading
cluster of seven methods, so \method{} is top-ranked but \emph{not} statistically
separated from them; the cluster is significantly ahead of the remaining
$\sim$16 detectors.
This pooled diagram mixes inlier-only and contaminated-data methods.
Table~\ref{tab:strat} re-analyzes the \emph{same} scores within regimes
(App.~\ref{sup:strat} for the per-scale breakdown): \method{} leads the
inlier-only pool by a wide margin ($1.69$/$1.71$ of $4$), and against the
contaminated pool on the complete-case subset (no rank imputation) it is first
or tied with LUNAR (AUPR gap $0.11$) -- so the headline ordering is not an
artefact of pooling regimes or of the missing-run convention.

\begin{table}[t]
\centering\footnotesize
\setlength{\tabcolsep}{4pt}
\caption{\textbf{Stratified and complete-case re-analysis}
(App.~\ref{sup:strat}). Mean rank recomputed \emph{within} training regimes
from the same per-dataset scores, no model re-run (lower is better).
\emph{Inlier-only pool}: \method{}, TCCM, AutoEncoder, VAE re-ranked $1$--$4$
per dataset. \emph{Contaminated pool}: \method{} inserted into the $19$-method
contaminated set ($20$ methods). \emph{Complete-case}: the $39/48$ datasets on
which every method finished (no imputation). The top-cluster result survives
both the stratification and the removal of imputation.}
\label{tab:strat}
\begin{tabular}{l cc}
\toprule
 & \shortstack{Rank\\(ROC--AUC)} & \shortstack{Rank\\(AUPR)}\\
\midrule
\multicolumn{3}{l}{\emph{Inlier-only pool} (of $4$), all $48$ datasets}\\
\quad\textbf{\method{}}      & \textbf{1.69} & \textbf{1.71}\\
\quad AutoEncoder            & 2.50 & 2.58\\
\quad TCCM                   & 2.83 & 2.62\\
\quad VAE                    & 2.98 & 3.08\\
\midrule
\multicolumn{3}{l}{\emph{\method{} vs.\ contaminated-data pool} (of $20$)}\\
\quad all $48$ datasets, \method{}          & 3.35 & 3.46\\
\quad complete-case ($39$), \method{}       & 3.59 & 3.74\\
\quad complete-case ($39$), best contam.\ (LUNAR) & 4.03 & 3.85\\
\bottomrule
\end{tabular}
\end{table}

\begin{figure*}[t]
\centering
\begin{subfigure}{0.49\linewidth}
\centering
\includegraphics[width=\linewidth]{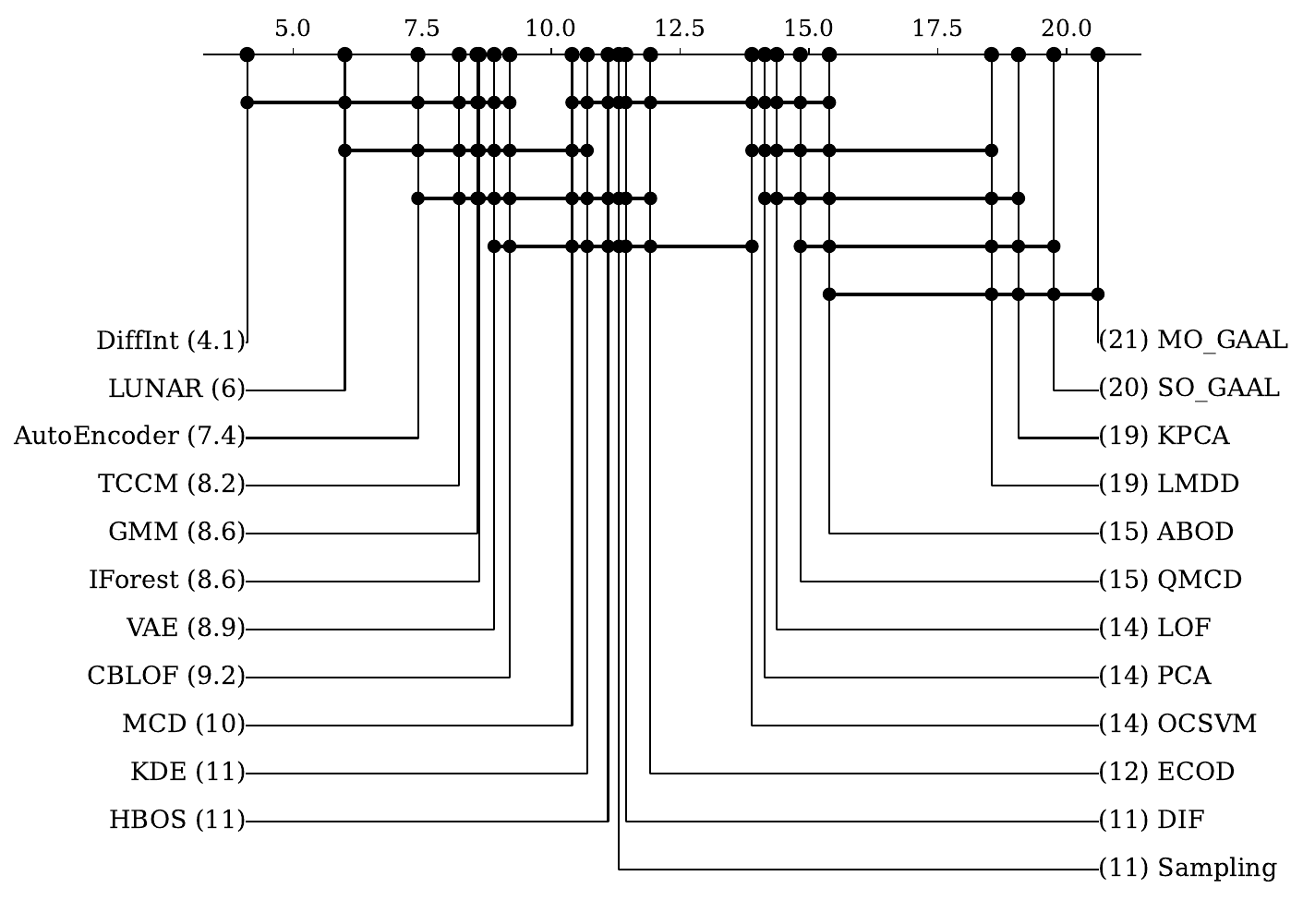}
\caption{ROC--AUC (mean rank \method{}: 4.10).}
\label{fig:cd-auc}
\end{subfigure}\hfill
\begin{subfigure}{0.49\linewidth}
\centering
\includegraphics[width=\linewidth]{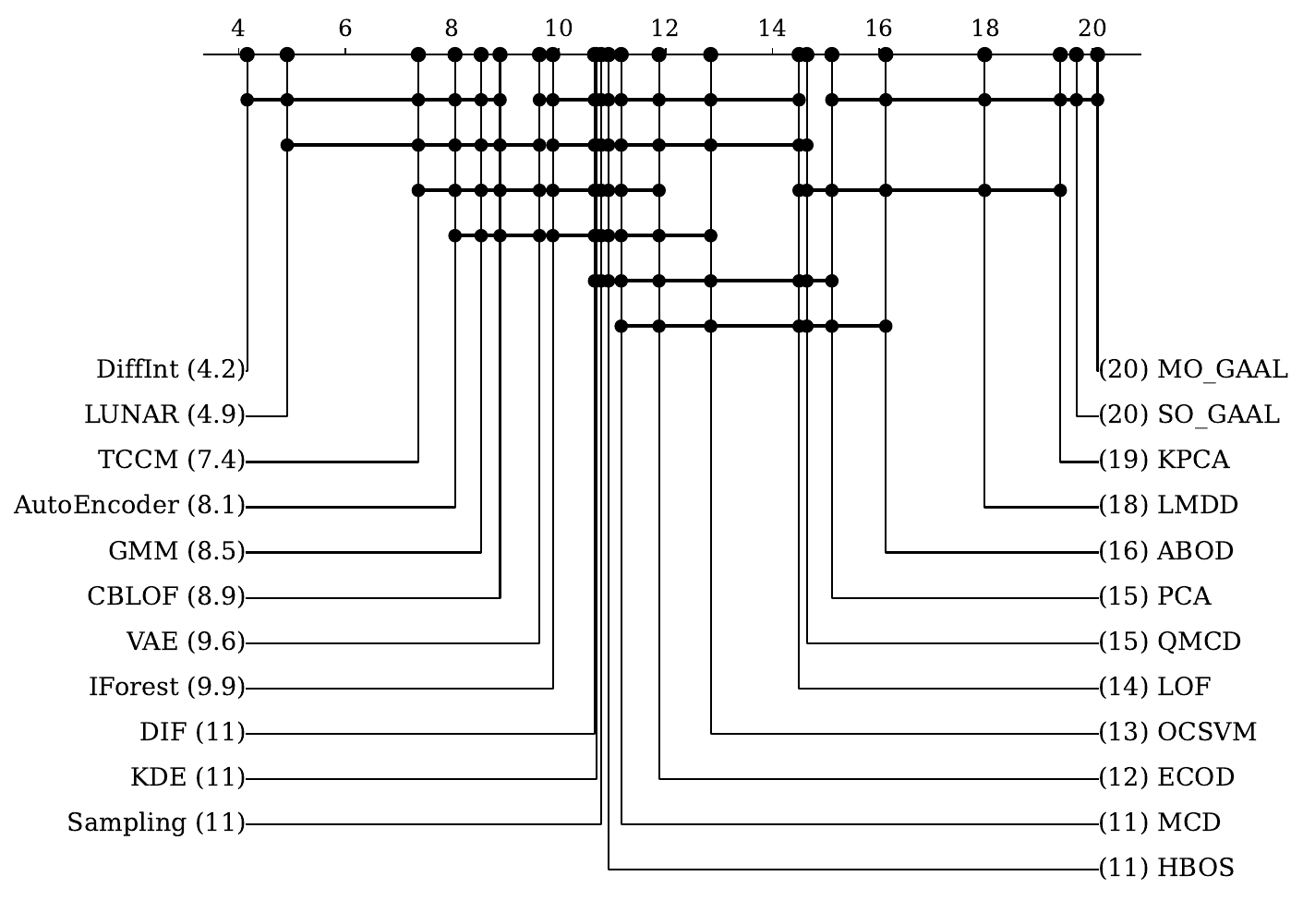}
\caption{AUPR (mean rank \method{}: 4.16).}
\label{fig:cd-aupr}
\end{subfigure}
\caption{Critical-difference diagrams across the $48$ ADBench datasets and
$23$ methods, runs that did not complete floored to worst rank. \method{}
attains the best mean rank on both metrics and is statistically tied with the
same leading cluster (LUNAR, AutoEncoder, TCCM, GMM, IForest, VAE); the
cluster is significantly ahead of the remaining $\sim$16 detectors.}
\label{fig:cd}
\end{figure*}

\begin{table}[t]
\centering
\caption{Mean rank by dataset scale (lower is better; missing runs floored to
worst). \textbf{Bold}: \method{} leads strictly. $^\dagger$: tied with the
best baseline within the per-scale Nemenyi CD ($\approx\!9$--$11$ at
$\alpha{=}0.05$). LU=LUNAR, IF=IForest, AE=AutoEncoder, CB=CBLOF.}
\label{tab:ranks}
\footnotesize
\setlength{\tabcolsep}{3pt}
\begin{tabular}{l@{\hspace{4pt}}cc@{\hspace{4pt}}cc}
\toprule
       & \multicolumn{2}{c}{Rank (ROC--AUC)} & \multicolumn{2}{c}{Rank (AUPR)}\\
\cmidrule(lr){2-3}\cmidrule(lr){4-5}
Regime ($N$) & \method{} & best BL & \method{} & best BL\\
\midrule
\textbf{Overall (48)} & \textbf{4.10} & LU 6.00 & \textbf{4.16} & LU 4.91\\
\midrule
Small (12)  & \textbf{4.42} & LU 6.75 & \textbf{4.33} & CB 6.67\\
Medium (15) & \textbf{3.77} & LU 4.23 & \textbf{3.57} & LU 3.60\\
Large (11)  & \textbf{3.18} & AE 4.73 & \textbf{2.73} & LU 3.91\\
High-d (10) & \textbf{5.25} & LU 5.75 & $6.40^\dagger$ & LU 5.45\\
\bottomrule
\end{tabular}
\end{table}

Table~\ref{tab:ranks} and the per-scale CD diagrams (App.~\ref{sec:auc}) agree:
\method{} attains the lowest mean rank strictly in $7$ of the $8$
(regime$\,\times\,$metric) cells of Table~\ref{tab:ranks}; the $8$th
(high-d AUPR, $6.40$ vs.\ LUNAR $5.45$) has LUNAR slightly ahead but the
gap is well within the per-scale Nemenyi CD ($\approx\!11$ at $n{=}10$), so
the two methods are statistically tied there. \method{} is therefore the
leading or co-leading method on every regime$\,\times\,$metric cell.
Six robustness checks are reported in App.~\ref{sec:robust}: sensitivity to
missing-run handling, a $\tau$ sweep, an active-coordinate threshold check,
the per-layer spectral norms of the certified decoder, an inference-clamping
isolation, and LFI stability across seeds. The top-cluster conclusion and
the interpretability claims hold under all of them.

\subsection{Reconstruction sanity check vs.\ a binary differentiable model}
Table~\ref{tab:binaps} compares reconstruction MAE against
BinaPs~\cite{fischer_differentiable_2021}, the closest differentiable model.
\textbf{Caveat:} BinaPs targets binary data and does not optimize numerical
reconstruction, so this is a sanity check that numerical intervals reconstruct
better than binarized patterns, \emph{not} a pattern-mining comparison and
not evidence about pattern-set quality. Under that caveat, \method{}
reconstructs numerical inliers far better (Table~\ref{tab:binaps}) and, avoiding
the combinatorial cost of binary pattern search, trains one to two orders of
magnitude faster than BinaPs.

\paragraph*{Cost relative to deep baselines}
Asymptotically \method{} is in the deep-baseline class: per-sample inference is
$O(Kd)$ memberships plus a fixed MLP decoder and an $O(K)$ softmax, one
forward pass with no neighbour search, unlike memory-based detectors such as
LUNAR whose $k$NN queries grow with $n$. Table~\ref{tab:binaps} confirms this in
wall-clock at a matched $1000$-epoch budget: \method{} trains faster than
LUNAR on most datasets and markedly on the largest (shuttle $194$s vs.\ $658$s).

\begin{table}[t]
\centering
\caption{Reconstruction MAE (vs.\ BinaPs, sanity check) and training time (s);
lower better, per-row best \green{green}. Times: matched $1000$-epoch budget for
\method{}/LUNAR; ``OOM'': out of memory. \method{} trains $1$--$2$ orders
faster than BinaPs and faster than LUNAR on most datasets.}
\label{tab:binaps}
\footnotesize
\setlength{\tabcolsep}{3pt}
\begin{tabular}{l cc c cc}
\toprule
& \multicolumn{2}{c}{Reconstruction MAE} && \multicolumn{2}{c}{Training time (s)}\\
\cmidrule(lr){2-3}\cmidrule(lr){5-6}
Dataset & BinaPs & \method{} && \method{} & LUNAR\\
\midrule
arrhythmia & 0.296 & \green{0.121} && 30.2 & \green{7.9}\\
cardio     & 0.305 & \green{0.064} && \green{3.0} & 20.6\\
glass      & 0.336 & \green{0.043} && 9.2 & \green{7.5}\\
ionosphere & 0.252 & \green{0.129} && 16.1 & \green{7.2}\\
optdigits  & 0.281 & \green{0.203} && 60.1 & \green{57.6}\\
pima       & 0.295 & \green{0.070} && 8.6 & \green{8.1}\\
satellite  & 0.247 & \green{0.035} && \green{38.7} & 52.6\\
satimage-2 & 0.257 & \green{0.033} && \green{54.6} & 65.7\\
shuttle    & 0.057 & \green{0.001} && \green{193.5} & 658.0\\
vowels     & OOM & \green{0.061} && \green{2.7} & 31.3\\
wbc        & OOM & \green{0.067} && \green{2.6} & 15.9\\
\bottomrule
\end{tabular}
\end{table}

\subsection{Effect of bottleneck width $K$}
Fig.~\ref{fig:ablation} varies $K\!\in\!\{5,10,20,50,100,200,300\}$ on three
representative datasets. ROC--AUC (\subref{fig:ablation-auc}) saturates well
below $K{=}200$ on every dataset: once enough capacity is available, adding
units neither helps nor hurts. The companion panel
(\subref{fig:ablation-outmae}) confirms \emph{selective reconstruction}: as $K$
grows, inlier reconstruction sharpens but the mean MAE on test \emph{outliers}
stays high, the empirical signature of the clipping margin
(Thm.~\ref{prop:clip}). We therefore fix $K{=}200$ as a generous capacity
default that absorbs dataset complexity without per-dataset tuning; it is
\emph{not} interpreted as a minimal pattern set. Practitioners on a tight
compute budget can safely halve $K$ without measurable accuracy loss.

\begin{figure*}[t]
\centering
\begin{subfigure}{0.49\linewidth}
\centering
\includegraphics[width=\linewidth]{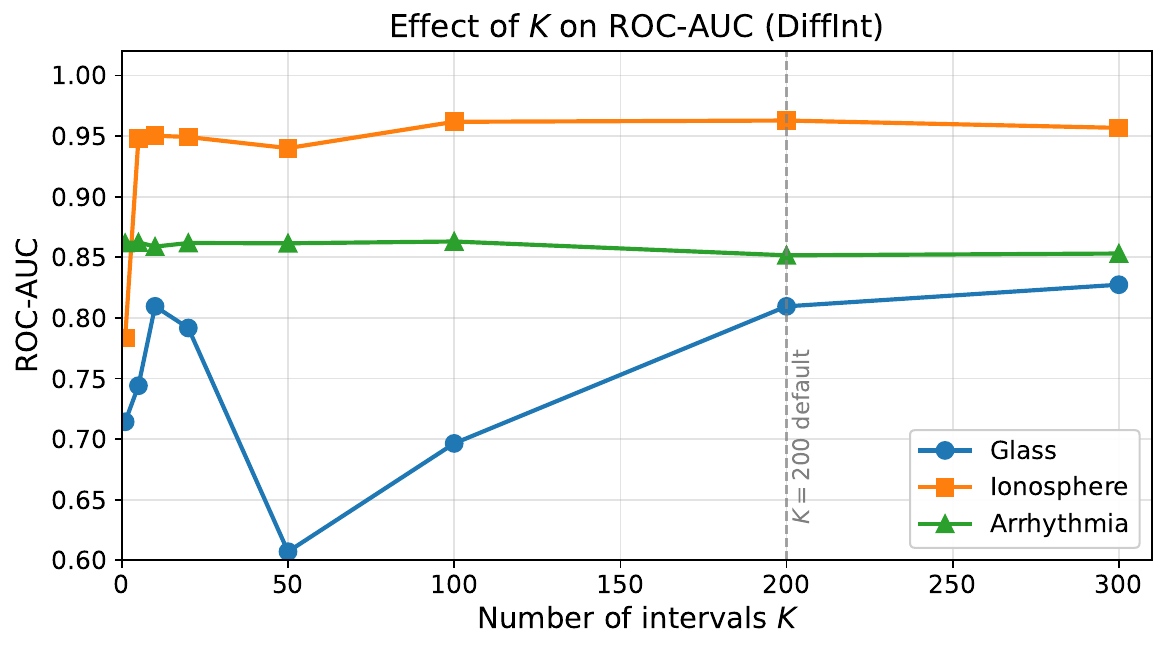}
\caption{ROC--AUC saturates after a few units.}
\label{fig:ablation-auc}
\end{subfigure}\hfill
\begin{subfigure}{0.49\linewidth}
\centering
\includegraphics[width=\linewidth]{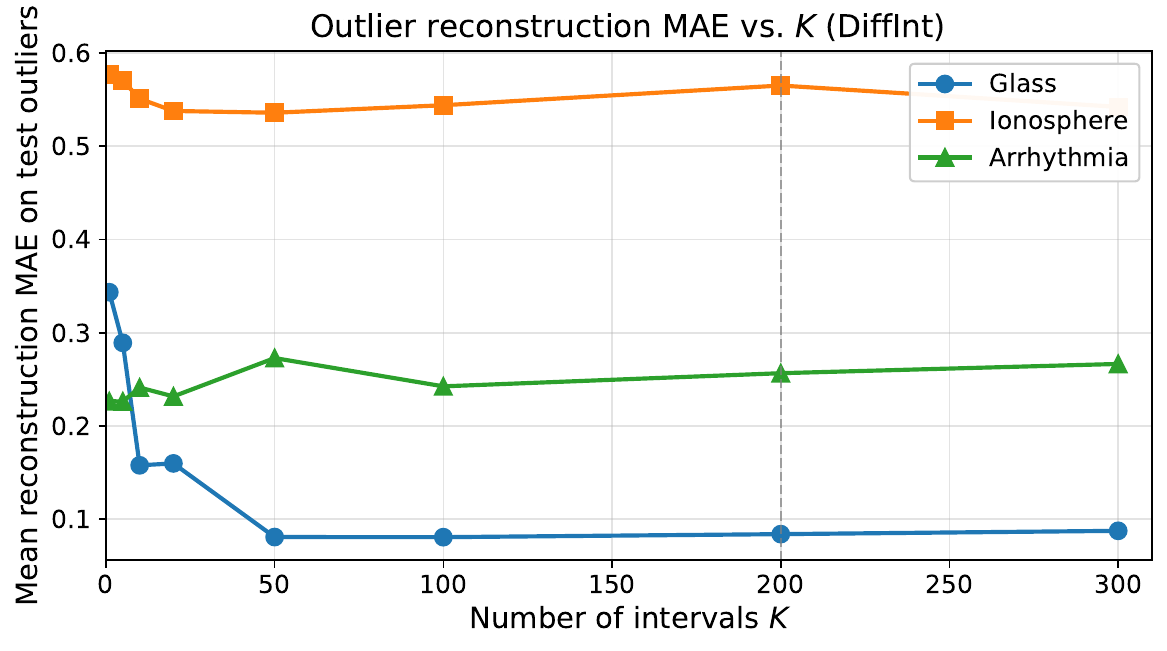}
\caption{Mean MAE on test outliers stays high.}
\label{fig:ablation-outmae}
\end{subfigure}
\caption{Effect of $K$ on three ADBench datasets. Detection accuracy saturates
quickly (a); meanwhile the mean reconstruction error on test outliers does not
collapse with growing capacity (b), so \method{} reconstructs the inlier
manifold faithfully while keeping out-of-support points far from their input
(selective reconstruction, consistent with Thm.~\ref{prop:clip}). We fix
$K{=}200$ as a capacity dial, not a pattern-set size.}
\label{fig:ablation}
\end{figure*}

\subsection{Ablations: normalization and contamination}
\label{sec:ablate}
Table~\ref{tab:ablate} reports two ablations on seven small ADBench
datasets ($n\!\ge\!100$; $K{=}200$, 3 seeds; Hepatitis, $n{=}80$, is too
small for a stable inlier-only split): preprocessing sensitivity and contaminated
``clean'' training. The Lipschitz assumption of Thm.~\ref{prop:clip} is examined
separately in Sec.~\ref{sec:certify}.
\textbf{(a) Normalization.} Mapping into a bounded symmetric domain is what
matters: dropping it costs $\approx3$ AUROC points, and among bounded/standardized
choices $[-1,1]$ is the \emph{most accurate} ($0.833$ vs.\ StandardScaler $0.818$,
$[0,1]$ $0.824$; a $\le\!1.5$-point spread), so it is both the \emph{theory-driven}
choice under which one configuration and the clipping bound transfer
(Sec.~\ref{sec:norm}) and the top performer.
\textbf{(b) Training contamination.} Anomalies in the ``inlier'' set degrade
AUROC gracefully ($0.833\!\to\!0.751\!\to\!0.735\!\to\!0.720$ at $0/5/10/20\%$),
staying well above chance: the EMA support averages over batches, so a minority
cannot dominate the intervals.

\begin{table}[t]
\centering
\caption{Ablations on seven small ADBench datasets ($n\!\ge\!100$; $K{=}200$,
3 seeds; mean ROC--AUC). (a) input normalization; (b) fraction of anomalies in the
training set.}
\label{tab:ablate}
\footnotesize
\setlength{\tabcolsep}{4pt}
\begin{tabular}{lcccc}
\toprule
\multicolumn{5}{l}{\textbf{(a) Normalization}}\\
            & none  & $[0,1]$ & standard & $[-1,1]$\\
ROC--AUC    & 0.805 & 0.824   & 0.818    & 0.833\\
\midrule
\multicolumn{5}{l}{\textbf{(b) Training contamination}}\\
            & 0\%   & 5\%   & 10\%  & 20\%\\
ROC--AUC    & 0.833 & 0.751 & 0.735 & 0.720\\
\bottomrule
\end{tabular}
\end{table}

\subsection{Certifying the clipping margin: does it hold?}
\label{sec:certify}
We test Thm.~\ref{prop:clip} directly with the Certified-Lipschitz \method{}
($K{=}200$, 3 seeds) on low-membership test points (inside no learned box).
Table~\ref{tab:certify} and Fig.~\ref{fig:bridge}: the bound is respected for
\emph{every} out-of-support point, the upper bound $L=\prod_l\|W_l\|_2$ is an
explicit constant near $1$, and the reconstructions collapse toward the
empty-code image $g_\theta(f_0)$ (within $\approx0.4$ of the inlier amplitude).
Certification trades some accuracy ($0.908\!\to\!0.851$ mean AUROC), so the
unconstrained decoder stays the default and the certified variant is used when a
guarantee is required; the clipping mechanism is present in both.

\begin{figure}[t]
\centering
\includegraphics[width=\linewidth]{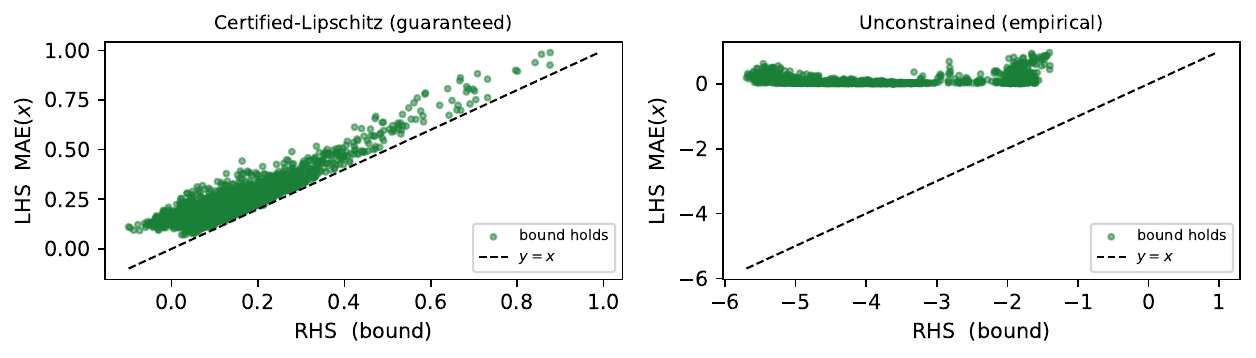}
\caption{Theorem-to-practice bridge: $\mathrm{MAE}(x)$ (LHS) vs.\ the
clipping-margin lower bound (RHS) for out-of-support points. \emph{Left}:
certified decoder, bound tight and respected ($0$ violations). \emph{Right}:
unconstrained, the inequality holds but the bound is loose and uncertified (note
the $x$-scale), an empirical phenomenon (Rem.~\ref{rem:spectral}).}
\label{fig:bridge}
\end{figure}

\begin{table}[t]
\centering\footnotesize
\caption{Certifying Thm.~\ref{prop:clip} (Certified-Lipschitz \method{}).
$L$: decoder Lipschitz upper bound; \emph{collapse}: mean distance of
out-of-support reconstructions to $g_\theta(f_0)$, as a fraction of inlier
amplitude (lower $=$ tighter collapse); \emph{bound}: fraction of out-of-support
points satisfying the inequality. AUROC is for the unconstrained (default) and
certified decoders.}
\label{tab:certify}
\setlength{\tabcolsep}{4pt}
\begin{tabular}{l cc c c c}
\toprule
& \multicolumn{2}{c}{AUROC} & & & \\
\cmidrule(lr){2-3}
Dataset & unc. & cert. & $L$ & collapse & bound \\
\midrule
glass      & 0.891 & 0.744 & 1.05 & 0.52 & 100\% \\
wbc        & 0.930 & 0.903 & 1.03 & 0.29 & 100\% \\
ionosphere & 0.953 & 0.828 & 1.04 & 0.24 & 100\% \\
pima       & 0.696 & 0.686 & 1.04 & 0.53 & 100\% \\
cardio     & 0.976 & 0.955 & 1.06 & 0.45 & 100\% \\
satimage-2 & 0.999 & 0.987 & 1.07 & 0.38 & 100\% \\
\midrule
Mean       & 0.908 & 0.851 & 1.05 & 0.40 & 100\% \\
\bottomrule
\end{tabular}
\end{table}

\subsection{Interpretability: do the intervals localize inliers?}
Fig.~\ref{fig:scatter} projects the top \method{} interval onto feature pairs.
The interval is picked by the label-free importance LFI (Sec.~\ref{sec:lfi}),
so the selection itself uses no anomaly labels. The bounds the model learns
form an explicit, auditable candidate constraint that concentrates inliers
and suppresses outliers through the clipping margin (Thm.~\ref{prop:clip},
Cor.~\ref{cor:sparse}).

\begin{figure*}[!t]
\centering
\begin{subfigure}{0.49\linewidth}
\centering
\includegraphics[width=0.9\linewidth]{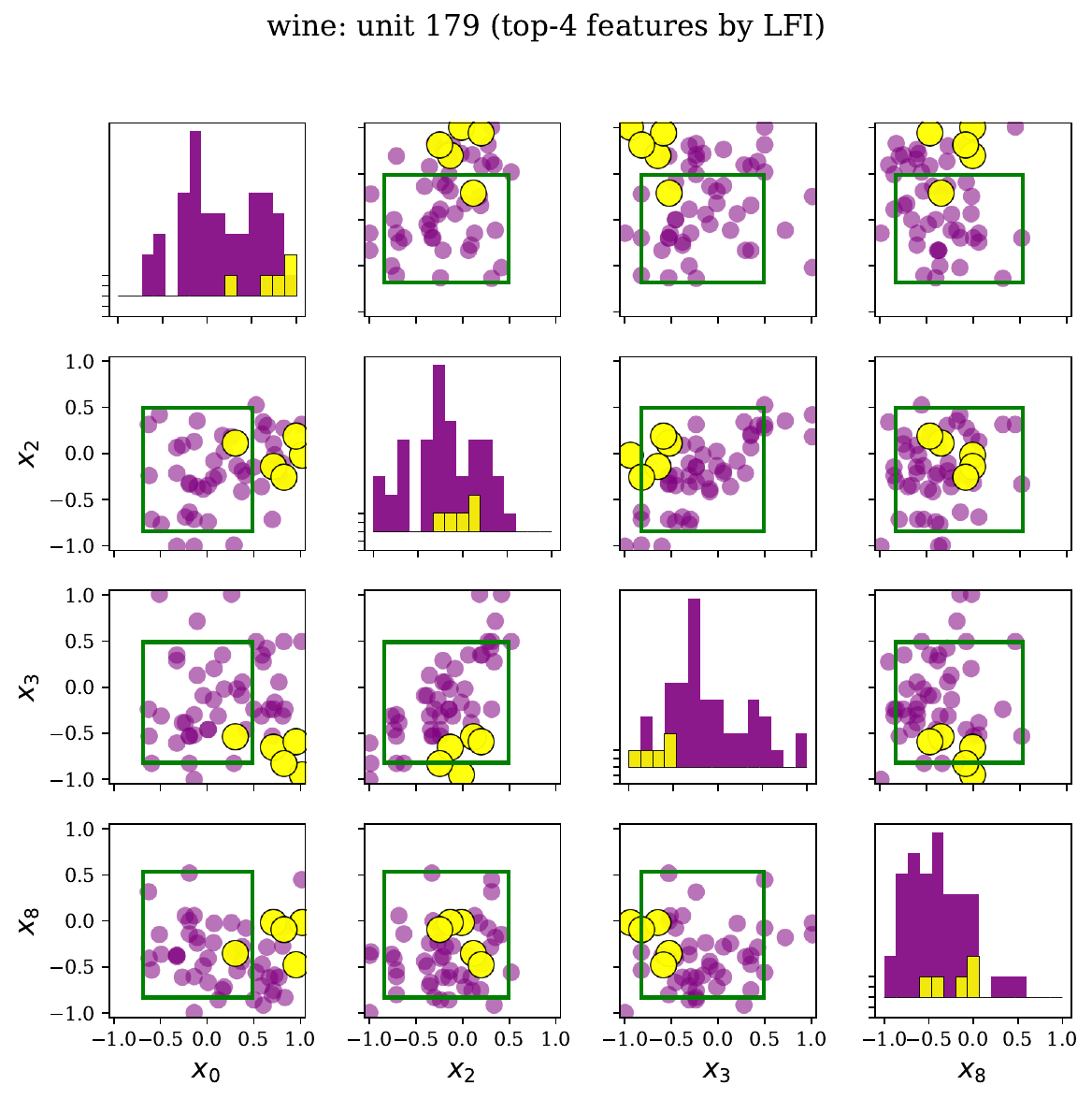}
\caption{\textsc{Wine}: unit $179$, candidate constraint
$x_{0}\!\in\![-0.68,0.48]\,\wedge\,x_{2}\!\in\![-0.84,0.50]\,\wedge\,
x_{3}\!\in\![-0.82,0.49]\,\wedge\,x_{8}\!\in\![-0.83,0.53]$.}
\label{fig:scatter-wine}
\end{subfigure}\hfill
\begin{subfigure}{0.49\linewidth}
\centering
\includegraphics[width=0.9\linewidth]{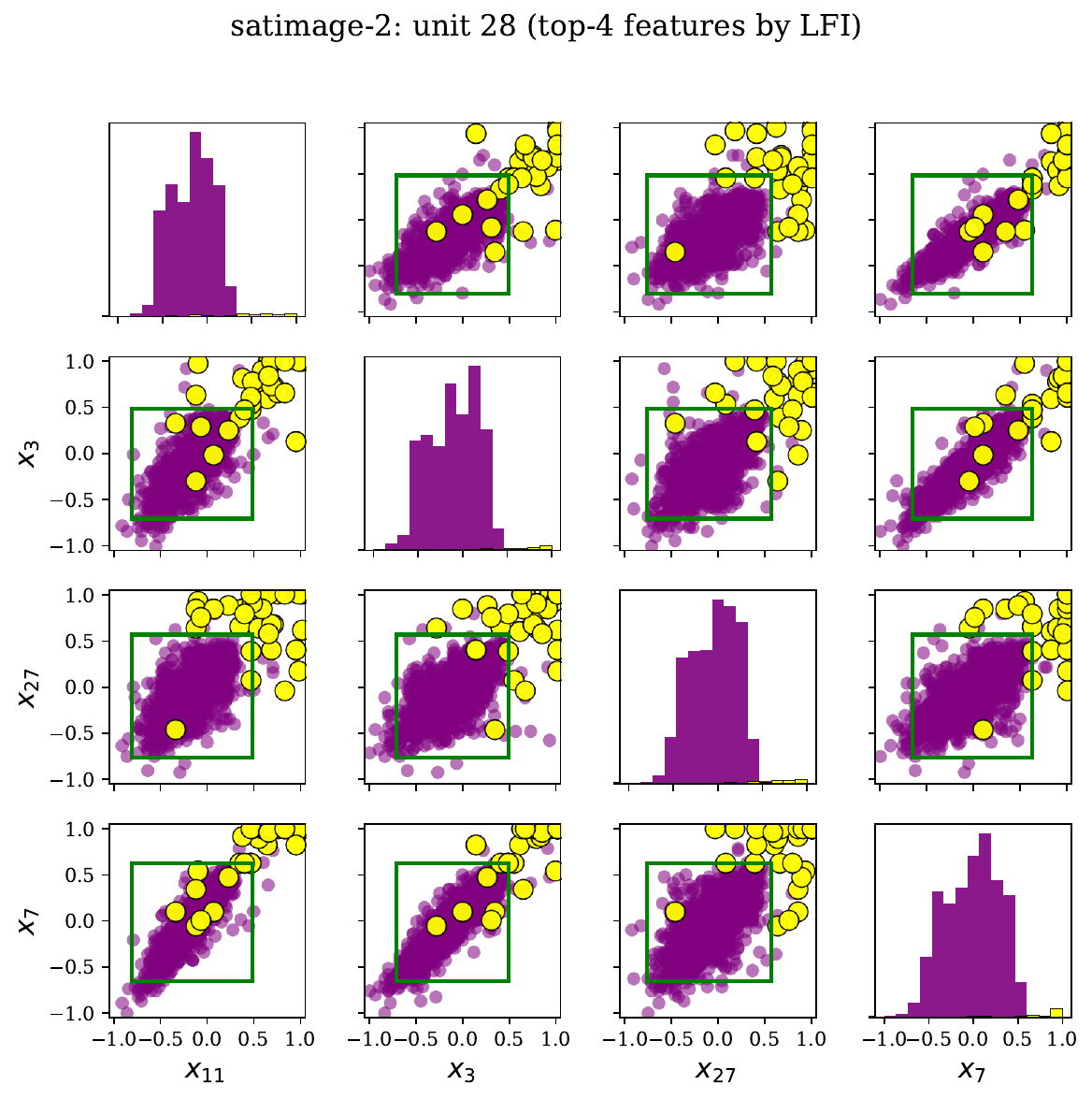}
\caption{\textsc{Satimage-2}: unit $28$, candidate constraint
$x_{3}\!\in\![-0.71,0.49]\,\wedge\,x_{7}\!\in\![-0.65,0.63]\,\wedge\,
x_{11}\!\in\![-0.80,0.48]\,\wedge\,x_{27}\!\in\![-0.76,0.57]$.}
\label{fig:scatter-sat}
\end{subfigure}
\caption{Pairwise 2D projections of the top \method{} interval on two ADBench
datasets, selected \emph{without any anomaly label} by the label-free
importance (LFI, Sec.~\ref{sec:lfi}); bounds are in the $[-1,1]$-scaled domain.
The green boxes concentrate the inlier mass (purple) while a large fraction of
outliers (yellow) leave the box on at least one axis, most visibly in the
upper-right cluster of \textsc{Satimage-2}. Separation at the box boundary
is not strict, but a violation in any active coordinate is enough to suppress
this unit's contribution (Cor.~\ref{cor:sparse}) and trigger the clipping
margin (Thm.~\ref{prop:clip}).}
\label{fig:scatter}
\end{figure*}

\subsection{Discussion: when does interval structure help?}
Theory and experiment agree on one mechanism: the clipping margin
(Thm.~\ref{prop:clip}, Cor.~\ref{cor:sparse}) scales with the fraction of
\emph{informative} coordinates, so the gains are largest on small and medium
data, where the axis-aligned bias matches the typical anomaly profile of a few
abnormal features. The gains carry to large data (rank $3.18$ on AUC, $2.73$
on AUPR) and to high-dimensional data, where \method{} still leads on ROC--AUC
($5.25$ vs.\ LUNAR $5.75$); the only cell where a competitor (LUNAR) edges
\method{} in mean rank is high-dimensional AUPR, and that gap is within the
per-scale Nemenyi CD, i.e.\ statistically a tie. On top of this, the structured bottleneck delivers what an
entangled code cannot: a \emph{mechanistic} reason \emph{why} a point is anomalous,
with auditable, feature-grounded candidate constraints (Sec.~\ref{sec:lfi},
Fig.~\ref{fig:scatter}).

\subsection{Limitations}
\label{sec:limits}
Three settings sit outside the assumptions under which \method{} is evaluated.
(1) \emph{Heavy-tailed features}: min--max scaling assumes finite per-feature
extrema. Features with heavy tails (e.g.\ income, latency) saturate the
endpoints and warp the $[-1,1]$ geometry the temperature $\tau$ and the
clipping margin rely on; a robust quantile mapping into $[-1,1]$ would preserve
$\tau$, but we do not evaluate it here.
(2) \emph{Concept drift}: the protocol assumes a stationary inlier distribution.
On non-stationary streams the scaler and the EMA support must be refreshed (a
short re-fit on a recent inlier window suffices in principle), but drift is
unevaluated in this paper.
(3) \emph{Missing values}: the current implementation assumes complete
features and treats imputation as a preprocessing step. A missing-aware
membership (e.g.\ a neutral $I_{kj}{=}1$ for missing $x_j$, so the unit
abstains rather than firing) is straightforward but not measured here.
A quantitative comparison of the global LFI ranking against aggregated
post-hoc attributors is in App.~\ref{sec:xai-quant}; two further points -- a
robust/quantile normalization ablation and a comparison to diffusion /
retrieval / tabular-foundation-model detectors -- are outside the scope of
this evaluation and discussed in App.~\ref{sup:rest}.

\section{Conclusion}
We presented \method{}, an autoencoder for tabular anomaly detection whose
bottleneck is a set of soft, learnable numerical intervals. Three pieces fit
together: a differentiable architecture that needs no discretization, a
clipping-margin result (Thm.~\ref{prop:clip}, certified under a Lipschitz
decoder and empirical otherwise) explaining why reconstruction error is a
valid score on this bottleneck, and a closed-form label-free importance, LFI
(Sec.~\ref{sec:lfi}), that turns each (unit, feature) pair into an auditable
candidate constraint without consulting any anomaly label. On 48 ADBench
datasets, \method{} reaches the lowest mean rank on both metrics ($4.10$
ROC--AUC, $4.16$ AUPR) and sits in a seven-method cluster the Nemenyi test
cannot separate; LUNAR slightly edges it only on high-dimensional AUPR, and
that gap is within the per-scale CD. The explanation and the score share one
mechanism, so the intervals read off the bottleneck are exactly what drives
detection. A natural next step is to add a soft-$L_0$ gate and an MDL term so
the bottleneck becomes a compact \emph{interval-pattern-set} miner suitable
for pattern-mining benchmarks alongside anomaly detection.

\section*{Acknowledgment}
This work was supported by French state aid managed by the National Research
Agency (ANR) under the France 2030 program, with the reference ``PANDORA''
(ANR-24-CE23-0950). The authors thank the anonymous reviewers for their
constructive feedback.

\bibliographystyle{IEEEtran}
\bibliography{references}

\appendices
\section{Proofs}
\label{app:proofs}
We give the complete proofs of the results stated (with sketches) in
Sec.~\ref{sec:theory}.

\subsection{Proof of Theorem~\ref{prop:clip} (clipping margin)}
\begin{proof}
Write the code $f_k(x)=\mathrm{softmax}_k(\ell_k)$,
$\ell_k=\sum_j\log I_{kj}(x_j)$. Let $V$ be the set of active coordinates on
which $x$ exceeds the unit in a fixed direction (say above the upper bound) by
margin $\ge\eta$, so $I_{kj}(x_j)\le\beta=\sigma(-\eta/\tau)\in(0,\tfrac12)$ for
$j\in V$. There the upper sigmoid saturates and
$\log I_{kj}(x_j)=-\tfrac1\tau(x_j-\overline c_{kj})+O(\beta)$, whose
$x$-dependent part $-x_j/\tau$ is \emph{identical across all units} $k$ (this
uses the same-side hypothesis: for a below-violation the common part is
$+x_j/\tau$, still unit-independent). Since here \emph{all} active coordinates
are violated, the only non-saturated terms are the inactive ones, which
contribute $\log I_{kj}\approx0$. Hence $\ell_k=c(x)+b_k+O(|V|\beta)$ with
$c(x)=-\tfrac1\tau\sum_{j\in V}x_j$ common to all $k$ and
$b_k=\tfrac1\tau\sum_{j\in V}\overline c_{kj}$ $x$-independent. As softmax is
shift-invariant, $f(x)=\mathrm{softmax}(b+O(|V|\beta))$, so
$\|f(x)-f_0\|_2\le\kappa(\beta)=O(|V|\beta)$ with $f_0=\mathrm{softmax}(b)$ the
fixed $x$-independent empty-activation code. Lipschitzness of $g_\theta$ gives
$\|g_\theta(f(x))-g_\theta(f_0)\|_2\le L\,\kappa(\beta)$, and the reverse
triangle inequality with $\|\cdot\|_1\le\sqrt d\,\|\cdot\|_2$,
$\|x-g_\theta(f(x))\|_1\ge\|x-g_\theta(f_0)\|_1-\sqrt d\,L\,\kappa(\beta)$,
divided by $d$, yields $\mathrm{MAE}(x)\ge\tfrac1d\|x-g_\theta(f_0)\|_1
-\tfrac{L}{\sqrt d}\kappa(\beta)$.

\emph{Partial violations (Cor.~\ref{cor:sparse}).} If only $s_k$ active
coordinates of unit $k$ are violated, the remaining active coordinates
contribute $x$-dependent, unit-specific terms, so $f(x)$ is \emph{not} forced to
a single $x$-independent code. What survives is the per-unit bound
$\ell_k=\sum_j\log I_{kj}\le\sum_{j\in V\cap A_k}\log I_{kj}\le s_k\log\beta$
(every $\log I_{kj}\le0$), i.e.\ $\exp(\ell_k)\le\beta^{s_k}$: each unit's
unnormalized weight is suppressed geometrically in its number of violated active
coordinates, which is the graded statement of Cor.~\ref{cor:sparse}.
\end{proof}
Fig.~\ref{fig:toy} illustrates this: inliers are preserved while
out-of-support points are collapsed and thus incur a large error.

\subsection{EMA support and encoder regularity}
\begin{proposition}[EMA support: asymptotically unbiased, finite memory]
\label{prop:ema}
If the per-step support $Z_t\!\in\![0,1]$ is strictly stationary with mean
$\mu$, the recursion $S_{t+1}=\rho S_t+(1-\rho)Z_{t+1}$ ($\rho\!\in\![0,1)$) has
a unique stationary limit $S_\infty=(1-\rho)\sum_{m\ge0}\rho^mZ_{-m}$ with
$\mathbb{E}[S_\infty]=\mu$, i.i.d.\ variance $\sigma^2\tfrac{1-\rho}{1+\rho}$,
effective sample size $\tfrac{1+\rho}{1-\rho}$, and transient
$\mathbb{E}[S_t]-\mu=\rho^t(S_0-\mu)$.
\end{proposition}
\begin{proof}
$\mathcal{T}(X)=\rho X+(1-\rho)Z$ is an $L^2$ contraction ($\rho<1$), with
unique fixed point the absolutely summable moving average
$S_\infty=(1-\rho)\sum_{m\ge0}\rho^mZ_{-m}$. Linearity gives
$\mathbb{E}[S_\infty]=\mu$; bilinearity of covariance with
$\gamma(h)=\sigma^2\mathbb{1}[h{=}0]$ gives
$\mathrm{Var}=\sigma^2(1-\rho)^2\sum_m\rho^{2m}=\sigma^2\tfrac{1-\rho}{1+\rho}$,
i.e.\ $n_{\mathrm{eff}}=\tfrac{1+\rho}{1-\rho}$; and unrolling
$S_t=\rho^tS_0+(1-\rho)\sum_{k=1}^t\rho^{t-k}Z_k$ gives the transient
$\mathbb{E}[S_t]-\mu=\rho^t(\mathbb{E}[S_0]-\mu)$. (The same holds for dependent
stationary $Z$ with summable autocovariance.)
\end{proof}

\begin{proposition}[Bounded encoder sensitivity]
\label{prop:lip}
The \method{} encoder is globally Lipschitz on $[-1,1]^d$ with constant
$\tfrac{K\,d}{\tau}$; the constant is expressed in the interpretable model
quantities $(K, d, \tau)$.
\end{proposition}
\begin{proof}
Let $x\in[-1,1]^d$. With $a_{kj}=(x_j-\underline c_{kj})/\tau$ and
$b_{kj}=(\overline c_{kj}-x_j)/\tau$,
$\partial_{x_j}\log I_{kj}=\tfrac1\tau(\sigma(b_{kj})-\sigma(a_{kj}))$, hence
$|\partial_{x_j}\log I_{kj}|\le\tfrac1\tau$ and
$\|\nabla_x\ell_k\|\le\tfrac{d}{\tau}$ for $\ell_k=\sum_j\log I_{kj}$. As
$f_k=\mathrm{softmax}_k(\ell)$ is $1$-Lipschitz in $\ell$,
$\|\nabla_x f_k\|\le\tfrac{d}{\tau}$, and summing over the $K$ outputs gives
$\|\nabla_x f\|\le\tfrac{Kd}{\tau}$. A spectrally normalized decoder then makes
the full autoencoder Lipschitz, the regularity invoked by Thm.~\ref{prop:clip}.
\end{proof}

\subsection{Proof of Proposition~\ref{prop:capacity} (capacity)}
\begin{proof}
Since $D$ is compact it is Jordan-measurable, so for any $\varepsilon>0$ there is
a finite union of axis-aligned dyadic boxes $C=\bigcup_{k=1}^K B_k$, with
$B_k=\prod_j[\underline c_{kj},\overline c_{kj}]$, such that
$\lambda(D\triangle C)<\varepsilon/2$ ($\lambda$ Lebesgue measure); take these
boxes as the units' intervals. As $\tau\to0$,
$I_{kj}(x_j)=\sigma(\tfrac{x_j-\underline c_{kj}}{\tau})\,
\sigma(\tfrac{\overline c_{kj}-x_j}{\tau})\to
\mathbf 1\{x_j\in[\underline c_{kj},\overline c_{kj}]\}$ pointwise except on the
box faces (a null set), monotonically; hence $m_k\to\mathbf 1_{B_k}$ and
$F=1-\prod_k(1-m_k)\to\mathbf 1_{C}$ pointwise a.e. Because $0\le F\le1$,
dominated convergence gives a $\tau$ with
$\|F-\mathbf 1_C\|_{L^1}<\varepsilon/2$. With
$\|\mathbf 1_C-\mathbf 1_D\|_{L^1}=\lambda(D\triangle C)$, the triangle
inequality yields $\|F-\mathbf 1_D\|_{L^1}<\varepsilon$. Finally, on
$\{\operatorname{dist}(\cdot,\partial D)\ge\delta\}$, refining the dyadic grid so
the box faces lie within the shell makes the target locally constant and the
boundary sigmoids saturate uniformly, giving uniform convergence there.
\end{proof}

\definecolor{rankone}{HTML}{1A7A34}   
\definecolor{ranktwo}{HTML}{1F5FA8}   
\definecolor{rankthree}{HTML}{C25E00} 
\renewcommand{\green}[1]{\textcolor{rankone}{\textbf{#1}}}
\renewcommand{\blue}[1]{\textcolor{ranktwo}{\textbf{#1}}}
\renewcommand{\lightred}[1]{\textcolor{rankthree}{\textbf{#1}}}
\setcounter{topnumber}{4}
\setcounter{totalnumber}{6}
\renewcommand{\topfraction}{0.99}
\renewcommand{\textfraction}{0.01}
\renewcommand{\floatpagefraction}{0.75}

\section{Supplementary: overview}
This appendix reproduces the standalone supplementary material: the full
dataset and protocol details, per-dataset robustness ablations, the
interpretability-evidence suite, complexity and scalability, the stratified
and complete-case re-analysis, and the \emph{complete} per-dataset ROC--AUC
and AUPR tables with per-scale critical-difference diagrams.

\section{Datasets and experimental protocol}
\label{sec:data}

\subsection{Datasets}
We evaluate on the 48 numerical datasets of the ADBench benchmark, grouped by
scale into \emph{small} (12), \emph{medium} (15), \emph{large} (11), and
\emph{high-dimensional} (10). Tables~\ref{tab:chars-small}--\ref{tab:chars-high}
list, per dataset, the number of samples, features, and outliers, the outlier
percentage, and the batch size used by \method{} ($64$ for $n\!\le\!10\text{k}$,
$512$ for $10\text{k}\!<\!n\!\le\!20\text{k}$, $1024$ otherwise). The suite spans
four orders of magnitude in sample size ($n\!\in\![80,\,6.2\!\times\!10^5]$) and
from $3$ to $1{,}555$ features, so it stresses both data-scarce and
high-dimensional regimes.

\begin{table}[ht]
\centering\footnotesize
\caption{Small-scale datasets.}
\label{tab:chars-small}
\begin{tabular}{l r r r r r}
\toprule
\textbf{Dataset} & \textbf{Samples} & \textbf{Feat.} & \textbf{\#Out.} & \textbf{Out.\,\%} & \textbf{Batch} \\
\midrule
4\_breastw       & 683  & 9  & 239 & 34.99 & 64 \\
14\_glass        & 214  & 7  & 9   & 4.21  & 64 \\
15\_Hepatitis    & 80   & 19 & 13  & 16.25 & 64 \\
18\_Ionosphere   & 351  & 33 & 126 & 35.90 & 64 \\
21\_Lymphography & 148  & 18 & 6   & 4.05  & 64 \\
29\_Pima         & 768  & 8  & 268 & 34.90 & 64 \\
37\_Stamps       & 340  & 9  & 31  & 9.12  & 64 \\
39\_vertebral    & 240  & 6  & 30  & 12.50 & 64 \\
42\_WBC          & 223  & 9  & 10  & 4.48  & 64 \\
43\_WDBC         & 367  & 30 & 10  & 2.72  & 64 \\
45\_wine         & 129  & 13 & 10  & 7.75  & 64 \\
46\_WPBC         & 198  & 33 & 47  & 23.74 & 64 \\
\bottomrule
\end{tabular}
\end{table}

\begin{table}[ht]
\centering\footnotesize
\caption{Medium-scale datasets.}
\label{tab:chars-medium}
\begin{tabular}{l r r r r r}
\toprule
\textbf{Dataset} & \textbf{Samples} & \textbf{Feat.} & \textbf{\#Out.} & \textbf{Out.\,\%} & \textbf{Batch} \\
\midrule
2\_annthyroid       & 7,200 & 6  & 534   & 7.42  & 64 \\
6\_cardio           & 1,831 & 21 & 176   & 9.61  & 64 \\
7\_Cardiotocography & 2,114 & 21 & 466   & 22.04 & 64 \\
12\_fault           & 1,941 & 27 & 673   & 34.67 & 64 \\
19\_landsat         & 6,435 & 36 & 1,333 & 20.71 & 64 \\
20\_letter          & 1,600 & 32 & 100   & 6.25  & 64 \\
27\_PageBlocks      & 5,393 & 10 & 510   & 9.46  & 64 \\
28\_pendigits       & 6,870 & 16 & 156   & 2.27  & 64 \\
30\_satellite       & 6,435 & 36 & 2,036 & 31.64 & 64 \\
31\_satimage-2      & 5,803 & 36 & 71    & 1.22  & 64 \\
38\_thyroid         & 3,772 & 6  & 93    & 2.47  & 64 \\
40\_vowels          & 1,456 & 12 & 50    & 3.43  & 64 \\
41\_Waveform        & 3,443 & 21 & 100   & 2.90  & 64 \\
44\_Wilt            & 4,819 & 5  & 257   & 5.33  & 64 \\
47\_yeast           & 1,500 & 8  & 20    & 1.33  & 64 \\
\bottomrule
\end{tabular}
\end{table}

\begin{table}[ht]
\centering\footnotesize
\caption{Large-scale datasets.}
\label{tab:chars-large}
\begin{tabular}{l r r r r r}
\toprule
\textbf{Dataset} & \textbf{Samples} & \textbf{Feat.} & \textbf{\#Out.} & \textbf{Out.\,\%} & \textbf{Batch} \\
\midrule
1\_ALOI        & 49,534  & 27 & 1,508  & 3.04  & 1,024 \\
8\_celeba      & 202,599 & 39 & 4,547  & 2.24  & 1,024 \\
10\_cover      & 286,048 & 10 & 2,747  & 0.96  & 1,024 \\
11\_donors     & 619,326 & 10 & 36,710 & 5.93  & 1,024 \\
13\_fraud      & 284,807 & 29 & 492    & 0.17  & 1,024 \\
16\_http       & 567,498 & 3  & 2,211  & 0.39  & 1,024 \\
22\_magic.gamma& 19,020  & 10 & 6,688  & 35.16 & 512 \\
23\_mammography& 11,183  & 6  & 260    & 2.32  & 512 \\
32\_shuttle    & 49,097  & 9  & 3,511  & 7.15  & 1,024 \\
33\_skin       & 245,057 & 3  & 50,859 & 20.75 & 1,024 \\
34\_smtp       & 95,156  & 3  & 30     & 0.03  & 1,024 \\
\bottomrule
\end{tabular}
\end{table}

\begin{table}[ht]
\centering\footnotesize
\caption{High-dimensional datasets.}
\label{tab:chars-high}
\begin{tabular}{l r r r r r}
\toprule
\textbf{Dataset} & \textbf{Samples} & \textbf{Feat.} & \textbf{\#Out.} & \textbf{Out.\,\%} & \textbf{Batch} \\
\midrule
3\_backdoor     & 95,329  & 196   & 2,329  & 2.44  & 1,024 \\
5\_campaign     & 41,188  & 62    & 4,640  & 11.27 & 1,024 \\
9\_census       & 299,285 & 500   & 18,568 & 6.20  & 1,024 \\
17\_InternetAds & 1,966   & 1,555 & 368    & 18.72 & 64 \\
24\_mnist       & 7,603   & 100   & 700    & 9.21  & 64 \\
25\_musk        & 3,062   & 166   & 97     & 3.17  & 64 \\
26\_optdigits   & 5,216   & 64    & 150    & 2.88  & 64 \\
35\_SpamBase    & 4,207   & 57    & 1,679  & 39.91 & 64 \\
36\_speech      & 3,686   & 400   & 61     & 1.65  & 64 \\
48\_arrhythmia  & 452     & 274   & 66     & 14.60 & 64 \\
\bottomrule
\end{tabular}
\end{table}

\subsection{Protocol, baselines, and metrics}
\label{sec:setup}
\textbf{Protocol.} \method{} is trained \emph{semi-supervised} on inliers only
(label $0$), with a $40\%$ test split and $10$ random seeds, under a single
$[-1,1]$-normalized configuration shared across all datasets
(lr $5\!\times\!10^{-5}$, $K{=}200$, $\tau{=}0.1$, $\rho{=}0.999$, $1000$ epochs,
Adam); no hyper-parameter is tuned per dataset.
\textbf{Baselines.} We compare against $22$ baselines from the ADBench/TCCM
codebase: classical detectors (ABOD, LOF, PCA, KPCA, OCSVM, ECOD, CBLOF, KDE,
GMM, MCD, HBOS, IForest, LMDD, QMCD, Sampling) and deep/inductive ones
(AutoEncoder, VAE, DIF, LUNAR, SO/MO\_GAAL, and the flow-matching model TCCM);
together with \method{} this gives the $23$ methods ranked in the tables. All
methods share the same $[-1,1]$ normalization.
\textbf{Metrics.} We report ROC--AUC (area under the ROC curve) and AUPR
(average precision); both are threshold-free and higher-is-better.
\textbf{Missing runs.} A few baselines did not complete on the largest datasets
within budget; each such run is floored to the worst rank (the number of
methods) on that dataset, the standard conservative convention used in the main
paper.
\textbf{How to read the result tables.} In every per-dataset table
(Tables~\ref{tab:auc-small-1}--\ref{tab:aupr-high-2}) each cell is
mean$_{\pm\text{std}}$\,(rank) over the $10$ seeds, and the parenthesised rank
is coloured per column: \green{(1)}~$=$~best, \blue{(2)}~$=$~second,
\lightred{(3)}~$=$~third.

\section{Robustness ablations}
\label{sec:supp-ablate}
We expand the main-paper ablations to per-dataset numbers on seven small
datasets ($n\!\ge\!100$; $K{=}200$, 3 seeds; ROC--AUC), probing two design
choices; Hepatitis ($n{=}80$) is excluded as too small for a stable inlier-only
test split.
\textbf{Normalization} (Table~\ref{tab:ab-norm}): mapping into a bounded
symmetric domain is what matters. Dropping it (``none'') is the only
choice that collapses on several datasets, while among bounded/standardized
choices $[-1,1]$ is the most accurate, within $1.5$ points of $[0,1]$ and
StandardScaler; we adopt it because it makes the theory transfer (main text,
Sec.~\ref{sec:norm}) and it is also the top performer. \textbf{Training contamination} (Table~\ref{tab:ab-contam}): injecting
anomalies into the ``inlier'' training set degrades accuracy \emph{gracefully}
rather than catastrophically, since the EMA support averages over batches so a
minority of contaminating points cannot dominate the learned intervals.

\begin{table}[t]
\centering\footnotesize
\caption{Input normalization. Mapping into a bounded symmetric domain is what
matters; among bounded/standardized choices $[-1,1]$ is the most accurate, within
$1.5$ points of the others.}
\label{tab:ab-norm}
\begin{tabular}{l c c c c}
\toprule
Dataset & none & $[0,1]$ & standard & $[-1,1]$ \\
\midrule
45\_wine        & \green{0.990} & 0.932 & 0.883 & 0.950 \\
21\_Lymphography& 0.990 & \green{0.998} & 0.996 & 0.994 \\
46\_WPBC        & 0.510 & 0.543 & 0.541 & \green{0.584} \\
14\_glass       & \green{0.896} & 0.886 & 0.877 & 0.876 \\
42\_WBC         & 0.990 & 0.989 & 0.990 & \green{0.998} \\
39\_vertebral   & 0.335 & 0.498 & \green{0.511} & 0.477 \\
37\_Stamps      & 0.924 & 0.921 & 0.929 & \green{0.953} \\
\midrule
\textbf{Mean}   & 0.805 & 0.824 & 0.818 & \green{0.833} \\
\bottomrule
\end{tabular}
\end{table}

\begin{table}[t]
\centering\footnotesize
\caption{Training contamination: ROC--AUC vs.\ fraction of anomalies injected
into the ``inlier'' training set. Degradation is graceful.}
\label{tab:ab-contam}
\begin{tabular}{l c c c c}
\toprule
Dataset & 0\% & 5\% & 10\% & 20\% \\
\midrule
45\_wine        & 0.950 & 0.659 & 0.645 & 0.624 \\
21\_Lymphography& 0.994 & 0.978 & 0.970 & 0.973 \\
46\_WPBC        & 0.584 & 0.565 & 0.570 & 0.520 \\
14\_glass       & 0.876 & 0.755 & 0.728 & 0.726 \\
42\_WBC         & 0.998 & 0.979 & 0.982 & 0.976 \\
39\_vertebral   & 0.477 & 0.458 & 0.432 & 0.456 \\
37\_Stamps      & 0.953 & 0.865 & 0.816 & 0.766 \\
\midrule
\textbf{Mean}   & 0.833 & 0.751 & 0.735 & 0.720 \\
\bottomrule
\end{tabular}
\end{table}

\section{Robustness studies}
\label{sec:robust}
We report four robustness checks beyond the per-dataset ablations: the
ranking aggregation under different missing-run handlings, sensitivity of
detection to the sigmoid temperature $\tau$, sensitivity of the certified
clipping check to the active-coordinate threshold, and the per-layer
spectral norms of the Certified-Lipschitz decoder.

\paragraph{Ranking sensitivity to missing-run handling
(Table~\ref{tab:rank-sens}).} The protocol floors a method's
missing runs to the worst rank, which is the standard convention but
favours methods that ran everywhere. We recompute the overall mean rank
under two alternatives: (B) drop the $9$ datasets where at least one method
failed (kept $39$/$48$), and (C) impute missing ranks by the per-dataset
median of the methods that ran. \method{} holds the lowest mean rank on
both metrics under all three schemes. The smallest gap is AUPR scheme B
($0.05$ above LUNAR); the largest is AUC scheme A ($1.90$). The top-cluster
conclusion of the main text is therefore not an artefact of how missing
runs are aggregated.

\begin{table}[t]
\centering\footnotesize
\setlength{\tabcolsep}{4pt}
\caption{Overall mean rank under three missing-run handlings (lower is
better). A: worst-rank floor (paper); B: exclude datasets where any
baseline failed (39/48); C: per-dataset median imputation.}
\label{tab:rank-sens}
\begin{tabular}{l ccc | ccc}
\toprule
& \multicolumn{3}{c}{Rank (ROC--AUC)} & \multicolumn{3}{c}{Rank (AUPR)}\\
\cmidrule(lr){2-4}\cmidrule(lr){5-7}
Method & A & B & C & A & B & C\\
\midrule
\textbf{\method{}}   & \textbf{4.10} & \textbf{4.38} & \textbf{4.10}
                     & \textbf{4.16} & \textbf{4.45} & \textbf{4.16}\\
LUNAR                & 6.00 & 4.77 & 5.07 & 4.91 & 4.50 & 4.54\\
TCCM                 & 8.22 & 8.55 & 8.22 & 7.36 & 7.88 & 7.36\\
AutoEncoder          & 7.42 & 7.46 & 7.05 & 8.05 & 8.04 & 7.69\\
GMM                  & 8.56 & 8.73 & 8.20 & 8.54 & 8.65 & 8.18\\
IForest              & 8.60 & 8.46 & 8.60 & 9.89 & 9.42 & 9.89\\
VAE                  & 8.90 & 9.32 & 8.90 & 9.64 & 10.05 & 9.64\\
CBLOF                & 9.20 & 8.76 & 8.83 & 8.90 & 8.54 & 8.53\\
\bottomrule
\end{tabular}
\end{table}

\paragraph{Temperature $\tau$ sensitivity (Table~\ref{tab:tau-sweep}).}
We sweep $\tau\in\{0.05,0.10,0.20,0.50\}$ at $K{=}200$, $200$ epochs on
three small datasets. The paper's default $\tau{=}0.1$ is within
$0.06$ AUC of the best $\tau$ on every dataset. The only setting that hurts
substantially is $\tau{=}0.05$ on \textsc{Glass}, where the sigmoid is too
sharp for a small dataset's interval shapes to settle. Robustness across
$\tau\in[0.1,0.5]$ argues for a single shared $\tau$ rather than a
per-feature $\tau_j$.

\begin{table}[t]
\centering\footnotesize
\setlength{\tabcolsep}{6pt}
\caption{ROC--AUC vs.\ temperature $\tau$ (DiffInt-PA, $K{=}200$, $200$
epochs, seed $0$). The paper default $\tau{=}0.1$ is competitive on every
dataset; performance degrades sharply only at $\tau{=}0.05$ on \textsc{Glass}.}
\label{tab:tau-sweep}
\begin{tabular}{l cccc}
\toprule
Dataset & $\tau{=}0.05$ & $\tau{=}0.10$ & $\tau{=}0.20$ & $\tau{=}0.50$\\
\midrule
\textsc{Glass}       & 0.619 & 0.810 & \textbf{0.857} & 0.810\\
\textsc{Ionosphere}  & 0.950 & 0.963 & \textbf{0.967} & 0.966\\
\textsc{Cardio}      & 0.931 & \textbf{0.969} & 0.961 & 0.960\\
\bottomrule
\end{tabular}
\end{table}

\paragraph{Active-coordinate threshold sensitivity
(Table~\ref{tab:thresh-sens}).} The certified clipping inequality of
Theorem~1 holds for points that violate \emph{every} active coordinate of
at least one unit. ``Active'' is operationalised as mean inlier membership
$\bar I_{kj}<\theta$ for some $\theta$ (the paper uses $\theta{=}0.9$).
We vary $\theta\in\{0.70,0.80,0.90,0.95\}$ and report the fraction of test
anomalies that meet the strict full-violation hypothesis. A lower
$\theta$ makes the active set sparser (fewer coordinates per unit), which
makes the strict hypothesis easier to satisfy but covers fewer of the
unit's coordinates; a higher $\theta$ does the reverse. The fraction is
therefore dataset-dependent: high when units keep few active coordinates
(\textsc{Glass}, $2$--$3$ active/unit at $\theta{=}0.9$, $100\%$) and low
when they keep many (\textsc{Cardio}, ${\sim}18$ active/unit, $0\%$ at
$\theta{=}0.9$ and $16\%$ at $\theta{=}0.7$). Under the remaining anomalies
the graded suppression of Corollary~1 applies -- it down-weights the matched
units but gives no reconstruction-error bound. We do not tune $\theta$ as a
hyperparameter; we report the strict regime as a verifiable sanity check
rather than a guarantee that holds pointwise, and the MAE inequality of
Theorem~1(ii) is in any case satisfied for \emph{every} out-of-support point
under the Certified-Lipschitz decoder (Table~\ref{tab:certify}).

\begin{table}[t]
\centering\footnotesize
\setlength{\tabcolsep}{4pt}
\caption{Fraction of test anomalies that satisfy the strict full-violation
hypothesis of Theorem~1 as the active-coordinate threshold $\theta$ varies
(Certified-Lipschitz DiffInt, $K{=}200$, $200$ epochs, seed $0$). Mean
active coordinates per unit is also reported.}
\label{tab:thresh-sens}
\begin{tabular}{l cccc cccc}
\toprule
& \multicolumn{4}{c}{out-of-support fraction} & \multicolumn{4}{c}{mean active/unit}\\
\cmidrule(lr){2-5}\cmidrule(lr){6-9}
Dataset & $\theta{=}0.70$ & $0.80$ & $0.90$ & $0.95$ & $0.70$ & $0.80$ & $0.90$ & $0.95$\\
\midrule
\textsc{Glass}       & 1.00 & 1.00 & 1.00 & 0.00 & 0.2 & 0.9 & 2.5 & 4.0\\
\textsc{Ionosphere}  & 1.00 & 0.14 & 0.10 & 0.09 & 1.0 & 11.6 & 17.1 & 29.3\\
\textsc{Cardio}      & 0.16 & 0.00 & 0.00 & 0.00 & 12.2 & 14.8 & 17.7 & 20.0\\
\bottomrule
\end{tabular}
\end{table}

\paragraph{Per-layer spectral norms of the certified decoder
(Table~\ref{tab:spec-norms}).} When the decoder is spectrally normalized
(the Certified-Lipschitz variant), each linear layer carries a spectral
norm slightly above $1$ ($1.00$--$1.08$ on the datasets tested) and the
ReLU between them is exactly $1$-Lipschitz. The product $L=\prod_l\|W_l\|_2$
stays close to $1$ ($1.03$--$1.08$), so the clipping-margin bound's slack
$\tfrac{L}{\sqrt d}\kappa(\beta)$ is small.

\begin{table}[t]
\centering\footnotesize
\setlength{\tabcolsep}{5pt}
\caption{Per-layer spectral norm $\|W_l\|_2$ of the Certified-Lipschitz
decoder (two-layer ReLU MLP, hidden $h{=}128$) on three datasets. ReLU
contributes $1$ to the product. $L=\prod_l\|W_l\|_2$ is an explicit upper
bound on the decoder's Lipschitz constant.}
\label{tab:spec-norms}
\begin{tabular}{l ccc c}
\toprule
Dataset & $\|W_1\|_2$ & ReLU & $\|W_2\|_2$ & $L=\prod\|W_l\|_2$\\
\midrule
\textsc{Glass}       & 1.022 & 1.000 & 1.006 & 1.028\\
\textsc{Ionosphere}  & 1.071 & 1.000 & 1.007 & 1.078\\
\textsc{Cardio}      & 1.010 & 1.000 & 1.000 & 1.010\\
\bottomrule
\end{tabular}
\end{table}

\paragraph{Inference clamping isolated from normalization
(Table~\ref{tab:clamp}).} At inference the test features are MinMax-scaled
by the \emph{train} scaler and clipped to $[-1,1]$. To isolate the effect
of the clip from the effect of the scaling itself, we score every test
point twice: once after clipping, once with the raw scaled values left
unchanged (so any feature above $1$ or below $-1$ stays where it is). The
overall AUC differences are within $0.012$ on the three datasets tested,
and remain small ($\le 0.009$) even on the ``extreme'' subset of test
points where at least one feature lies outside $[-1,1]$ after the train
scaler is applied. The mean absolute change in the anomaly score on
extreme points is $0.04$--$0.12$, i.e.\ clipping nudges individual scores
but rarely re-orders them. Consistent with the structural argument
(clamping can only \emph{lower} memberships), we observe no case where
clipping masks an anomaly.

\begin{table}[t]
\centering\footnotesize
\setlength{\tabcolsep}{4pt}
\caption{Effect of inference clamping in isolation. ``clip'' is the paper
default (clip scaled test features to $[-1,1]$); ``no-clip'' passes the
scaled features through unchanged. $n_{\text{ext}}$ is the number of test
points with at least one feature outside $[-1,1]$ after train-scaling.
AUC differences are at most $0.012$; clipping never re-orders the leading
score on the full test set.}
\label{tab:clamp}
\begin{tabular}{l rr cc c}
\toprule
& & & \multicolumn{2}{c}{AUC on full test} &
\multicolumn{1}{c}{$|\Delta\text{score}|$}\\
\cmidrule(lr){4-5}
Dataset & $n_{\text{test}}$ & $n_{\text{ext}}$ & clip & no-clip & (extreme pts)\\
\midrule
\textsc{Glass}       &  86 &  5 & 0.810 & 0.798 & 0.10\\
\textsc{Ionosphere}  & 141 & 70 & 0.963 & 0.969 & 0.04\\
\textsc{Cardio}      & 733 & 62 & 0.969 & 0.971 & 0.12\\
\bottomrule
\end{tabular}
\end{table}

\paragraph{LFI stability across seeds (Table~\ref{tab:seed-stab}).}
The third LFI factor, $\mathrm{stab}_{kj}$, encodes inverse cross-seed
spread. To check that LFI itself ranks features stably across re-runs we
train $5$ seeds per dataset (varying the train/test split and the model
initialization), aggregate LFI per feature as $\max_k \mathrm{LFI}_{kj}$,
and compute the pairwise Spearman rank correlation across the ten
seed-pairs. The mean correlation is $0.74$ on \textsc{Glass}
(small data with only $7$ features), $0.88$ on \textsc{Ionosphere}, and
$0.92$ on \textsc{Cardio}, with minimum pair $0.57$/$0.84$/$0.83$
respectively. LFI surfaces the same per-feature constraints in $\ge\!75\%$
of pairs even on the smallest dataset, and the larger feature spaces are
very stable.

\begin{table}[t]
\centering\footnotesize
\setlength{\tabcolsep}{6pt}
\caption{LFI stability across $5$ training seeds (different splits and
initializations). Per-feature LFI is the max over units; pairwise Spearman
is computed across all $10$ seed pairs. Higher is more stable.}
\label{tab:seed-stab}
\begin{tabular}{l c ccc}
\toprule
Dataset & $d$ & mean $\rho$ & min $\rho$ & std $\rho$\\
\midrule
\textsc{Glass}       &  7 & 0.74 & 0.57 & 0.11\\
\textsc{Ionosphere}  & 32 & 0.88 & 0.84 & 0.03\\
\textsc{Cardio}      & 21 & 0.92 & 0.83 & 0.04\\
\bottomrule
\end{tabular}
\end{table}

\section{Interpretability as evidence}
\label{sec:faith}
We support the two interpretability claims of the main text with quantitative
measurements ($K{=}200$, 3 seeds; Table~\ref{tab:interp}).

\textbf{(1) Label-free ranking (LFI).} We rank (unit, feature) constraints by
the label-free importance
$\mathrm{LFI}_{kj}=s_{kj}\cdot(1-\Delta_{kj}/\mathrm{range}_j)\cdot
\mathrm{stab}_{kj}$ (support $\times$ tightness $\times$ stability), using only
quantities the trained model already maintains. This is the only ranking we
use in the main text (Sec.~\ref{sec:lfi}); it requires no anomaly label and is computable
in closed form from the model parameters and the EMA support buffer.

\textbf{(2) Faithfulness perturbation.} For each test inlier we perturb its
top-LFI features to \emph{cross} their interval boundaries and measure the
anomaly-score rise $\Delta_{\text{top}}$, against perturbing the same number of
random features $\Delta_{\text{rand}}$. The top-ranked constraints raise the
score $1.5{\times}$ more on average (on $69\%$ of inliers; Table~\ref{tab:interp}),
so the LFI-highlighted constraints \emph{causally} drive the score rather than
merely correlating with it.

\begin{table}[t]
\centering\footnotesize
\setlength{\tabcolsep}{4pt}
\caption{Faithfulness perturbation test: score rise from crossing the top-LFI
feature boundaries ($\Delta_{\text{top}}$) vs.\ random features
($\Delta_{\text{rand}}$), their ratio, and the fraction of inliers with
$\Delta_{\text{top}}\!>\!\Delta_{\text{rand}}$.}
\label{tab:interp}
\begin{tabular}{l c c c c}
\toprule
Dataset & $\Delta_{\text{top}}$ & $\Delta_{\text{rand}}$ & ratio & \% top$>$rand \\
\midrule
14\_glass       & 0.200 & 0.134 & 1.48 & 71\% \\
42\_WBC         & 0.094 & 0.074 & 1.26 & 62\% \\
18\_Ionosphere  & 0.105 & 0.078 & 1.34 & 64\% \\
29\_Pima        & 0.171 & 0.100 & 1.70 & 71\% \\
6\_cardio       & 0.119 & 0.066 & 1.80 & 79\% \\
\midrule
\textbf{Mean}   & 0.138 & 0.091 & 1.52 & 69\% \\
\bottomrule
\end{tabular}
\end{table}

\textbf{(3) Conciseness and readability.} For a flagged point, the per-feature
reconstruction error concentrates: on average $\approx$half the features
($10.8$ of $20$; Table~\ref{tab:concise}) account for $80\%$ of the anomaly
score, so a short shortlist suffices to explain it. The interval units
themselves are \emph{soft and wide} (mean width $0.70$ of the feature range, and
a unit softly constrains most coordinates by the membership criterion of
Def.~1), which is precisely why we present intervals as \emph{candidate
constraints} rather than crisp rules and surface only the top-$k$ by LFI
($k\!=\!2$ in the main-paper read-out) rather than a whole unit.

\begin{table}[t]
\centering\footnotesize
\setlength{\tabcolsep}{5pt}
\caption{Conciseness and human-readability proxies ($K{=}200$, 3 seeds).
\#feat$_{80\%}$: features (ranked by per-feature error) explaining $80\%$ of a
flagged point's score; active/unit: mean number of active features
($\bar I_{kj}<0.9$) per top interval unit; width: mean active-interval width as a
fraction of the $[-1,1]$ range.}
\label{tab:concise}
\begin{tabular}{l c c c c}
\toprule
Dataset & $d$ & \#feat$_{80\%}$ & active/unit & width \\
\midrule
14\_glass       &  9 &  4.8 &  6.7 & 0.70 \\
42\_WBC         & 30 & 16.0 & 24.8 & 0.69 \\
18\_Ionosphere  & 33 & 17.9 & 29.8 & 0.70 \\
29\_Pima        &  8 &  4.6 &  5.3 & 0.70 \\
6\_cardio       & 21 & 11.0 & 17.2 & 0.70 \\
\midrule
\textbf{Mean}   & 20 & 10.8 & 16.8 & 0.70 \\
\bottomrule
\end{tabular}
\end{table}

\section{Quantitative study of the global LFI ranking}
\label{sec:xai-quant}

Reviewers asked for a systematic quantitative comparison of LFI against
established explanation methods. \textbf{LFI is a global \emph{interpretability}
readout, not a per-instance explainer:} from the trained model alone it ranks
the learned interval constraints -- aggregated over units, the features -- by
support $\times$ tightness $\times$ cross-seed stability (Sec.~\ref{sec:faith}),
returning \emph{one} ranking for the whole model. We evaluate that global
ranking against the same ranking recovered by \emph{aggregating} post-hoc
attributions (mean $|\cdot|$ over a pool of flagged anomalies):
gradient$\times$input, integrated gradients ($24$ steps), KernelSHAP ($120$
background samples), and the detection-native ECOD tail surprise
$-\log 2\min(F_j,1{-}F_j)$. All four attribute the same \method{} score.

\paragraph{Scope.} LFI is built entirely from per-feature interval tightness
and support, so it is informative exactly in \method{}'s own design regime:
\emph{marginal, axis-aligned anomalies whose signal is visible
coordinate-by-coordinate}, violating coordinates the model constrains tightly.
We evaluate strictly inside that regime. Anomalies visible only in feature
\emph{interactions} are the detector's documented blind spot
(Sec.~\ref{sec:limits}) and lie outside the scope of any per-feature
diagnostic.

\paragraph{Protocol.} \method{} is trained inliers-only under the paper
configuration ($K{=}200$, $\tau{=}0.1$). We use three controlled families,
$7$ re-trainings in total: $d\in\{20,28,36\}$, with $6$--$14$ tightly
distributed coordinates among broad ones, and every anomaly a marginal
violation of the tight coordinates -- so the relevant set is fixed and known.
Metrics: \emph{sep-AUC} (does the global ranking put the relevant coordinates
first?), \emph{agreement} with each aggregated post-hoc ranking (Spearman
$\rho$), cross-seed rank \emph{stability} over the re-trainings, a
\emph{model-randomisation sanity check}~\cite{adebayo2018sanity} (rank $\rho$
against the same score on a weight-randomised model whose support is
re-estimated identically; low is good), attribution \emph{concentration}
(Gini), and model \emph{queries}.

\begin{table}[t]\centering\footnotesize
\setlength{\tabcolsep}{3.6pt}
\caption{The zero-query \emph{global} LFI feature ranking vs.\ the same ranking from \emph{aggregated} post-hoc attributions (mean $|\cdot|$ over a pool of flagged anomalies), inside \method{}'s regime: three controlled families ($d\in\{20,28,36\}$, $6$--$14$ tightly-constrained coordinates among broad ones, $7$ re-trainings total) in which every anomaly is a marginal violation of the tight coordinates. \emph{sep-AUC}: does the ranking put the relevant coordinates first? \emph{agree}: mean Spearman $\rho$ with the LFI ranking. \emph{stab.}: mean cross-seed Spearman $\rho$. \emph{sanity}: $\rho$ with the ranking from a weight-randomised model (support re-estimated identically; low is good). \emph{Gini}: concentration of the vector. $q$: model passes per pooled anomaly.}
\label{tab:lfi-global}
\begin{tabular}{lcccccc}
\toprule
Method & sep-AUC & agree & stab. & sanity & Gini & $q$ \\
\midrule
\textbf{LFI (global)} & 1.00 & -- & 0.71 & +0.01 & 0.11 & 0 \\
grad$\times$input (agg.) & 1.00 & +0.71 & 0.69 & -- & 0.33 & 1 \\
IntGrad (agg.) & 1.00 & +0.69 & 0.66 & -- & 0.45 & 24 \\
KernelSHAP (agg.) & 1.00 & +0.72 & 0.67 & -- & 0.58 & 120 \\
ECOD-tail (agg.) & 1.00 & +0.74 & 0.67 & -- & 0.47 & 0 \\
\bottomrule
\end{tabular}
\end{table}

\begin{figure}[t]
\centering
\includegraphics[width=.92\linewidth]{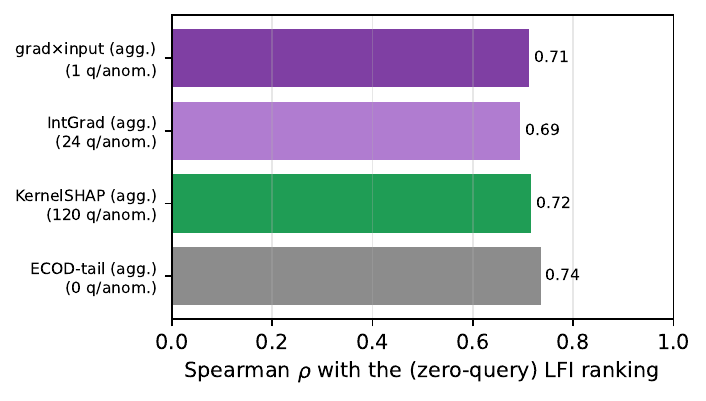}
\caption{Spearman $\rho$ between the zero-query LFI feature ranking and each
\emph{aggregated} post-hoc ranking, with the model-query cost each post-hoc
ranking pays. The LFI ranking reproduces the same global picture
($\rho\approx0.7$) that KernelSHAP buys with $120$ model passes per pooled
anomaly.}
\label{fig:interp-agree}
\end{figure}

\paragraph{Findings.}
\emph{(1) The relevant coordinates are recovered.} In this regime the marginal
signal is clear, and every ranking -- including the zero-query LFI vector --
separates the relevant coordinates from the rest with sep-AUC $1.00$
(Table~\ref{tab:lfi-global}).
\emph{(2) LFI reproduces what the expensive methods report.} The LFI ranking
agrees with the aggregated KernelSHAP / integrated-gradients / gradient
rankings at Spearman $\rho=0.69$--$0.74$ (Fig.~\ref{fig:interp-agree}) -- the
same ordering, not merely the same top set -- read off the model parameters at
\textbf{zero} model queries, against up to $120$ passes \emph{per pooled
anomaly} for the gradient- and sampling-based methods.
\emph{(3) LFI is stable and passes the sanity check.} Cross-seed rank
stability $\rho=0.71$ over these re-trainings, on par with KernelSHAP ($0.67$)
and the gradient methods ($0.66$--$0.69$); it is marginally below the
$\rho\ge0.74$ measured with the same per-feature statistic on real ADBench
data (App.~\ref{sec:robust}, Table~\ref{tab:seed-stab}) because here roughly
half the coordinates are constrained near-equally, so their relative order is
intrinsically noisier. Rank correlation with a weight-randomised model is
$\rho\approx0.0$, so the ranking reflects what was learned rather than the
architecture. The LFI vector is also the least spiky (Gini $0.11$ vs.\
$0.33$--$0.58$): it flags a \emph{set} of comparably diagnostic constraints
rather than a single coordinate.
\emph{Bottom line:} within the axis-aligned regime \method{} is built for, the
global LFI ranking carries the same information as an aggregated KernelSHAP
ranking costing hundreds of model queries, at no cost, and it is stable and
model-faithful. Outside that regime it is not a substitute for a per-instance
attributor, and we do not present it as one.

\section{Complexity and scalability}
\label{sec:complexity}
Practical deployability matters as much as accuracy, so we summarize the resource
profile of \method{}. Let $n$ be the number of training points, $d$ the feature
count, $K$ the number of interval units, $B$ the batch size, $E$ the epochs, and
$H$ the (fixed) decoder hidden width; the decoder is an MLP whose size is
independent of $n$.

\textbf{Parameter count.} The interval bottleneck stores a centre and a softplus
half-width per (unit, feature): $2Kd$ parameters, linear in $K$ and $d$ and
\emph{independent of $n$}. The decoder adds a compact MLP ($O(KH+Hd)$ weights).
The total stays MLP-class ($10^5$--$10^6$ weights even at $d{=}1{,}555$,
InternetAds), far below typical deep tabular detectors.
Table~\ref{tab:complexity} reports the $2Kd$ bottleneck size and the
matched-budget training time across scales.

\textbf{Training complexity.} A minibatch computes $Bd$ sigmoid pairs for the soft
memberships, an $O(BKd)$ aggregation (a per-unit log-sum over features plus an
$O(K)$ softmax), and one decoder forward/backward; an epoch is thus
$O(nKd)$ plus the $n$-independent decoder cost, with \emph{no} pairwise, kernel, or
neighbour computation. Total training is $O(E\,nKd)$.

\textbf{Runtime vs.\ baselines.} \method{} uses one forward pass per sample, unlike
memory-based detectors (e.g.\ LUNAR) whose $k$NN queries grow with $n$.
Empirically (Table~\ref{tab:binaps}, matched $1000$-epoch budget) \method{} trains
one to two orders of magnitude faster than the binary differentiable miner BinaPs,
faster than LUNAR on most datasets, and scales to the largest ADBench sets
($n\!\approx\!6.2\!\times\!10^5$).

\textbf{Inference and memory footprint.} Scoring a point is a single forward pass:
$O(Kd)$ memberships, the aggregation, the decoder, then the MAE. \method{} is
\emph{inductive}: it keeps only the trained parameters, not the training set,
so inference latency and memory are constant in $n$, whereas
transductive/memory-based detectors must retain or query the data. The training
footprint is the parameters ($O(Kd)$), one batch of memberships ($O(BKd)$
activations), and the $K\times d$ EMA support buffer ($O(Kd)$); none depends on
$n$. At inference only the $O(Kd)$ parameters and a single $K\times d$ membership
tensor are held.

\textbf{The $K$ knob.} $K$ trades capacity for cost linearly. ROC--AUC saturates
after a few units (Fig.~\ref{fig:ablation}), so $K$ can be reduced well below the
default under tight budgets with little accuracy loss; we keep $K{=}200$ as one
untuned default across all $48$ datasets.

\begin{table}[t]
\centering\footnotesize
\setlength{\tabcolsep}{5pt}
\caption{Resource profile across scales ($K{=}200$). Bottleneck parameters
($2Kd$) are linear in $d$ and independent of $n$; training time (\method{},
matched $1000$-epoch budget) is reproduced from Table~\ref{tab:binaps} and
grows with $n$ and $d$.}
\label{tab:complexity}
\begin{tabular}{l r r r r}
\toprule
Dataset & $n$ & $d$ & bottleneck $2Kd$ & train (s) \\
\midrule
6\_cardio      & 1,831  & 21  & 8,400   & 3.0   \\
31\_satimage-2 & 5,803  & 36  & 14,400  & 54.6  \\
26\_optdigits  & 5,216  & 64  & 25,600  & 60.1  \\
48\_arrhythmia & 452    & 274 & 109,600 & 30.2  \\
32\_shuttle    & 49,097 & 9   & 3,600   & 193.5 \\
\bottomrule
\end{tabular}
\end{table}

\section{Stratified and complete-case comparison}
\label{sup:strat}

\method{} and three baselines (TCCM, AutoEncoder, VAE) train on inliers only;
the other $19$ baselines run unsupervised on the contaminated data. Pooling all
$23$ into one mean-rank / Nemenyi analysis (Fig.~\ref{fig:cd}) therefore
compares across training regimes, which reviewers rightly flagged. Using the
\emph{same} per-dataset scores as everywhere else (mean over $10$ seeds;
Tables~\ref{tab:auc-small-1}--\ref{tab:aupr-high-2}), we recompute mean ranks
\emph{within} strata (Table~\ref{tab:strat}, placed with the main results;
per dataset scale, Table~\ref{tab:strat-scale}). No model is re-run and no
setting is changed.

\begin{table}[t]
\centering\footnotesize
\setlength{\tabcolsep}{5pt}
\caption{\method{} mean rank \emph{within the inlier-only pool of $4$}, per
dataset scale. Best method of the four in each cell shown in brackets.}
\label{tab:strat-scale}
\begin{tabular}{l cc}
\toprule
Scale ($N$) & \shortstack{Rank\\(ROC--AUC)} & \shortstack{Rank\\(AUPR)}\\
\midrule
Small (12)  & \textbf{1.83} & \textbf{1.75}\\
Medium (15) & \textbf{1.53} & \textbf{1.47}\\
Large (11)  & \textbf{1.36} & \textbf{1.45}\\
High-dim (10) & \textbf{2.10} & $2.30$ (AE $2.20$)\\
\bottomrule
\end{tabular}
\end{table}

Table~\ref{tab:strat}: (i) among inlier-only detectors \method{} is first by a
clear margin on both metrics; per scale (Table~\ref{tab:strat-scale}) its
inlier-only rank is $1.4$--$2.1$ of $4$, best in every scale$\times$metric cell
except high-dimensional AUPR, where AutoEncoder leads $2.20$ vs.\ $2.30$.
(ii) Against the contaminated-data pool alone, \method{} is still first on all
$48$ datasets; on the complete-case subset -- which removes the worst-rank
imputation the main protocol applies to unfinished runs -- \method{} and LUNAR
are effectively tied ($3.59$ vs.\ $4.03$ on AUC, $3.74$ vs.\ $3.85$ on AUPR).
Adding \method{} to the full $23$-method complete-case pool gives mean rank
$4.31$/$4.46$ vs.\ LUNAR $4.82$/$4.54$, consistent with the missing-run
sensitivity table (Table~\ref{tab:rank-sens}, scheme B). The ``competitive,
not beaten'' phrasing the main text uses for contaminated-data detectors is
thus the accurate reading, and the top-cluster conclusion survives both the
stratification and the removal of imputation.

\section{Points deferred to a future version}
\label{sup:rest}

\paragraph{Newer detector families.} The $22$ baselines are the ADBench/TCCM
tabular suite. Families we do not benchmark include diffusion-based
reconstruction detectors, retrieval / memory-based detectors, and tabular
foundation-model (PFN-style) scorers. They target the same task from different
priors; a controlled comparison, ideally under the same $[-1,1]$ protocol and
inlier-only split, is planned for an extended version. \method{}'s contribution
is orthogonal -- an interpretable bottleneck, not a new density estimator.

\paragraph{Normalization coverage.} The main-paper ablation compares no
scaling, $[0,1]$, standardization, and $[-1,1]$; it does not include a robust /
quantile mapping into $[-1,1]$. Such a mapping would keep the bounded symmetric
$[-1,1]$ geometry the temperature and clipping margin rely on while tolerating
heavy tails (main text, Limitations), and is the natural next ablation.

\paragraph{Quantitative interpretability.} \emph{Addressed}: a systematic
comparison of the global LFI feature ranking against aggregated KernelSHAP /
integrated-gradients / gradient / ECOD rankings, on three controlled
ground-truth families inside \method{}'s regime ($7$ re-trainings), is in
App.~\ref{sec:xai-quant}
(Table~\ref{tab:lfi-global}, Fig.~\ref{fig:interp-agree}).

\section{Complete detection accuracy}
\label{sec:accuracy}
This section reports the full per-dataset detection accuracy underlying the
critical-difference analysis of the main text: ROC--AUC (Sec.~\ref{sec:auc}) and
average precision (Sec.~\ref{sec:aupr}). Across both metrics \method{} sits in
the leading group on every scale and is most often first on small/medium/large
data; in high dimensions it remains competitive with the strongest baselines.

\subsection{ROC--AUC}
\label{sec:auc}
Tables~\ref{tab:auc-small-1}--\ref{tab:auc-high-2} give the per-dataset ROC--AUC.
Each cell is the mean score $\pm$ standard deviation over the $10$ seeds, with
the method's rank on that dataset in parentheses, coloured \green{(1)}~best,
\blue{(2)}~second, \lightred{(3)}~third; the same format is used for the AUPR
tables (Sec.~\ref{sec:aupr}).
Figs.~\ref{fig:cdscale-auc}--\ref{fig:cdscale-aupr} show the per-scale
critical-difference diagrams for ROC--AUC and AUPR (the overall diagrams are
in the main text). They support the regime-level claims of
Table~\ref{tab:ranks}: \method{} attains the lowest mean rank in every
(scale, metric) cell
except high-dimensional AUPR, where LUNAR is marginally ahead but the two are
statistically indistinguishable within the per-scale CD.

\begin{figure*}[t]
\centering
\begin{subfigure}{0.49\linewidth}
\includegraphics[width=\linewidth]{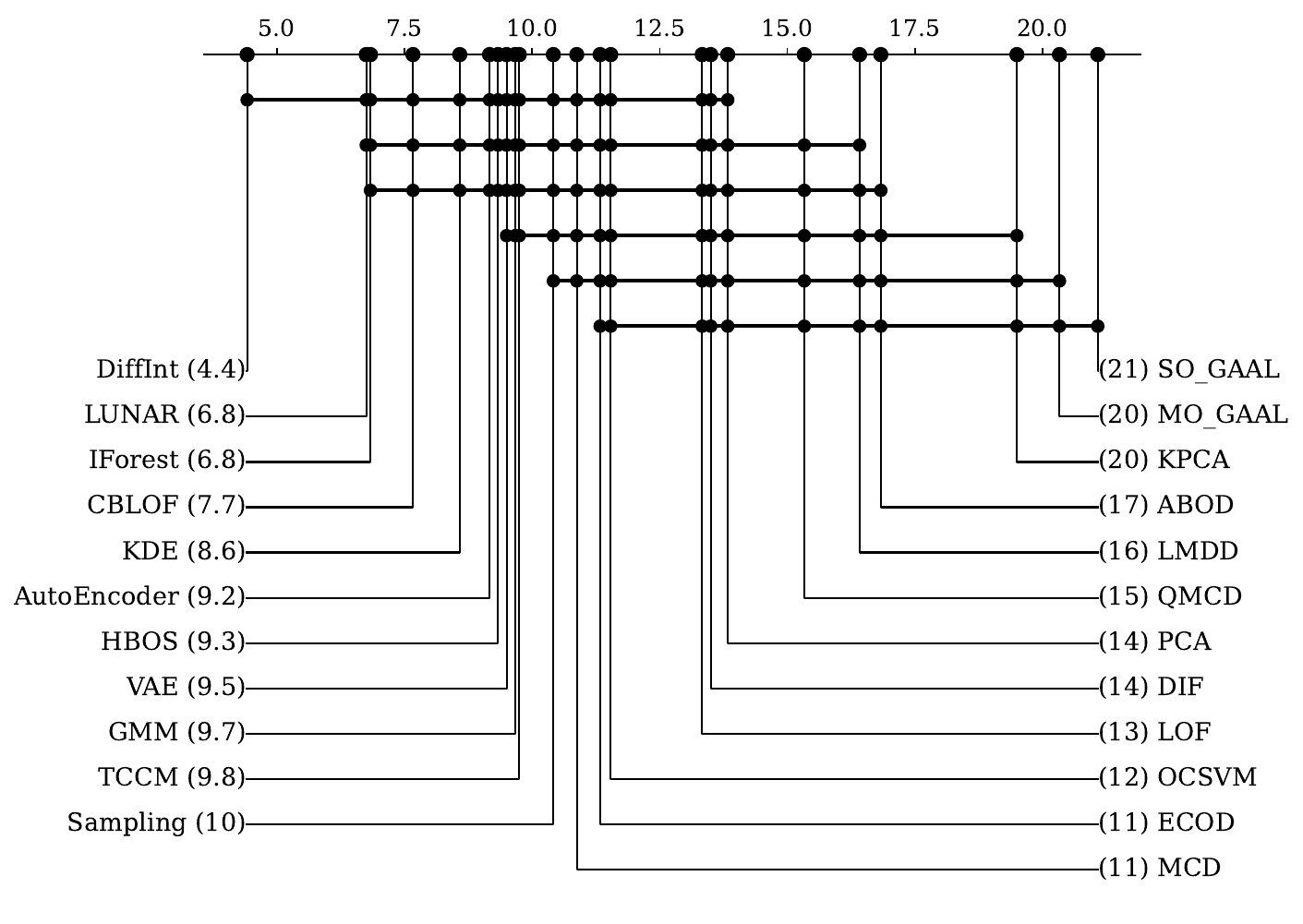}\caption{Small ($n{=}12$).}
\end{subfigure}\hfill
\begin{subfigure}{0.49\linewidth}
\includegraphics[width=\linewidth]{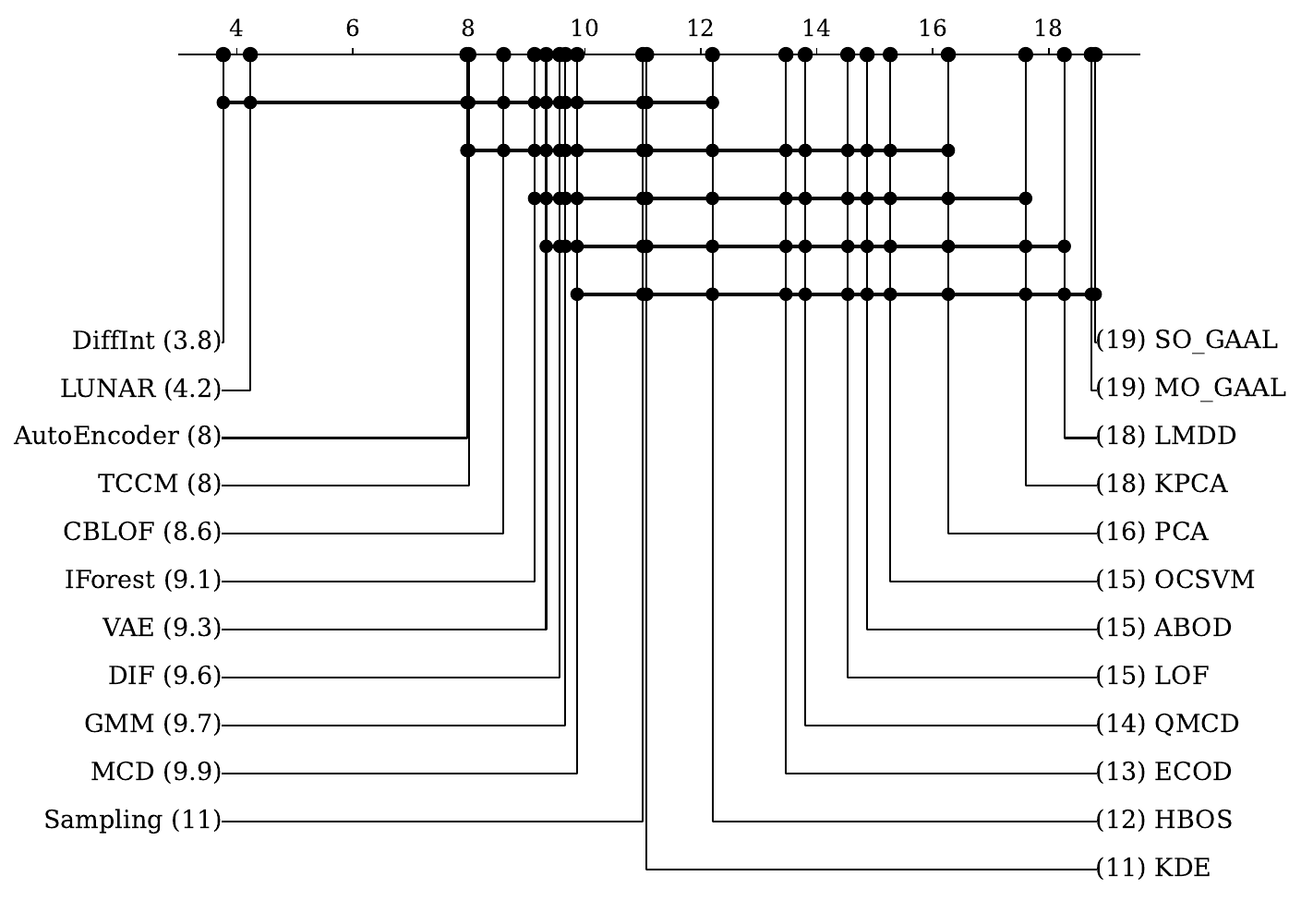}\caption{Medium ($n{=}15$).}
\end{subfigure}\\[3pt]
\begin{subfigure}{0.49\linewidth}
\includegraphics[width=\linewidth]{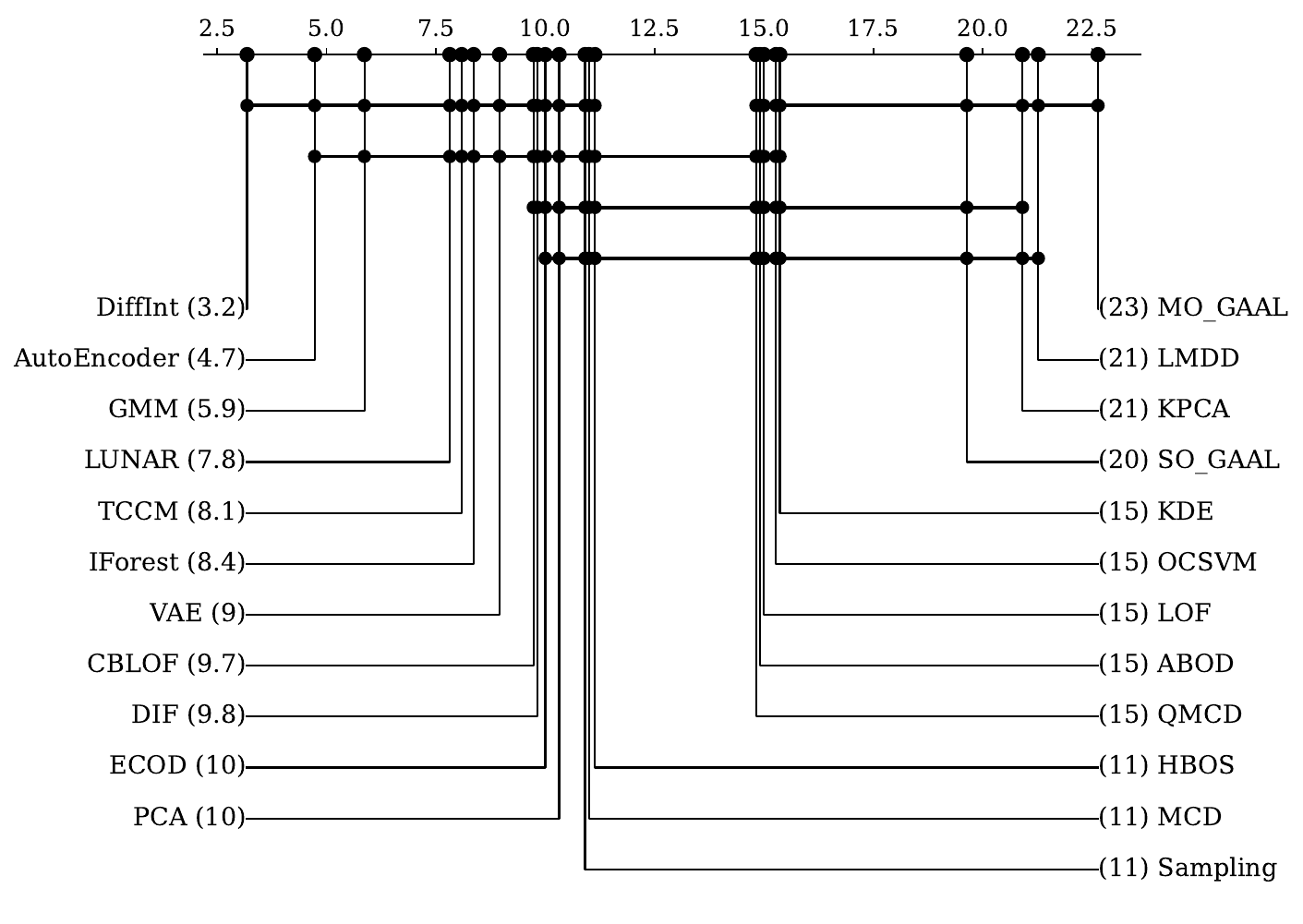}\caption{Large ($n{=}11$).}
\end{subfigure}\hfill
\begin{subfigure}{0.49\linewidth}
\includegraphics[width=\linewidth]{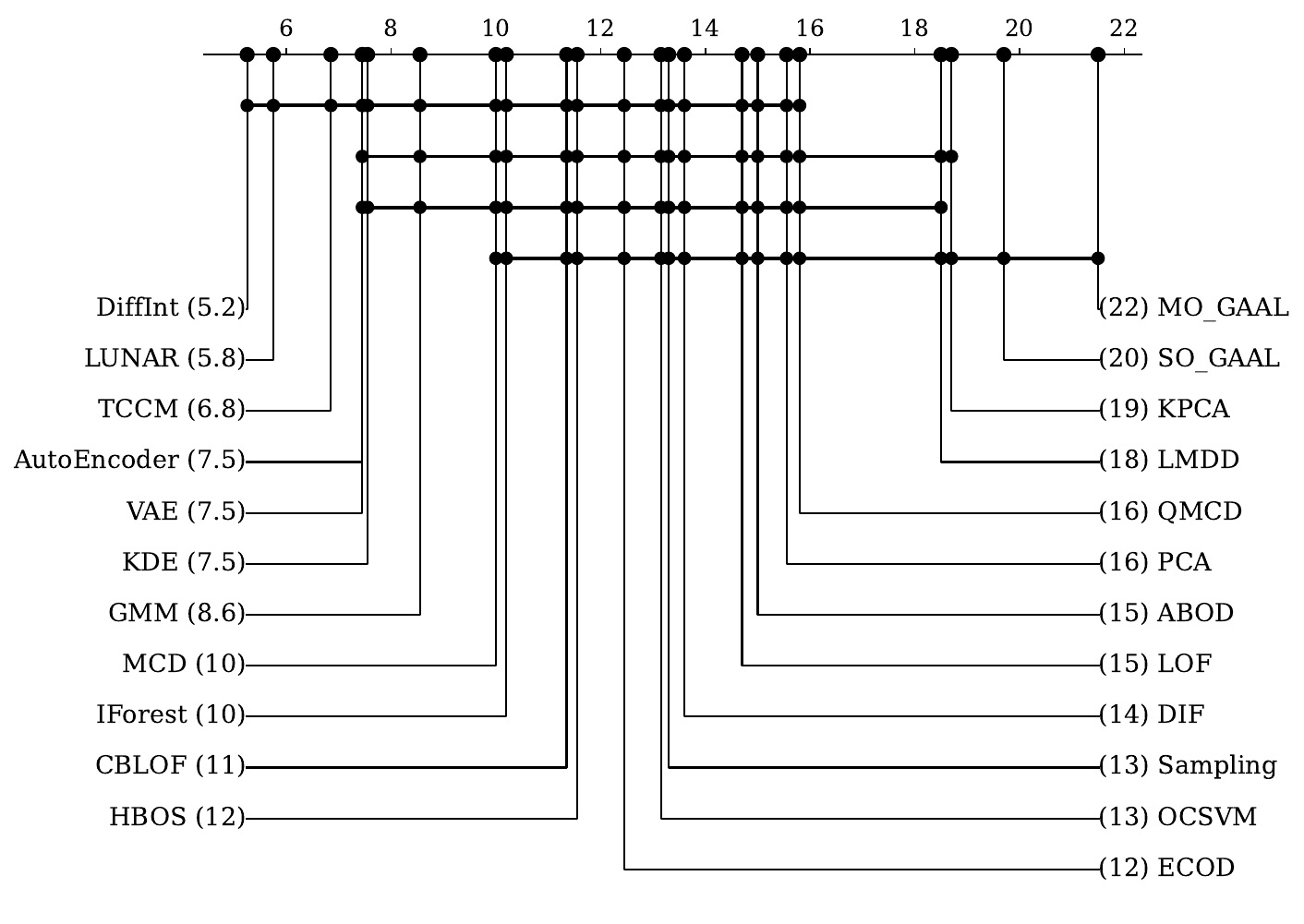}\caption{High-dimensional ($n{=}10$).}
\end{subfigure}
\caption{Per-scale \textbf{ROC--AUC} critical-difference diagrams (missing runs
floored to worst rank, bars indicate non-significant differences at $\alpha{=}0.05$).
\method{} leads strictly on every scale; matches the ``\method{} AUC''
column of Table~\ref{tab:ranks}.}
\label{fig:cdscale-auc}
\end{figure*}

\begin{figure*}[t]
\centering
\begin{subfigure}{0.49\linewidth}
\includegraphics[width=\linewidth]{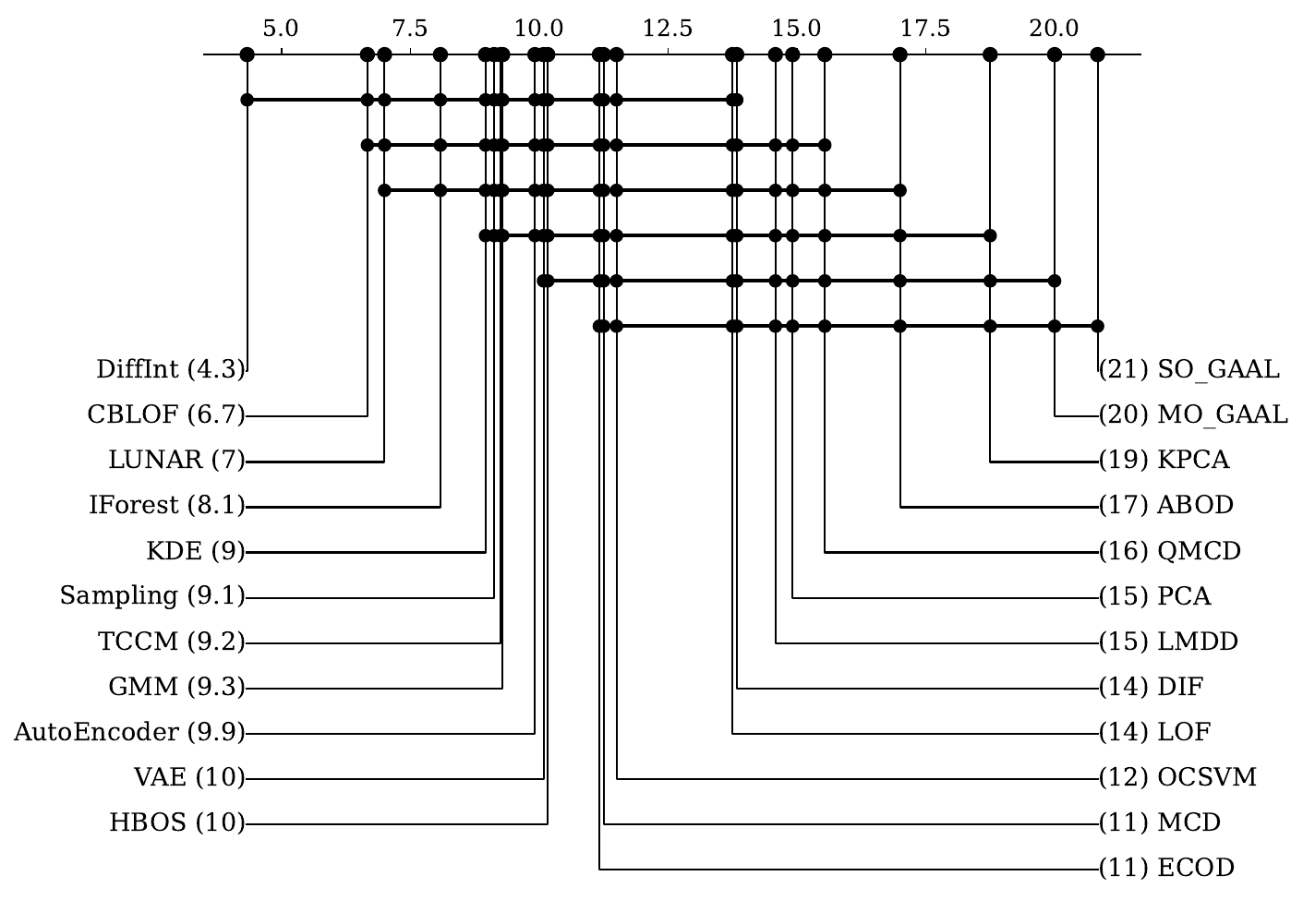}\caption{Small ($n{=}12$).}
\end{subfigure}\hfill
\begin{subfigure}{0.49\linewidth}
\includegraphics[width=\linewidth]{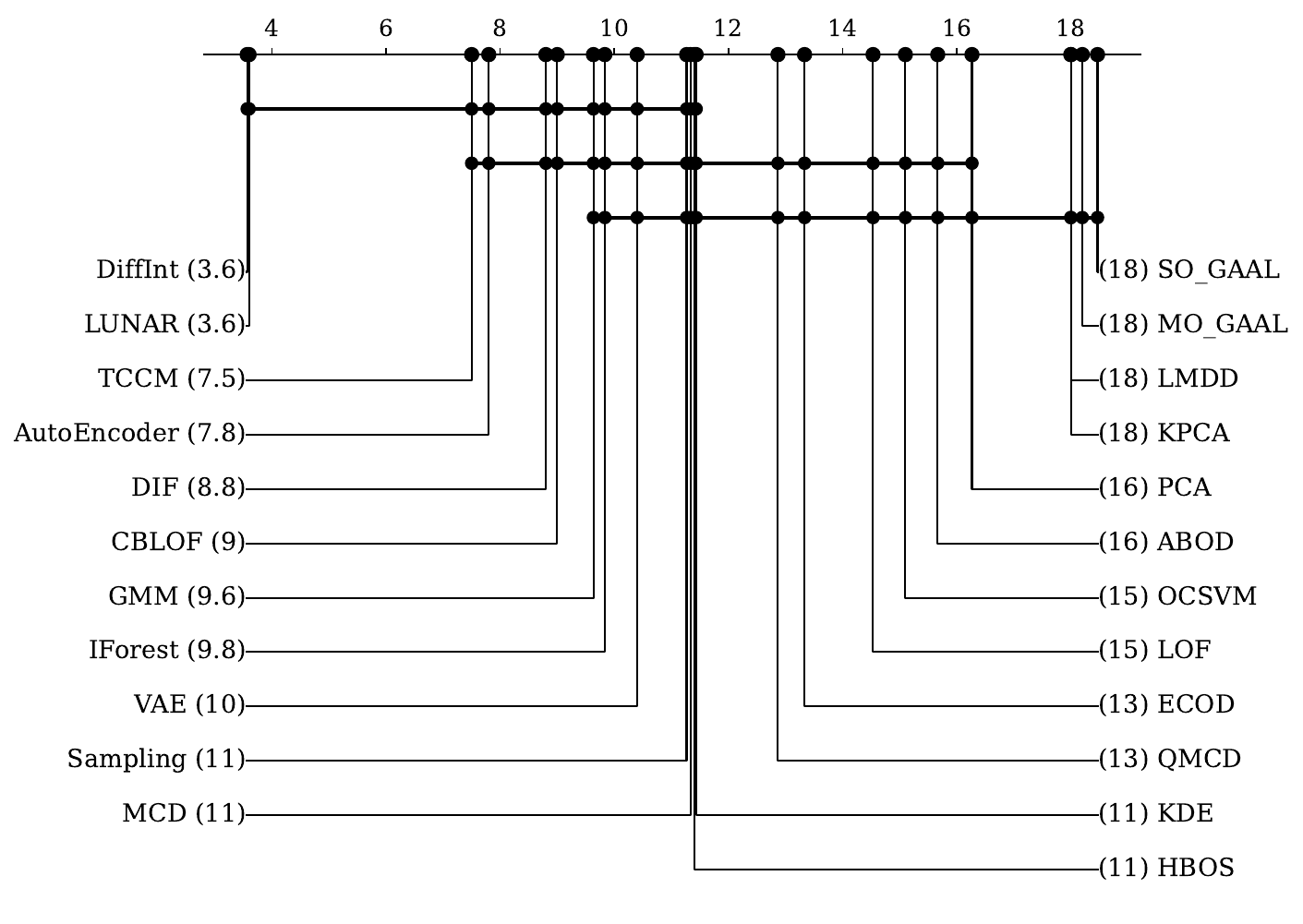}\caption{Medium ($n{=}15$).}
\end{subfigure}\\[3pt]
\begin{subfigure}{0.49\linewidth}
\includegraphics[width=\linewidth]{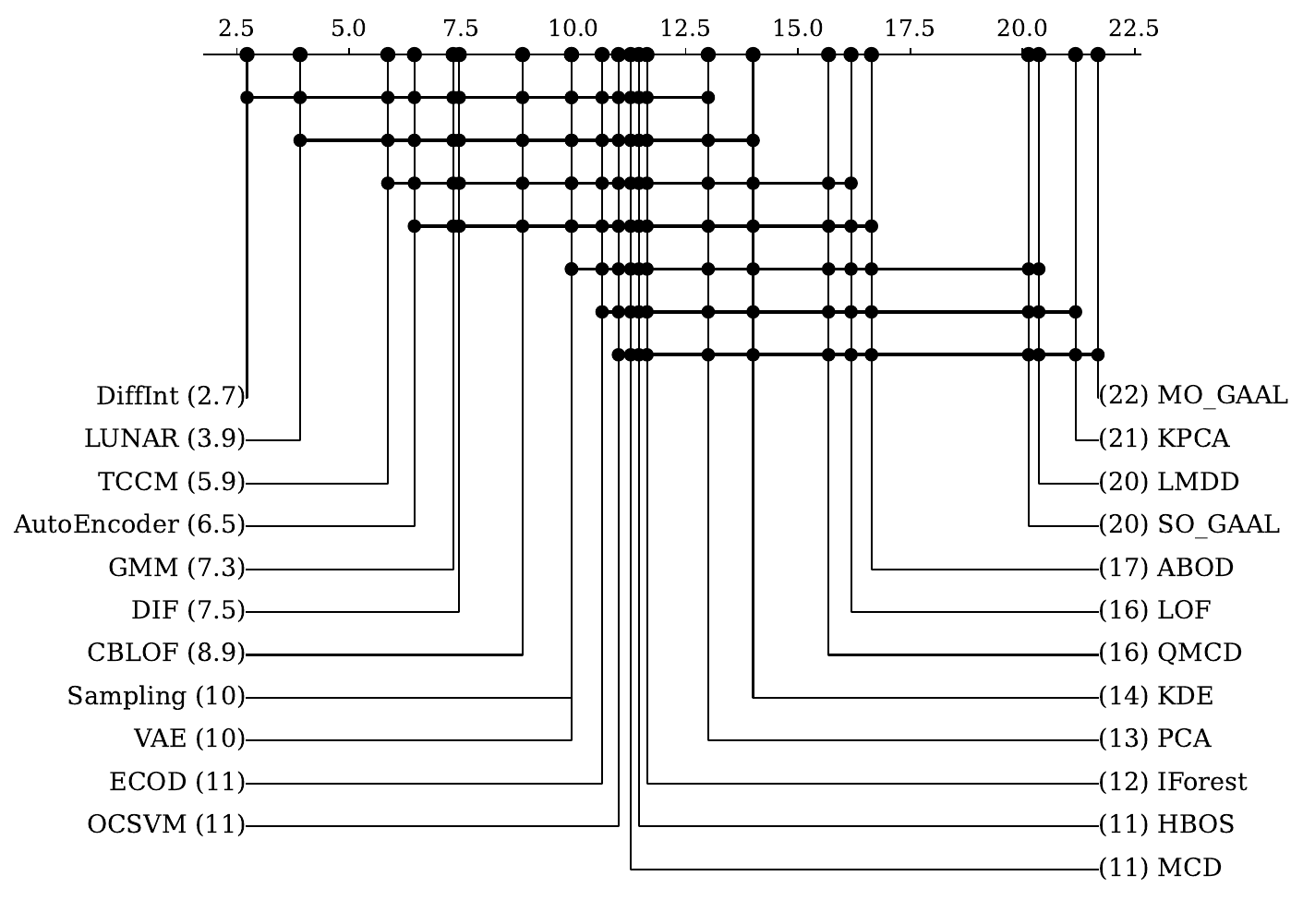}\caption{Large ($n{=}11$).}
\end{subfigure}\hfill
\begin{subfigure}{0.49\linewidth}
\includegraphics[width=\linewidth]{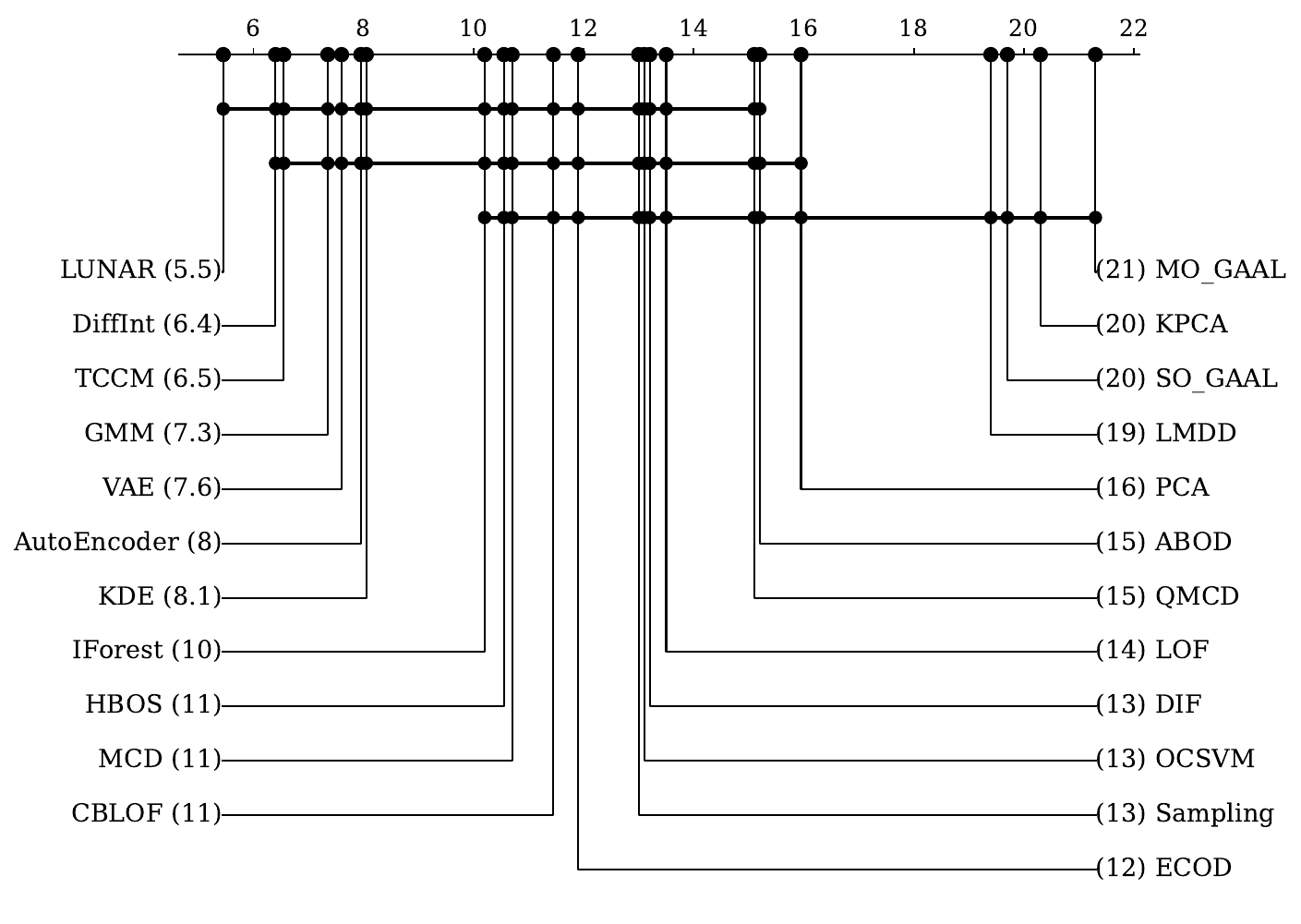}\caption{High-dimensional ($n{=}10$).}
\end{subfigure}
\caption{Per-scale \textbf{AUPR} critical-difference diagrams (same protocol
as Fig.~\ref{fig:cdscale-auc}). \method{} leads strictly on small, medium, and
large data; on high-dimensional data LUNAR's mean rank (5.45) is slightly
ahead of \method{} (6.40), but the gap is well inside the per-scale Nemenyi
CD ($\approx\!11$), i.e.\ a statistical tie. Matches the ``\method{} AUPR''
column of Table~\ref{tab:ranks}.}
\label{fig:cdscale-aupr}
\end{figure*}

\begingroup\scriptsize\setlength{\tabcolsep}{3pt}
\begin{table*}[t]\centering
\caption{ROC--AUC, small datasets (1/2).}\label{tab:auc-small-1}
\begin{tabular}{lcccccc}
\toprule
Model & 4\_breastw & 14\_glass & 15\_Hepatitis & 18\_Ionosphere & 21\_Lymphography & 29\_Pima \\
\midrule
ABOD & 0.3737$_{\pm 0.0338}$(20) & 0.8264$_{\pm 0.0697}$(8) & 0.6542$_{\pm 0.1520}$(19) & 0.9255$_{\pm 0.0185}$(11) & 0.9682$_{\pm 0.0240}$(18) & 0.6548$_{\pm 0.0408}$(14) \\
AutoEncoder & 0.9855$_{\pm 0.0075}$(12) & 0.7107$_{\pm 0.0779}$(16) & 0.8501$_{\pm 0.0387}$(\blue{2}) & 0.9604$_{\pm 0.0090}$(5) & 0.9904$_{\pm 0.0087}$(12) & 0.7070$_{\pm 0.0239}$(10) \\
CBLOF & 0.9939$_{\pm 0.0026}$(5) & 0.8784$_{\pm 0.0345}$(\lightred{3}) & 0.8089$_{\pm 0.0874}$(10) & 0.9557$_{\pm 0.0223}$(7) & 0.9816$_{\pm 0.0438}$(15) & 0.7006$_{\pm 0.0232}$(11) \\
DIF & 0.8377$_{\pm 0.0375}$(16) & 0.8562$_{\pm 0.0371}$(5) & 0.7723$_{\pm 0.0901}$(16) & 0.9587$_{\pm 0.0074}$(6) & 0.8833$_{\pm 0.0992}$(19) & 0.6144$_{\pm 0.0258}$(19) \\
DiffInt & 0.9954$_{\pm 0.0018}$(\green{1}) & 0.8788$_{\pm 0.0245}$(\blue{2}) & 0.8333$_{\pm 0.0401}$(7) & 0.9757$_{\pm 0.0030}$(\blue{2}) & 0.9997$_{\pm 0.0014}$(\blue{2}) & 0.7304$_{\pm 0.0113}$(\lightred{3}) \\
ECOD & 0.9932$_{\pm 0.0029}$(7) & 0.6911$_{\pm 0.0797}$(17) & 0.8402$_{\pm 0.0550}$(\lightred{3}) & 0.7745$_{\pm 0.0221}$(17) & 0.9990$_{\pm 0.0022}$(4) & 0.6374$_{\pm 0.0239}$(16) \\
GMM & 0.9863$_{\pm 0.0055}$(11) & 0.7508$_{\pm 0.0624}$(14) & 0.8018$_{\pm 0.0785}$(13) & 0.9664$_{\pm 0.0071}$(\lightred{3}) & 0.9873$_{\pm 0.0092}$(13) & 0.7202$_{\pm 0.0217}$(7) \\
\bottomrule
\end{tabular}
\end{table*}
\begin{table*}[t]\centering
\caption{ROC--AUC, small datasets (2/2).}\label{tab:auc-small-2}
\begin{tabular}{lcccccc}
\toprule
Model & 37\_Stamps & 39\_vertebral & 42\_WBC & 43\_WDBC & 45\_wine & 46\_WPBC \\
\midrule
ABOD & 0.6954$_{\pm 0.0525}$(19) & 0.4214$_{\pm 0.0847}$(18) & 0.9297$_{\pm 0.0301}$(18) & 0.9143$_{\pm 0.0504}$(18) & 0.5828$_{\pm 0.2278}$(20) & 0.4897$_{\pm 0.0563}$(19) \\
AutoEncoder & 0.9217$_{\pm 0.0398}$(10) & 0.4644$_{\pm 0.1041}$(12) & 0.9841$_{\pm 0.0220}$(12) & 0.9937$_{\pm 0.0055}$(\blue{2}) & 0.9364$_{\pm 0.0499}$(6) & 0.5178$_{\pm 0.0481}$(11) \\
CBLOF & 0.9457$_{\pm 0.0235}$(4) & 0.4711$_{\pm 0.0742}$(10) & 0.9961$_{\pm 0.0039}$(\blue{2}) & 0.9860$_{\pm 0.0099}$(14) & 0.9517$_{\pm 0.0376}$(\blue{2}) & 0.5280$_{\pm 0.0407}$(9) \\
DIF & 0.9426$_{\pm 0.0215}$(5) & 0.5781$_{\pm 0.0895}$(\lightred{3}) & 0.7821$_{\pm 0.0531}$(19) & 0.7227$_{\pm 0.0519}$(19) & 0.7978$_{\pm 0.1078}$(17) & 0.4971$_{\pm 0.0485}$(18) \\
DiffInt & 0.9516$_{\pm 0.0116}$(\green{1}) & 0.5274$_{\pm 0.0429}$(4) & 0.9973$_{\pm 0.0050}$(\green{1}) & 0.9741$_{\pm 0.0076}$(16) & 0.8852$_{\pm 0.0594}$(13) & 0.5652$_{\pm 0.0222}$(\green{1}) \\
ECOD & 0.9157$_{\pm 0.0244}$(13) & 0.4645$_{\pm 0.0796}$(11) & 0.9953$_{\pm 0.0046}$(6) & 0.9870$_{\pm 0.0068}$(11) & 0.8484$_{\pm 0.0960}$(14) & 0.5029$_{\pm 0.0422}$(17) \\
GMM & 0.9270$_{\pm 0.0226}$(9) & 0.4909$_{\pm 0.0550}$(8) & 0.9767$_{\pm 0.0198}$(16) & 0.9882$_{\pm 0.0093}$(9) & 0.9677$_{\pm 0.0268}$(\green{1}) & 0.5141$_{\pm 0.0507}$(12) \\
\bottomrule
\end{tabular}
\end{table*}
\begin{table*}[t]\centering
\caption{ROC--AUC, medium datasets (1/3).}\label{tab:auc-med-1}
\begin{tabular}{lcccccc}
\toprule
Model & 2\_annthyroid & 6\_cardio & 7\_Cardiotocography & 12\_fault & 19\_landsat & 20\_letter \\
\midrule
ABOD & 0.7439$_{\pm 0.0124}$(13) & 0.5877$_{\pm 0.0275}$(19) & 0.4917$_{\pm 0.0184}$(21) & 0.6928$_{\pm 0.0110}$(10) & 0.5164$_{\pm 0.0127}$(17) & 0.8652$_{\pm 0.0149}$(\lightred{3}) \\
AutoEncoder & 0.8612$_{\pm 0.0192}$(6) & 0.9307$_{\pm 0.0271}$(12) & 0.7202$_{\pm 0.0335}$(12) & 0.7386$_{\pm 0.0188}$(4) & 0.5773$_{\pm 0.0111}$(13) & 0.8135$_{\pm 0.0340}$(7) \\
CBLOF & 0.6578$_{\pm 0.0179}$(17) & 0.9546$_{\pm 0.0043}$(7) & 0.7815$_{\pm 0.0315}$(8) & 0.7094$_{\pm 0.0358}$(8) & 0.6907$_{\pm 0.0179}$(\lightred{3}) & 0.7755$_{\pm 0.0370}$(9) \\
DIF & 0.8025$_{\pm 0.0282}$(10) & 0.9570$_{\pm 0.0075}$(6) & 0.6877$_{\pm 0.0171}$(15) & 0.7386$_{\pm 0.0145}$(4) & 0.5947$_{\pm 0.0167}$(11) & 0.6661$_{\pm 0.0575}$(12) \\
DiffInt & 0.8330$_{\pm 0.0094}$(8) & 0.9721$_{\pm 0.0050}$(\blue{2}) & 0.8023$_{\pm 0.0122}$(6) & 0.7976$_{\pm 0.0044}$(\blue{2}) & 0.6540$_{\pm 0.0045}$(5) & 0.9006$_{\pm 0.0102}$(\blue{2}) \\
ECOD & 0.8249$_{\pm 0.0099}$(9) & 0.9519$_{\pm 0.0047}$(10) & 0.8322$_{\pm 0.0096}$(\lightred{3}) & 0.4919$_{\pm 0.0126}$(20) & 0.3959$_{\pm 0.0111}$(21) & 0.5871$_{\pm 0.0418}$(17) \\
GMM & 0.8412$_{\pm 0.0270}$(7) & 0.9448$_{\pm 0.0075}$(11) & 0.7507$_{\pm 0.0134}$(10) & 0.7024$_{\pm 0.0160}$(9) & 0.4967$_{\pm 0.0096}$(19) & 0.8473$_{\pm 0.0350}$(5) \\
\bottomrule
\end{tabular}
\end{table*}
\begin{table*}[t]\centering
\caption{ROC--AUC, medium datasets (2/3).}\label{tab:auc-med-2}
\begin{tabular}{lcccccc}
\toprule
Model & 27\_PageBlocks & 28\_pendigits & 30\_satellite & 31\_satimage-2 & 38\_thyroid & 40\_vowels \\
\midrule
ABOD & 0.7166$_{\pm 0.0251}$(20) & 0.6559$_{\pm 0.0303}$(18) & 0.5619$_{\pm 0.0148}$(20) & 0.7878$_{\pm 0.0471}$(18) & 0.9111$_{\pm 0.0147}$(16) & 0.9440$_{\pm 0.0194}$(6) \\
AutoEncoder & 0.9434$_{\pm 0.0122}$(4) & 0.9805$_{\pm 0.0078}$(5) & 0.8047$_{\pm 0.0102}$(8) & 0.9975$_{\pm 0.0020}$(7) & 0.9816$_{\pm 0.0080}$(7) & 0.9460$_{\pm 0.0187}$(5) \\
CBLOF & 0.8792$_{\pm 0.0185}$(14) & 0.9760$_{\pm 0.0054}$(7) & 0.8575$_{\pm 0.0259}$(4) & 0.9985$_{\pm 0.0019}$(\lightred{3}) & 0.9354$_{\pm 0.0105}$(14) & 0.9343$_{\pm 0.0234}$(7) \\
DIF & 0.9354$_{\pm 0.0115}$(5) & 0.9815$_{\pm 0.0070}$(4) & 0.7886$_{\pm 0.0133}$(12) & 0.9979$_{\pm 0.0017}$(4) & 0.9838$_{\pm 0.0039}$(6) & 0.8136$_{\pm 0.0444}$(10) \\
DiffInt & 0.9500$_{\pm 0.0043}$(\blue{2}) & 0.9966$_{\pm 0.0011}$(\blue{2}) & 0.8416$_{\pm 0.0046}$(5) & 0.9993$_{\pm 0.0010}$(\green{1}) & 0.9770$_{\pm 0.0047}$(8) & 0.9808$_{\pm 0.0068}$(\blue{2}) \\
ECOD & 0.9253$_{\pm 0.0030}$(7) & 0.9329$_{\pm 0.0094}$(14) & 0.6155$_{\pm 0.0115}$(16) & 0.9749$_{\pm 0.0169}$(17) & 0.9859$_{\pm 0.0025}$(4) & 0.6123$_{\pm 0.0509}$(18) \\
GMM & 0.9577$_{\pm 0.0054}$(\green{1}) & 0.8485$_{\pm 0.0120}$(16) & 0.8028$_{\pm 0.0065}$(10) & 0.9958$_{\pm 0.0013}$(10) & 0.9760$_{\pm 0.0050}$(10) & 0.9480$_{\pm 0.0177}$(4) \\
\bottomrule
\end{tabular}
\end{table*}
\begin{table*}[t]\centering
\caption{ROC--AUC, medium datasets (3/3).}\label{tab:auc-med-3}
\begin{tabular}{lccc}
\toprule
Model & 41\_Waveform & 44\_Wilt & 47\_yeast \\
\midrule
ABOD & 0.6265$_{\pm 0.0327}$(15) & 0.5544$_{\pm 0.0226}$(9) & 0.4188$_{\pm 0.0212}$(18) \\
AutoEncoder & 0.6848$_{\pm 0.0565}$(12) & 0.5556$_{\pm 0.0760}$(8) & 0.4561$_{\pm 0.0156}$(9) \\
CBLOF & 0.7546$_{\pm 0.0522}$(\lightred{3}) & 0.3755$_{\pm 0.0282}$(19) & 0.4777$_{\pm 0.0209}$(6) \\
DIF & 0.7470$_{\pm 0.0248}$(5) & 0.4179$_{\pm 0.0614}$(16) & 0.3969$_{\pm 0.0183}$(23) \\
DiffInt & 0.7678$_{\pm 0.0178}$(\green{1}) & 0.6789$_{\pm 0.0272}$(5) & 0.4926$_{\pm 0.0163}$(5) \\
ECOD & 0.6140$_{\pm 0.0352}$(16) & 0.3994$_{\pm 0.0286}$(18) & 0.4492$_{\pm 0.0167}$(12) \\
GMM & 0.5790$_{\pm 0.0356}$(17) & 0.7295$_{\pm 0.0295}$(\blue{2}) & 0.4374$_{\pm 0.0190}$(14) \\
\bottomrule
\end{tabular}
\end{table*}
\begin{table*}[t]\centering
\caption{ROC--AUC, large datasets (1/2).}\label{tab:auc-large-1}
\begin{tabular}{lcccccc}
\toprule
Model & 1\_ALOI & 8\_celeba & 10\_cover & 11\_donors & 13\_fraud & 16\_http \\
\midrule
ABOD & 0.7311$_{\pm 0.0082}$(\lightred{3}) & 0.4706$_{\pm 0.0053}$(19) & 0.7391$_{\pm 0.0072}$(16) & 0.4036$_{\pm 0.0014}$(16) & 0.8889$_{\pm 0.0172}$(17) & 0.7510$_{\pm 0.0017}$(16) \\
AutoEncoder & 0.5554$_{\pm 0.0126}$(6) & 0.6988$_{\pm 0.0311}$(11) & 0.9843$_{\pm 0.0044}$(\blue{2}) & 0.9359$_{\pm 0.0100}$(\lightred{3}) & 0.9592$_{\pm 0.0070}$(4) & 0.9988$_{\pm 0.0012}$(6) \\
CBLOF & 0.5449$_{\pm 0.0101}$(11) & 0.6790$_{\pm 0.0629}$(12) & 0.9245$_{\pm 0.0174}$(9) & 0.9159$_{\pm 0.0156}$(6) & 0.9535$_{\pm 0.0052}$(9) & 0.9949$_{\pm 0.0001}$(12) \\
DIF & 0.5498$_{\pm 0.0103}$(9) & 0.6636$_{\pm 0.0155}$(14) & 0.9787$_{\pm 0.0040}$(\lightred{3}) & 0.9118$_{\pm 0.0101}$(8) & 0.9528$_{\pm 0.0078}$(10) & 0.9935$_{\pm 0.0002}$(14) \\
DiffInt & 0.5857$_{\pm 0.0067}$(4) & 0.7463$_{\pm 0.0386}$(8) & 0.9949$_{\pm 0.0015}$(\green{1}) & 1.0000$_{\pm 0.0000}$(\green{1}) & 0.9596$_{\pm 0.0025}$(\lightred{3}) & 0.9998$_{\pm 0.0022}$(\green{1}) \\
ECOD & 0.5326$_{\pm 0.0086}$(16) & 0.7664$_{\pm 0.0030}$(7) & 0.9386$_{\pm 0.0019}$(6) & 0.9215$_{\pm 0.0005}$(5) & 0.9525$_{\pm 0.0057}$(11) & 0.9880$_{\pm 0.0003}$(15) \\
GMM & 0.5601$_{\pm 0.0100}$(5) & 0.8076$_{\pm 0.0032}$(\lightred{3}) & 0.9505$_{\pm 0.0012}$(5) & 0.9255$_{\pm 0.0004}$(4) & 0.9597$_{\pm 0.0053}$(\blue{2}) & 0.9994$_{\pm 0.0000}$(4) \\
\bottomrule
\end{tabular}
\end{table*}
\begin{table*}[t]\centering
\caption{ROC--AUC, large datasets (2/2).}\label{tab:auc-large-2}
\begin{tabular}{lccccc}
\toprule
Model & 22\_magic & 23\_mammography & 32\_shuttle & 33\_skin & 34\_smtp \\
\midrule
ABOD & 0.7930$_{\pm 0.0069}$(7) & 0.5424$_{\pm 0.0121}$(21) & 0.6135$_{\pm 0.0077}$(20) & 0.4064$_{\pm 0.0026}$(20) & 0.8606$_{\pm 0.0198}$(9) \\
AutoEncoder & 0.8211$_{\pm 0.0111}$(4) & 0.9114$_{\pm 0.0197}$(\blue{2}) & 0.9968$_{\pm 0.0010}$(\lightred{3}) & 0.7975$_{\pm 0.0685}$(9) & 0.9285$_{\pm 0.0433}$(\blue{2}) \\
CBLOF & 0.8235$_{\pm 0.0062}$(\lightred{3}) & 0.8238$_{\pm 0.0188}$(13) & 0.9871$_{\pm 0.0029}$(15) & 0.9178$_{\pm 0.0035}$(\lightred{3}) & 0.8446$_{\pm 0.0747}$(14) \\
DIF & 0.7822$_{\pm 0.0136}$(8) & 0.8281$_{\pm 0.0231}$(12) & 0.9905$_{\pm 0.0023}$(9) & 0.8759$_{\pm 0.0087}$(8) & 0.8490$_{\pm 0.0617}$(13) \\
DiffInt & 0.8644$_{\pm 0.0026}$(\green{1}) & 0.8803$_{\pm 0.0097}$(7) & 0.9992$_{\pm 0.0004}$(\blue{2}) & 0.9909$_{\pm 0.0021}$(\green{1}) & 0.8961$_{\pm 0.0151}$(6) \\
ECOD & 0.6541$_{\pm 0.0055}$(19) & 0.9115$_{\pm 0.0163}$(\green{1}) & 0.9962$_{\pm 0.0007}$(5) & 0.5421$_{\pm 0.0011}$(18) & 0.8868$_{\pm 0.0419}$(7) \\
GMM & 0.8040$_{\pm 0.0036}$(5) & 0.8814$_{\pm 0.0198}$(6) & 0.9936$_{\pm 0.0016}$(8) & 0.8885$_{\pm 0.0004}$(6) & 0.8311$_{\pm 0.0862}$(16) \\
\bottomrule
\end{tabular}
\end{table*}
\begin{table*}[t]\centering
\caption{ROC--AUC, high-dimensional datasets (1/2).}\label{tab:auc-high-1}
\begin{tabular}{lcccccc}
\toprule
Model & 3\_backdoor & 5\_campaign & 9\_census & 17\_InternetAds & 24\_mnist & 25\_musk \\
\midrule
ABOD & 0.6392$_{\pm 0.0091}$(17) & 0.6588$_{\pm 0.0065}$(16) & 0.4026$_{\pm 0.0026}$(9) & 0.6577$_{\pm 0.0328}$(14) & 0.7120$_{\pm 0.0145}$(16) & 0.2821$_{\pm 0.0881}$(20) \\
AutoEncoder & 0.9331$_{\pm 0.0084}$(5) & 0.8128$_{\pm 0.0074}$(\green{1}) & - & 0.8806$_{\pm 0.0147}$(\lightred{3}) & 0.9345$_{\pm 0.0081}$(\green{1}) & 1.0000$_{\pm 0.0000}$(\green{1}) \\
CBLOF & 0.8396$_{\pm 0.0204}$(12) & 0.6743$_{\pm 0.0125}$(15) & - & 0.7185$_{\pm 0.0115}$(7) & 0.8727$_{\pm 0.0109}$(11) & 1.0000$_{\pm 0.0000}$(\green{1}) \\
DIF & 0.9311$_{\pm 0.0067}$(6) & 0.6776$_{\pm 0.0116}$(14) & - & 0.6346$_{\pm 0.0217}$(16) & 0.8818$_{\pm 0.0097}$(8) & 0.9999$_{\pm 0.0001}$(11) \\
DiffInt & 0.9613$_{\pm 0.0043}$(\green{1}) & 0.8106$_{\pm 0.0050}$(\blue{2}) & 0.6273$_{\pm 0.0007}$(4) & 0.7036$_{\pm 0.0053}$(11) & 0.8950$_{\pm 0.0074}$(6) & 1.0000$_{\pm 0.0000}$(\green{1}) \\
ECOD & 0.8567$_{\pm 0.0077}$(11) & 0.7836$_{\pm 0.0040}$(8) & - & 0.7079$_{\pm 0.0116}$(9) & 0.7814$_{\pm 0.0121}$(14) & 0.9984$_{\pm 0.0008}$(14) \\
GMM & 0.9276$_{\pm 0.0083}$(7) & 0.7973$_{\pm 0.0040}$(6) & - & 0.9123$_{\pm 0.0089}$(\green{1}) & 0.9240$_{\pm 0.0080}$(5) & 1.0000$_{\pm 0.0000}$(\green{1}) \\
\bottomrule
\end{tabular}
\end{table*}
\begin{table*}[t]\centering
\caption{ROC--AUC, high-dimensional datasets (2/2).}\label{tab:auc-high-2}
\begin{tabular}{lcccc}
\toprule
Model & 26\_optdigits & 35\_SpamBase & 36\_speech & arrhythmia \\
\midrule
ABOD & 0.4934$_{\pm 0.0350}$(18) & 0.4517$_{\pm 0.0131}$(20) & 0.6062$_{\pm 0.0881}$(\blue{2}) & 0.7597$_{\pm 0.0329}$(18) \\
AutoEncoder & 0.8707$_{\pm 0.0148}$(7) & 0.8186$_{\pm 0.0140}$(\lightred{3}) & 0.4570$_{\pm 0.0492}$(13) & 0.7890$_{\pm 0.0327}$(13) \\
CBLOF & 0.8937$_{\pm 0.0239}$(4) & 0.6153$_{\pm 0.0182}$(13) & 0.4515$_{\pm 0.0473}$(16) & 0.8082$_{\pm 0.0357}$(7) \\
DIF & 0.6939$_{\pm 0.0792}$(11) & 0.5566$_{\pm 0.0306}$(17) & 0.4397$_{\pm 0.0622}$(22) & 0.8060$_{\pm 0.0362}$(8) \\
DiffInt & 0.9604$_{\pm 0.0072}$(\lightred{3}) & 0.7041$_{\pm 0.0072}$(11) & 0.5047$_{\pm 0.0308}$(4) & 0.8193$_{\pm 0.0185}$(5) \\
ECOD & 0.6357$_{\pm 0.0142}$(15) & 0.7194$_{\pm 0.0097}$(9) & 0.4511$_{\pm 0.0486}$(17) & 0.8236$_{\pm 0.0319}$(4) \\
GMM & 0.8067$_{\pm 0.0087}$(10) & 0.7974$_{\pm 0.0147}$(7) & 0.5013$_{\pm 0.0483}$(5) & 0.7819$_{\pm 0.0346}$(16) \\
\bottomrule
\end{tabular}
\end{table*}
\endgroup

\subsection{Average precision (AUPR)}
\label{sec:aupr}
Tables~\ref{tab:aupr-small-1}--\ref{tab:aupr-high-2} give the per-dataset AUPR
under the same protocol and highlighting. AUPR weights the rare anomaly class
more heavily than ROC--AUC; the \method{} ranking is consistent across the two
metrics, indicating the gains are not an artifact of the chosen measure.

\begingroup\scriptsize\setlength{\tabcolsep}{3pt}
\begin{table*}[t]\centering
\caption{AUPR, small datasets (1/2).}\label{tab:aupr-small-1}
\begin{tabular}{lcccccc}
\toprule
Model & 4\_breastw & 14\_glass & 15\_Hepatitis & 18\_Ionosphere & 21\_Lymphography & 29\_Pima \\
\midrule
ABOD & 0.2877$_{\pm 0.0169}$(20) & 0.2252$_{\pm 0.1340}$(8) & 0.3422$_{\pm 0.1897}$(19) & 0.9189$_{\pm 0.0188}$(9) & 0.6225$_{\pm 0.2143}$(18) & 0.4735$_{\pm 0.0308}$(16) \\
AutoEncoder & 0.9664$_{\pm 0.0184}$(13) & 0.1502$_{\pm 0.1025}$(16) & 0.6348$_{\pm 0.1385}$(\blue{2}) & 0.9416$_{\pm 0.0131}$(5) & 0.8167$_{\pm 0.1820}$(15) & 0.5387$_{\pm 0.0181}$(7) \\
CBLOF & 0.9892$_{\pm 0.0043}$(4) & 0.2286$_{\pm 0.1182}$(7) & 0.5388$_{\pm 0.1740}$(4) & 0.9201$_{\pm 0.0374}$(8) & 0.9299$_{\pm 0.1435}$(10) & 0.5244$_{\pm 0.0279}$(10) \\
DIF & 0.6048$_{\pm 0.0525}$(17) & 0.2519$_{\pm 0.1434}$(4) & 0.4093$_{\pm 0.1430}$(17) & 0.9367$_{\pm 0.0108}$(6) & 0.4838$_{\pm 0.2684}$(19) & 0.4144$_{\pm 0.0207}$(20) \\
DiffInt & 0.9919$_{\pm 0.0032}$(\green{1}) & 0.3294$_{\pm 0.0800}$(\green{1}) & 0.5211$_{\pm 0.0667}$(8) & 0.9574$_{\pm 0.0073}$(\blue{2}) & 0.9951$_{\pm 0.0201}$(\blue{2}) & 0.5668$_{\pm 0.0149}$(\lightred{3}) \\
ECOD & 0.9869$_{\pm 0.0053}$(7) & 0.1812$_{\pm 0.1485}$(11) & 0.4933$_{\pm 0.0920}$(10) & 0.6874$_{\pm 0.0444}$(17) & 0.9867$_{\pm 0.0292}$(4) & 0.4992$_{\pm 0.0179}$(13) \\
GMM & 0.9693$_{\pm 0.0139}$(12) & 0.2004$_{\pm 0.1567}$(10) & 0.5098$_{\pm 0.1610}$(9) & 0.9569$_{\pm 0.0099}$(\lightred{3}) & 0.7933$_{\pm 0.1614}$(16) & 0.5418$_{\pm 0.0200}$(6) \\
\bottomrule
\end{tabular}
\end{table*}
\begin{table*}[t]\centering
\caption{AUPR, small datasets (2/2).}\label{tab:aupr-small-2}
\begin{tabular}{lcccccc}
\toprule
Model & 37\_Stamps & 39\_vertebral & 42\_WBC & 43\_WDBC & 45\_wine & 46\_WPBC \\
\midrule
ABOD & 0.2120$_{\pm 0.0573}$(19) & 0.1222$_{\pm 0.0268}$(17) & 0.3421$_{\pm 0.1683}$(18) & 0.2573$_{\pm 0.1460}$(18) & 0.1717$_{\pm 0.0927}$(20) & 0.2141$_{\pm 0.0518}$(22) \\
AutoEncoder & 0.4350$_{\pm 0.1051}$(13) & 0.1265$_{\pm 0.0253}$(14) & 0.8106$_{\pm 0.2272}$(13) & 0.7782$_{\pm 0.2134}$(\lightred{3}) & 0.5529$_{\pm 0.2822}$(6) & 0.2347$_{\pm 0.0631}$(12) \\
CBLOF & 0.5528$_{\pm 0.1300}$(\blue{2}) & 0.1314$_{\pm 0.0278}$(9) & 0.9480$_{\pm 0.0504}$(5) & 0.6179$_{\pm 0.2237}$(13) & 0.6318$_{\pm 0.2372}$(\blue{2}) & 0.2413$_{\pm 0.0610}$(6) \\
DIF & 0.5181$_{\pm 0.0937}$(4) & 0.1855$_{\pm 0.0487}$(\lightred{3}) & 0.1239$_{\pm 0.0324}$(19) & 0.0662$_{\pm 0.0339}$(20) & 0.2498$_{\pm 0.1465}$(18) & 0.2223$_{\pm 0.0594}$(19) \\
DiffInt & 0.6095$_{\pm 0.0647}$(\green{1}) & 0.1548$_{\pm 0.0192}$(5) & 0.9658$_{\pm 0.0466}$(\green{1}) & 0.5729$_{\pm 0.1161}$(16) & 0.4625$_{\pm 0.1262}$(11) & 0.2834$_{\pm 0.0336}$(\green{1}) \\
ECOD & 0.4575$_{\pm 0.1030}$(11) & 0.1310$_{\pm 0.0239}$(10) & 0.9174$_{\pm 0.0831}$(8) & 0.6875$_{\pm 0.1531}$(6) & 0.3033$_{\pm 0.0885}$(16) & 0.2158$_{\pm 0.0474}$(21) \\
GMM & 0.4618$_{\pm 0.0950}$(10) & 0.1291$_{\pm 0.0219}$(11) & 0.6966$_{\pm 0.1988}$(17) & 0.6851$_{\pm 0.2442}$(8) & 0.6901$_{\pm 0.2396}$(\green{1}) & 0.2392$_{\pm 0.0704}$(8) \\
\bottomrule
\end{tabular}
\end{table*}
\begin{table*}[t]\centering
\caption{AUPR, medium datasets (1/3).}\label{tab:aupr-med-1}
\begin{tabular}{lcccccc}
\toprule
Model & 2\_annthyroid & 6\_cardio & 7\_Cardiotocography & 12\_fault & 19\_landsat & 20\_letter \\
\midrule
ABOD & 0.1685$_{\pm 0.0118}$(17) & 0.1865$_{\pm 0.0217}$(19) & 0.2591$_{\pm 0.0182}$(20) & 0.5010$_{\pm 0.0160}$(10) & 0.2219$_{\pm 0.0116}$(18) & 0.3120$_{\pm 0.0422}$(6) \\
AutoEncoder & 0.5003$_{\pm 0.0592}$(4) & 0.6195$_{\pm 0.0650}$(12) & 0.4920$_{\pm 0.0476}$(13) & 0.5937$_{\pm 0.0420}$(5) & 0.2604$_{\pm 0.0116}$(10) & 0.2659$_{\pm 0.0448}$(7) \\
CBLOF & 0.1605$_{\pm 0.0144}$(18) & 0.7159$_{\pm 0.0405}$(7) & 0.5803$_{\pm 0.0271}$(7) & 0.5318$_{\pm 0.0543}$(9) & 0.3355$_{\pm 0.0152}$(4) & 0.1951$_{\pm 0.0339}$(9) \\
DIF & 0.4211$_{\pm 0.0358}$(8) & 0.7767$_{\pm 0.0339}$(4) & 0.5001$_{\pm 0.0246}$(12) & 0.5894$_{\pm 0.0342}$(6) & 0.2932$_{\pm 0.0103}$(7) & 0.1555$_{\pm 0.0476}$(10) \\
DiffInt & 0.3652$_{\pm 0.0133}$(10) & 0.8114$_{\pm 0.0213}$(\blue{2}) & 0.6271$_{\pm 0.0166}$(\green{1}) & 0.6639$_{\pm 0.0128}$(\blue{2}) & 0.3688$_{\pm 0.0088}$(\blue{2}) & 0.3954$_{\pm 0.0306}$(\blue{2}) \\
ECOD & 0.3479$_{\pm 0.0222}$(11) & 0.6705$_{\pm 0.0231}$(11) & 0.5820$_{\pm 0.0238}$(6) & 0.3467$_{\pm 0.0165}$(22) & 0.1769$_{\pm 0.0043}$(21) & 0.0899$_{\pm 0.0180}$(16) \\
GMM & 0.4365$_{\pm 0.0291}$(7) & 0.6968$_{\pm 0.0376}$(8) & 0.5269$_{\pm 0.0250}$(9) & 0.5459$_{\pm 0.0311}$(7) & 0.2060$_{\pm 0.0064}$(20) & 0.3384$_{\pm 0.0526}$(\lightred{3}) \\
\bottomrule
\end{tabular}
\end{table*}
\begin{table*}[t]\centering
\caption{AUPR, medium datasets (2/3).}\label{tab:aupr-med-2}
\begin{tabular}{lcccccc}
\toprule
Model & 27\_PageBlocks & 28\_pendigits & 30\_satellite & 31\_satimage-2 & 38\_thyroid & 40\_vowels \\
\midrule
ABOD & 0.3540$_{\pm 0.0282}$(18) & 0.0601$_{\pm 0.0106}$(18) & 0.3867$_{\pm 0.0150}$(20) & 0.1329$_{\pm 0.0474}$(18) & 0.1251$_{\pm 0.0139}$(21) & 0.4213$_{\pm 0.1150}$(7) \\
AutoEncoder & 0.7179$_{\pm 0.0330}$(4) & 0.4177$_{\pm 0.0942}$(7) & 0.7908$_{\pm 0.0112}$(6) & 0.8612$_{\pm 0.0970}$(12) & 0.6788$_{\pm 0.0687}$(5) & 0.5336$_{\pm 0.0721}$(4) \\
CBLOF & 0.5343$_{\pm 0.0463}$(12) & 0.4306$_{\pm 0.1145}$(6) & 0.7996$_{\pm 0.0198}$(5) & 0.9727$_{\pm 0.0222}$(\blue{2}) & 0.2328$_{\pm 0.0654}$(16) & 0.3702$_{\pm 0.1387}$(8) \\
DIF & 0.6854$_{\pm 0.0315}$(6) & 0.5697$_{\pm 0.1402}$(5) & 0.7803$_{\pm 0.0134}$(7) & 0.9036$_{\pm 0.0684}$(9) & 0.7491$_{\pm 0.0391}$(\green{1}) & 0.2130$_{\pm 0.0693}$(10) \\
DiffInt & 0.7673$_{\pm 0.0121}$(\green{1}) & 0.8617$_{\pm 0.0325}$(\blue{2}) & 0.8168$_{\pm 0.0060}$(\blue{2}) & 0.9727$_{\pm 0.0117}$(\blue{2}) & 0.5704$_{\pm 0.0412}$(10) & 0.7692$_{\pm 0.0502}$(\blue{2}) \\
ECOD & 0.5734$_{\pm 0.0258}$(9) & 0.3003$_{\pm 0.0431}$(11) & 0.5882$_{\pm 0.0081}$(17) & 0.7695$_{\pm 0.0614}$(15) & 0.6398$_{\pm 0.0425}$(8) & 0.1090$_{\pm 0.0263}$(16) \\
GMM & 0.7209$_{\pm 0.0241}$(\lightred{3}) & 0.0879$_{\pm 0.0160}$(16) & 0.7768$_{\pm 0.0111}$(8) & 0.7337$_{\pm 0.1016}$(16) & 0.6474$_{\pm 0.0461}$(7) & 0.6293$_{\pm 0.0983}$(\lightred{3}) \\
\bottomrule
\end{tabular}
\end{table*}
\begin{table*}[t]\centering
\caption{AUPR, medium datasets (3/3).}\label{tab:aupr-med-3}
\begin{tabular}{lccc}
\toprule
Model & 41\_Waveform & 44\_Wilt & 47\_yeast \\
\midrule
ABOD & 0.0432$_{\pm 0.0084}$(15) & 0.0548$_{\pm 0.0024}$(9) & 0.3089$_{\pm 0.0129}$(19) \\
AutoEncoder & 0.0787$_{\pm 0.0299}$(7) & 0.0562$_{\pm 0.0139}$(8) & 0.3197$_{\pm 0.0182}$(13) \\
CBLOF & 0.2284$_{\pm 0.0701}$(\green{1}) & 0.0387$_{\pm 0.0030}$(19) & 0.3223$_{\pm 0.0202}$(11) \\
DIF & 0.0873$_{\pm 0.0246}$(6) & 0.0419$_{\pm 0.0059}$(18) & 0.2961$_{\pm 0.0133}$(23) \\
DiffInt & 0.0921$_{\pm 0.0117}$(5) & 0.0798$_{\pm 0.0060}$(5) & 0.3534$_{\pm 0.0129}$(5) \\
ECOD & 0.0431$_{\pm 0.0086}$(16) & 0.0428$_{\pm 0.0048}$(15) & 0.3393$_{\pm 0.0157}$(6) \\
GMM & 0.0426$_{\pm 0.0086}$(18) & 0.0918$_{\pm 0.0082}$(\lightred{3}) & 0.3096$_{\pm 0.0171}$(16) \\
\bottomrule
\end{tabular}
\end{table*}
\begin{table*}[t]\centering
\caption{AUPR, large datasets (1/2).}\label{tab:aupr-large-1}
\begin{tabular}{lcccccc}
\toprule
Model & 1\_ALOI & 8\_celeba & 10\_cover & 11\_donors & 13\_fraud & 16\_http \\
\midrule
ABOD & 0.1185$_{\pm 0.0111}$(\blue{2}) & 0.0202$_{\pm 0.0005}$(19) & 0.0288$_{\pm 0.0013}$(17) & 0.0492$_{\pm 0.0003}$(19) & 0.0159$_{\pm 0.0011}$(18) & 0.0077$_{\pm 0.0002}$(18) \\
AutoEncoder & 0.0419$_{\pm 0.0033}$(8) & 0.0499$_{\pm 0.0104}$(12) & 0.4044$_{\pm 0.0611}$(\lightred{3}) & 0.3215$_{\pm 0.0309}$(4) & 0.5606$_{\pm 0.0488}$(\green{1}) & 0.7764$_{\pm 0.1551}$(6) \\
CBLOF & 0.0406$_{\pm 0.0034}$(9) & 0.0617$_{\pm 0.0294}$(10) & 0.0772$_{\pm 0.0259}$(10) & 0.3004$_{\pm 0.0451}$(6) & 0.4031$_{\pm 0.0301}$(8) & 0.4121$_{\pm 0.0062}$(13) \\
DIF & 0.0438$_{\pm 0.0026}$(5) & 0.0491$_{\pm 0.0041}$(13) & 0.3492$_{\pm 0.0373}$(4) & 0.2382$_{\pm 0.0192}$(9) & 0.4106$_{\pm 0.0472}$(7) & 0.3531$_{\pm 0.0072}$(15) \\
DiffInt & 0.0807$_{\pm 0.0033}$(4) & 0.0528$_{\pm 0.0056}$(11) & 0.8938$_{\pm 0.0073}$(\green{1}) & 0.9994$_{\pm 0.0005}$(\green{1}) & 0.5503$_{\pm 0.0193}$(\blue{2}) & 0.8872$_{\pm 0.1252}$(\blue{2}) \\
ECOD & 0.0335$_{\pm 0.0014}$(20) & 0.1007$_{\pm 0.0055}$(4) & 0.1419$_{\pm 0.0065}$(7) & 0.3098$_{\pm 0.0022}$(5) & 0.2513$_{\pm 0.0266}$(13) & 0.2311$_{\pm 0.0028}$(16) \\
GMM & 0.0403$_{\pm 0.0025}$(11) & 0.0910$_{\pm 0.0042}$(6) & 0.1114$_{\pm 0.0025}$(8) & 0.2779$_{\pm 0.0016}$(8) & 0.4888$_{\pm 0.0380}$(4) & 0.8511$_{\pm 0.0052}$(4) \\
\bottomrule
\end{tabular}
\end{table*}
\begin{table*}[t]\centering
\caption{AUPR, large datasets (2/2).}\label{tab:aupr-large-2}
\begin{tabular}{lccccc}
\toprule
Model & 22\_magic & 23\_mammography & 32\_shuttle & 33\_skin & 34\_smtp \\
\midrule
ABOD & 0.7128$_{\pm 0.0079}$(9) & 0.0240$_{\pm 0.0016}$(21) & 0.1325$_{\pm 0.0051}$(19) & 0.1771$_{\pm 0.0013}$(21) & 0.0014$_{\pm 0.0002}$(20) \\
AutoEncoder & 0.7687$_{\pm 0.0135}$(\lightred{3}) & 0.3620$_{\pm 0.0625}$(\blue{2}) & 0.9496$_{\pm 0.0117}$(12) & 0.4034$_{\pm 0.1190}$(10) & 0.3247$_{\pm 0.2130}$(10) \\
CBLOF & 0.7596$_{\pm 0.0142}$(4) & 0.0785$_{\pm 0.0120}$(19) & 0.9627$_{\pm 0.0060}$(8) & 0.5527$_{\pm 0.0108}$(4) & 0.4802$_{\pm 0.1330}$(6) \\
DIF & 0.7236$_{\pm 0.0113}$(7) & 0.2561$_{\pm 0.0457}$(6) & 0.9701$_{\pm 0.0083}$(5) & 0.4606$_{\pm 0.0187}$(8) & 0.6147$_{\pm 0.1393}$(\lightred{3}) \\
DiffInt & 0.8143$_{\pm 0.0034}$(\green{1}) & 0.3010$_{\pm 0.0216}$(4) & 0.9911$_{\pm 0.0027}$(\blue{2}) & 0.9647$_{\pm 0.0053}$(\green{1}) & 0.7004$_{\pm 0.0484}$(\green{1}) \\
ECOD & 0.5576$_{\pm 0.0058}$(18) & 0.4516$_{\pm 0.0447}$(\green{1}) & 0.9559$_{\pm 0.0078}$(11) & 0.1988$_{\pm 0.0010}$(18) & 0.5990$_{\pm 0.1003}$(4) \\
GMM & 0.7397$_{\pm 0.0068}$(5) & 0.2793$_{\pm 0.0472}$(5) & 0.9404$_{\pm 0.0117}$(13) & 0.4953$_{\pm 0.0018}$(5) & 0.3155$_{\pm 0.1184}$(11) \\
\bottomrule
\end{tabular}
\end{table*}
\begin{table*}[t]\centering
\caption{AUPR, high-dimensional datasets (1/2).}\label{tab:aupr-high-1}
\begin{tabular}{lcccccc}
\toprule
Model & 3\_backdoor & 5\_campaign & 9\_census & 17\_InternetAds & 24\_mnist & 25\_musk \\
\midrule
ABOD & 0.0506$_{\pm 0.0021}$(17) & 0.1690$_{\pm 0.0041}$(17) & 0.0467$_{\pm 0.0004}$(10) & 0.2707$_{\pm 0.0321}$(17) & 0.2124$_{\pm 0.0143}$(15) & 0.0349$_{\pm 0.0173}$(18) \\
AutoEncoder & 0.8526$_{\pm 0.0112}$(4) & 0.3426$_{\pm 0.0116}$(6) & - & 0.8064$_{\pm 0.0288}$(\blue{2}) & 0.5906$_{\pm 0.0384}$(\lightred{3}) & 1.0000$_{\pm 0.0000}$(\green{1}) \\
CBLOF & 0.0859$_{\pm 0.0166}$(14) & 0.2479$_{\pm 0.0200}$(15) & - & 0.5940$_{\pm 0.0186}$(7) & 0.4009$_{\pm 0.0413}$(10) & 1.0000$_{\pm 0.0000}$(\green{1}) \\
DIF & 0.7901$_{\pm 0.0356}$(7) & 0.2486$_{\pm 0.0192}$(14) & - & 0.2758$_{\pm 0.0231}$(16) & 0.4513$_{\pm 0.0344}$(9) & 0.9963$_{\pm 0.0040}$(11) \\
DiffInt & 0.8897$_{\pm 0.0092}$(\blue{2}) & 0.3698$_{\pm 0.0056}$(\blue{2}) & 0.0767$_{\pm 0.0005}$(5) & 0.5616$_{\pm 0.0068}$(10) & 0.5209$_{\pm 0.0264}$(6) & 0.9998$_{\pm 0.0025}$(10) \\
ECOD & 0.1094$_{\pm 0.0036}$(13) & 0.3696$_{\pm 0.0115}$(\lightred{3}) & - & 0.5749$_{\pm 0.0143}$(9) & 0.2070$_{\pm 0.0118}$(16) & 0.9606$_{\pm 0.0175}$(14) \\
GMM & 0.8557$_{\pm 0.0138}$(\lightred{3}) & 0.3608$_{\pm 0.0071}$(4) & - & 0.8295$_{\pm 0.0217}$(\green{1}) & 0.5760$_{\pm 0.0312}$(4) & 1.0000$_{\pm 0.0000}$(\green{1}) \\
\bottomrule
\end{tabular}
\end{table*}
\begin{table*}[t]\centering
\caption{AUPR, high-dimensional datasets (2/2).}\label{tab:aupr-high-2}
\begin{tabular}{lcccc}
\toprule
Model & 26\_optdigits & 35\_SpamBase & 36\_speech & arrhythmia \\
\midrule
ABOD & 0.0562$_{\pm 0.0392}$(12) & 0.3895$_{\pm 0.0078}$(20) & 0.0247$_{\pm 0.0103}$(8) & 0.3791$_{\pm 0.0363}$(18) \\
AutoEncoder & 0.1018$_{\pm 0.0108}$(7) & 0.7571$_{\pm 0.0189}$(\blue{2}) & 0.0220$_{\pm 0.0086}$(13) & 0.4081$_{\pm 0.0651}$(15) \\
CBLOF & 0.1253$_{\pm 0.0305}$(6) & 0.4588$_{\pm 0.0198}$(12) & 0.0217$_{\pm 0.0084}$(15) & 0.4896$_{\pm 0.0815}$(8) \\
DIF & 0.0567$_{\pm 0.0228}$(11) & 0.4532$_{\pm 0.0279}$(14) & 0.0161$_{\pm 0.0048}$(20) & 0.4995$_{\pm 0.0716}$(7) \\
DiffInt & 0.3739$_{\pm 0.0616}$(\lightred{3}) & 0.5853$_{\pm 0.0083}$(11) & 0.0233$_{\pm 0.0046}$(10) & 0.5067$_{\pm 0.0406}$(5) \\
ECOD & 0.0385$_{\pm 0.0051}$(17) & 0.6205$_{\pm 0.0105}$(10) & 0.0227$_{\pm 0.0094}$(11) & 0.5172$_{\pm 0.0701}$(\lightred{3}) \\
GMM & 0.0719$_{\pm 0.0093}$(10) & 0.7368$_{\pm 0.0186}$(5) & 0.0329$_{\pm 0.0182}$(\blue{2}) & 0.3955$_{\pm 0.0553}$(17) \\
\bottomrule
\end{tabular}
\end{table*}
\endgroup
\clearpage

\end{document}